\documentclass{article}

\usepackage{arxiv}
\usepackage[utf8]{inputenc}
\usepackage{graphicx}
\usepackage{amsmath}
\usepackage{amsfonts}
\usepackage{hyperref}
\usepackage{booktabs}
\usepackage{caption}
\usepackage{subcaption}
\usepackage{multirow}

\usepackage{amsmath,amssymb,amsfonts,amsthm,bm}
\usepackage{mathtools}
\usepackage{booktabs}
\usepackage{multirow}
\usepackage{graphicx}
\usepackage{algorithm}
\usepackage{algpseudocode}
\usepackage{xcolor}
\usepackage{microtype}
\usepackage{enumitem}
\usepackage{siunitx}
\usepackage{float}
\usepackage{cite}
\usepackage{hyperref}

\newtheorem{theorem}{Theorem}
\newtheorem{proposition}{Proposition}
\newtheorem{corollary}{Corollary}
\newtheorem{assumption}{Assumption}
\newtheorem{remark}{Remark}
\newtheorem{definition}{Definition}

\newcommand{\R}{\mathcal{R}}
\newcommand{\Rh}{\widehat{\mathcal{R}}}
\newcommand{\E}{\mathbb{E}}
\newcommand{\Pp}{\mathbb{P}}
\newcommand{\KL}{\mathrm{KL}}

\newcommand{\keff}{K_{\mathrm{eff}}}

\newcommand{\kl}{\operatorname{kl}}

\title{When Is Accuracy Evidence? A Unified Theory of Generalisation, Validation, and Information Fusion}

\author{
  J.M. Gorriz\thanks{Corresponding author}\\
  Data Science and Computational Intelligence Institute\\
  University of Granada\\
  Granada, Spain \\
  \texttt{jg825@cam.ac.uk}\\
  Department of Psychiatry\\
  University of Cambridge\\
  Cambridge, UK\\
       \And
for the Alzheimer's Disease Neuroimaging Initiative (ADNI)\\
}

\begin{document}
\maketitle
\begin{abstract}
K-fold cross-validation (CV) is routinely interpreted as evidence of out-of-sample performance, although the folds are neither independent experiments nor equally informative under data heterogeneity. Cross Upper-Bound Validation (CUBV) was recently introduced to replace point-wise interpretation of CV accuracy by conservative upper bounds on the true risk. Here we generalise that idea through a single exponential framework. We assume that the moment-generating function of a generalisation gap is controlled by a cumulant envelope $\gamma(\lambda)$ and derive a family of bounds of the form
\[
\R\leq \Rh+
\inf_{\lambda>0}
\frac{\gamma(\lambda)+\log(1/\delta)}{\lambda}.
\]
Different choices of $\gamma$ recover or induce Hoeffding-type, variance-aware Bernstein-type, dependency-aware, PAC-Bayesian, and heterogeneous source-fusion bounds. In the K-fold setting, we explicitly model dependence between fold-wise generalisation gaps through a joint sub-Gaussian proxy matrix. Under an equicorrelation model this yields an effective number of folds, $\keff=\frac{K}{1+(K-1)\rho}$ showing that increasing $K$ does not necessarily increase statistical evidence when folds are strongly dependent. We further extend the framework to posterior distributions over predictors and to weighted multi-source fusion, where fusion weights can be selected by minimising an upper bound on future risk rather than empirical error alone. A dedicated fusion experiment evaluates this principle using four non-homogeneously sampled multimodal Gaussian-mixture sources with two interlaced modes per class, allowing the source uncertainty and mixture structure to differ while preserving the underlying binary task. We compare two validation paradigms, K-fold cross-validation and full-sample resubstitution with a risk correction, using genuinely trained linear classifiers on heterogeneous class-conditional Gaussian mixtures designed to mimic small-sample multimodal data with a dominant linear signal and minority sub-populations. Bounds are evaluated by coverage, $\Pp(\R_{\mathrm{true}}\le U)$, and tightness, $U-\R_{\mathrm{true}}$. Across a broad range of low-dimensional multimodal settings, the same structural effect appears under several bound families: when the sample is small relative to the number of modes and dimensions, K-fold partitioning can increase uncertainty because individual folds under-represent minority modes, while a corrected resubstitution estimate can remain valid and tighter; the effect disappears as the sample becomes sufficiently large. Together, the theoretical $\gamma$-CUBV construction and its bounded-discriminability instantiation separate observed performance, uncertainty, dependence, model complexity, and confidence into explicit terms, providing a unified route from cross-validation scores to risk statements and from risk statements to a principled criterion for validating low-dimensional features in heterogeneous small-sample applications such as neuroimaging.
\end{abstract}


\section{Introduction}

Machine learning studies commonly report cross-validation accuracy as evidence that a learned predictor, or a set of variables, latent features, or learned representations, will perform well on unseen data or contains discriminative information. Cross-validation has a long history as a method for predictive assessment and model choice \cite{stone1974}, but its inferential interpretation is more delicate than the point estimate alone suggests. This issue is particularly important in domains such as neuroimaging, where sample sizes are often small, measurements are high-dimensional before feature extraction, and biological heterogeneity may produce multiple sub-populations within the same diagnostic label \cite{varoquaux2017,varoquaux2018}. Accuracy is an empirical average; it conflates an observed value with the uncertainty surrounding it, and it does not by itself quantify the uncertainty associated with the learning process, the partition of the sample, or the future behaviour of the classifier. In particular, small samples, heterogeneous information sources, model-selection effects, and the strong overlap between K-fold training sets can produce optimistic or unstable conclusions. CV is itself a random procedure, and its variance can be substantial when training samples are small or heterogeneous; the dependence induced by overlapping training sets further complicates inference \cite{dietterich1998approximate,nadeau2003,bengio2004,bayle2020,bates2024}. Recent work has also clarified that the estimand targeted by CV need not coincide with the error of the particular model fitted to the complete observed sample, and that naive CV confidence intervals can undercover because fold-level errors are correlated \cite{bates2024}.

Cross Upper-Bound Validation (CUBV) was introduced in \cite{gorriz2026kfold} as a conservative alternative to conventional K-fold CV and as a statistical-learning-theory framework for controlling the uncertainty of machine-learning predictions through upper bounds on the actual risk. Instead of treating an empirical CV score as the final inferential quantity, CUBV upper-bounds the actual risk by means of concentration inequalities and PAC-Bayesian arguments. That work showed that small sample size, low predictor dimension, multimodality, and heterogeneous data sources can destabilise K-fold CV and contribute to excess false-positive conclusions. 

The present work develops, in two complementary parts, a general mathematical extension of that principle and a concrete application of it. Building on an exponential-envelope construction, we then turn the same machinery towards a second, more applied question: when can a low-dimensional feature or representation itself be considered discriminative? We propose the concept of \emph{bounded discriminability}: a feature is discriminative only when the future classification risk it supports can be upper-bounded with controlled confidence, rather than merely judged by empirical accuracy, cross-validated accuracy, visual separation, or post-hoc feature-importance scores. We instantiate four principal members of the $\gamma$-CUBV family - Hoeffding, empirical Bernstein, dependency-aware moment-generating-function, and PAC-Bayes bounds - and additionally include a uniform VC-Hoeffding reference that explicitly pays for post-training hypothesis selection. We further evaluate two conservative treatments of fold dependence: a PSD-inflated dependency bound based on an upper confidence matrix for the fold-gap covariance, and a H\"older construction that combines fold-wise exponential envelopes without factorising dependent moment-generating functions. 

\subsection{Part I: a general exponential framework}

The main idea of the first part of the paper is simple. Let
\begin{equation}
\Delta = \R-\Rh
\end{equation}
denote a generalisation gap. Rather than committing to a single concentration inequality, we assume that its exponential moments satisfy
\begin{equation}
\E\!\left[\exp\{\lambda\Delta\}\right]
\leq
\exp\{\gamma(\lambda)\},
\label{eq:mgf-master}
\end{equation}
where $\gamma$ is a cumulant-generating envelope. Exponential-moment and Chernoff-type arguments provide a standard route from such envelopes to non-asymptotic tail bounds \cite{boucheron2013concentration}. This single object determines the uncertainty penalty associated with the observed risk. Markov's inequality then generates an entire family of risk bounds.

This perspective has several advantages. First, classical concentration bounds become special cases of a common construction. Second, data variance can be introduced explicitly instead of being hidden behind a worst-case range. Third, dependence between folds can be represented rather than ignored. Fourth, the same construction naturally admits PAC-Bayesian complexity penalties. Finally, multi-source information fusion can be formulated as risk fusion plus uncertainty fusion, allowing the fusion weights themselves to be selected by upper-bound minimisation.

\subsection{Part II: bounded feature discriminability}

The second part of the paper begins from a different but related question: \emph{When can a feature itself be considered discriminative?}. Feature importance and feature discriminability are not the same concept. Classical feature-selection work focuses on relevance, redundancy, prediction performance, and the validity of selected variables \cite{guyon2003}, while later work has formalised the stability of feature-selection procedures under data perturbation \cite{nogueira2018}. Methods such as SHAP \cite{lundberg2017shap} and LIME \cite{ribeiro2016lime} quantify the contribution of variables to predictions, while methods such as t-SNE \cite{vandermaaten2008tsne} can visually reveal apparent class separation. More recently, risk-bound criteria have also been proposed for feature selection by linking distributional separation to bounds on Bayes classification error \cite{cao2025}. These approaches address related but distinct questions. Here we ask whether the future classification risk supported by a particular feature or low-dimensional representation can itself be upper-bounded with controlled confidence. We therefore introduce the concept of \emph{bounded discriminability}: a representation is considered discriminative only when the future classification risk it supports admits a controlled upper bound, using the same exponential-envelope machinery developed in Part I.

This perspective leads naturally to a further question. If the purpose of CV is to estimate future risk, is K-fold always the best empirical term on which to build the bound? In small multimodal samples, each K-fold model is trained on only $(K-1)n/K$ observations. As a consequence, rare sub-populations may be represented by very few observations, or even be absent, from a particular training split. Full-sample resubstitution, in contrast, trains on all available observations and therefore retains all modes present in the sample, although its empirical error is optimistically biased because training and evaluation use the same data. The relevant comparison is therefore not between the empirical risks alone, but between their \emph{corrected upper risks}.

\subsection{Contributions}

The principal contributions are:
\begin{enumerate}[label=(\roman*)]
    \item a general $\gamma$-CUBV theorem generated by exponential moment control;
    \item variance-aware CUBV specialisations based on sub-Gaussian and sub-gamma envelopes;
    \item a dependency-aware K-fold bound with an explicit effective number of folds, together with a safe construction for the (more realistic) case in which the dependency matrix must itself be estimated;
    \item a dependence-robust Hölder construction that does not factorise fold moment-generating functions;
    \item a PAC-Bayesian $\gamma$-CUBV theorem separating empirical risk, model complexity, uncertainty, and confidence;
    \item a heterogeneous information-fusion extension in which source weights can be chosen by minimising an upper confidence bound on future risk;
    \item the formulation of feature discriminability as an upper-risk validation problem (\emph{bounded discriminability});
    \item a comparison of K-fold and resubstitution under four distinct upper-bound families using the same trained classifiers and the same synthetic heterogeneous distributions, evaluated by coverage and tightness, identifying the regimes in which corrected resubstitution can be more informative than K-fold without sacrificing the desired confidence level;
    \item a uniform VC-Hoeffding reference for learned affine linear classifiers, derived through symmetrisation, a union bound, and Sauer's lemma, which makes explicit the statistical price of fitting and evaluating on the same sample and serves as a conservative benchmark for post-training resubstitution.
    \item an empirical implementation of the PSD-inflated dependence construction of Section~5.4, used as a deliberately conservative reference for covariance-estimation uncertainty;
    \item an implementation of H\"older-CUBV from Section~5.5, providing a dependence-robust K-fold bound that does not obtain a spurious $K^{-1/2}$ gain by treating overlapping folds as independent.
\end{enumerate}

The aim is not to claim that one bound universally dominates all others,
nor that resubstitution should replace K-fold. Rather, $\gamma$ provides a
common mechanism for encoding the uncertainty assumptions regarded as
compatible with the available evidence, while the resubstitution-versus-K-fold
comparison reveals when data partitioning itself becomes an avoidable source
of uncertainty.

Importantly, the classical concentration inequalities used as building blocks
in Part II are not claimed as new. The methodological novelty lies instead in
the $\gamma$-CUBV formulation: a common exponential-envelope construction from
which a family of generalisation upper bounds can be derived under different
assumptions on variance, dependence, posterior uncertainty, and source
interaction. This formulation makes bounds that are usually treated as
separate tools directly comparable within the same risk-certification
framework. A further contribution is operational. Rather than using these bounds only as
post-hoc guarantees on predictive error, we use the resulting upper risk as a criterion for assessing whether a learned low-dimensional representation
contains statistically defensible discriminative information. In this sense, the bound becomes a discriminability certificate: two representations with similar empirical accuracy may receive substantially different evidence for future discrimination once their uncertainty is taken into account. This
extends the dependency analysis of Part I from a comparison of validation
estimators to a concrete and reproducible criterion for validating learned feature representations.

\section{Preliminaries}

Let $Z=(X,Y)\sim\mathcal{D}$ and let $\ell(h,Z)\in[0,1]$ be a bounded loss associated with predictor $h$; for binary classification under $0$--$1$ loss, $\ell(h,Z)=\mathbf{1}\{h(X)\neq Y\}$. The population risk is
\begin{equation}
\R(h)=\E_{Z\sim\mathcal{D}}\ell(h,Z),
\end{equation}
and, for a sample $S=(Z_1,\ldots,Z_n)$,
\begin{equation}
\Rh_S(h)=\frac{1}{n}\sum_{i=1}^{n}\ell(h,Z_i).
\end{equation}

For a K-fold partition, write $V_k$ for the validation subset of fold $k$, $T_k$ for the corresponding training subset, and $h_k=A(T_k)$ for the model produced by learning algorithm $A$. The fold risk and validation risk are
\begin{align}
\R_k &= \E_{Z\sim\mathcal{D}}\ell(h_k,Z),\\
\Rh_k &= \frac{1}{|V_k|}\sum_{Z_i\in V_k}\ell(h_k,Z_i).
\end{align}
For non-negative weights $w_k$ summing to one,
\begin{align}
\R_{\mathrm{CV}} &= \sum_{k=1}^{K}w_k\R_k,\\
\Rh_{\mathrm{CV}} &= \sum_{k=1}^{K}w_k\Rh_k,
\end{align}
and the fold-wise generalisation gaps are
\begin{equation}
\Delta_k=\R_k-\Rh_k.
\end{equation}

A key point is that $\Delta_1,\ldots,\Delta_K$ are generally dependent because K-fold training sets overlap. Hence, multiplying fold-wise moment-generating functions is not justified unless an independence assumption is explicitly imposed. For symmetric matrices $\bm{A}$ and $\bm{B}$, we write
$\bm{A}\preceq\bm{B}$ when $\bm{B}-\bm{A}$ is positive semidefinite (PSD).
This order is useful below because any PSD upper confidence matrix for the
sub-Gaussian proxy immediately yields a conservative risk bound.

\section{The $\gamma$-CUBV generating principle}
Motivated by standard exponential-moment formulations of concentration inequalities \cite{boucheron2013concentration}, we assume that the generalisation gap admits a deterministic Moment Generating Function (MGF) envelope.
\begin{assumption}[Exponential envelope]
For some interval $\Lambda\subset(0,\infty)$ and deterministic function $\gamma:\Lambda\rightarrow\mathbb{R}$,
\begin{equation}
\E\!\left[e^{\lambda\Delta}\right]
\leq e^{\gamma(\lambda)}
\qquad \forall \lambda\in\Lambda.
\label{ass:gamma}
\end{equation}
\end{assumption}

\begin{theorem}[$\gamma$-CUBV]
Under Assumption~1, for every $\delta\in(0,1)$ and every $\lambda\in\Lambda$, with probability at least $1-\delta$,
\begin{equation}
\R
\leq
\Rh+
\frac{\gamma(\lambda)+\log(1/\delta)}{\lambda}.
\label{eq:gamma-cubv-fixed}
\end{equation}
Consequently,
\begin{equation}
\R
\leq
\Rh+
\inf_{\lambda\in\Lambda}
\frac{\gamma(\lambda)+\log(1/\delta)}{\lambda}.
\label{eq:gamma-cubv}
\end{equation}
\end{theorem}

\begin{proof}
By Markov's inequality,
\begin{align}
\Pp\!\left(
e^{\lambda\Delta}
\geq \frac{e^{\gamma(\lambda)}}{\delta}
\right)
&\leq
\delta e^{-\gamma(\lambda)}
\E[e^{\lambda\Delta}]
\leq\delta.
\end{align}
Therefore, with probability at least $1-\delta$,
\[
\lambda\Delta
\leq
\gamma(\lambda)+\log(1/\delta).
\]
Substituting $\Delta=\R-\Rh$ gives \eqref{eq:gamma-cubv-fixed}. Optimisation over $\lambda$ gives \eqref{eq:gamma-cubv}.
\end{proof}

Equation~\eqref{eq:gamma-cubv} is the central construction of the paper. The choice of $\gamma$ determines what notion of uncertainty is being controlled, and Sections~4--8 instantiate it in turn for range-based, variance-aware, dependency-aware, PAC-Bayesian, and fusion settings; Section~10 then uses these same instantiations to validate individual features.

\section{Classical and variance-aware specialisations}

\subsection{Sub-Gaussian / Hoeffding-type CUBV}

The general $\gamma$-CUBV framework depends on the log-moment generating function of the generalisation gap $\Delta = R-\widehat R$. In general, this function may be unknown or difficult to characterise
exactly. A natural first approximation is therefore to control it by a
quadratic sub-Gaussian envelope. Specifically, we take
\begin{equation}
\gamma(\lambda)
=
\frac{\lambda^2\sigma^2}{2},
\end{equation}
where $\sigma^2$ is a sub-Gaussian variance proxy controlling the
fluctuations of $\Delta$. Then
\begin{equation}
\R
\leq
\Rh+
\frac{\lambda\sigma^2}{2}
+
\frac{\log(1/\delta)}{\lambda}.
\end{equation}
Optimising in $\lambda$ gives
\begin{equation}
\lambda^\star=
\sqrt{\frac{2\log(1/\delta)}{\sigma^2}},
\end{equation}
and therefore
\begin{equation}
\R
\leq
\Rh+
\sqrt{2\sigma^2\log(1/\delta)}.
\label{eq:subgaussian}
\end{equation}

For the average of $n$ independent losses in $[0,1]$, Hoeffding's lemma \cite{hoeffding1963probability} gives $\sigma^2=1/(4n)$, yielding
\begin{equation}
\R
\leq
\Rh+
\sqrt{\frac{\log(1/\delta)}{2n}}.
\label{eq:hoeffding}
\end{equation}
Thus, a Hoeffding-style CUBV is recovered as a special case of the exponential framework; we write $\gamma_H=\sqrt{\log(1/\delta)/(2n)}$ and $U_H=\Rh+\gamma_H$ for the corresponding upper estimate, used in Sections~10--11 below.

\subsection{Variance-aware / Bernstein-CUBV}

Worst-case range information can be unnecessarily conservative when the loss variance is small. A standard variance-sensitive alternative is to control the log-MGF of the generalisation gap by a sub-gamma envelope \cite{boucheron2013concentration}. Specifically, we write \begin{equation}
\gamma(\lambda)
=
\frac{\lambda^2 v}{2(1-c\lambda)},
\qquad 0<\lambda<1/c.
\label{eq:subgamma}
\end{equation}

Here, $v>0$ is a variance-factor parameter controlling the quadratic
part of the fluctuations of the generalisation gap, while $c>0$ is a
scale parameter controlling the contribution of larger deviations.
Thus, $v$ plays a role analogous to a variance proxy, whereas $c$
determines how rapidly the sub-gamma envelope departs from the purely
sub-Gaussian regime. Optimising the resulting $\gamma$-CUBV inequality with respect to $\lambda$ yields the Bernstein-type tail bound
\begin{equation}
\R
\leq
\Rh+
\sqrt{2v\log(1/\delta)}
+
c\log(1/\delta).
\label{eq:bernstein-general}
\end{equation}

For an average of $n$ bounded variables, Bernstein-type inequalities can
be written with slightly different numerical constants depending on the
chosen formulation. Under the parameterisation adopted here
\cite{boucheron2013concentration}, we take
\begin{equation}
v=\frac{\sigma^2}{n},
\qquad
c=\frac{b}{3n},
\end{equation}
where $\sigma^2$ denotes the variance of the individual bounded loss
variable and $b$ is an upper bound on the magnitude of its centred
deviation, i.e.,
\[
|\ell-\mathbb{E}[\ell]|\leq b.
\]

Substituting these quantities into
\eqref{eq:bernstein-general} gives
\begin{equation}
\R
\leq
\Rh+
\sqrt{
\frac{2\sigma^2\log(1/\delta)}{n}
}
+
\frac{b\log(1/\delta)}{3n}.
\label{eq:bernstein}
\end{equation}

We refer to \eqref{eq:bernstein} as \emph{Variance-aware CUBV}
(V-CUBV). Unlike an accuracy-only summary, it distinguishes experiments
with similar average performance but substantially different
variability.

\subsubsection{Empirical-Bernstein CUBV}

When the independent bounded-observation assumptions required by
empirical-Bernstein inequalities \cite{maurer2009bernstein} are
appropriate, the unknown variance can be replaced by an empirical estimate,
\begin{equation}
\widehat V
=
\frac{1}{n-1}
\sum_{i=1}^{n}
(\ell_i-\overline{\ell})^2,
\end{equation}
where
\[
\overline{\ell}
=
\frac{1}{n}\sum_{i=1}^{n}\ell_i
=
\widehat R
\]
is the empirical mean loss. For the $0$--$1$ loss used in our
experiments $\ell_i=\mathbf{1}\{h(x_i)\neq y_i\}$. A representative empirical-Bernstein form is
\begin{equation}
R
\leq
\widehat R+
\sqrt{\frac{2\widehat{V}\log(c_1/\delta)}{n}}
+
\frac{c_2\log(c_1/\delta)}{n},
\label{eq:empbern}
\end{equation}
where the numerical constants depend on the particular empirical-Bernstein
inequality adopted. For the empirical-Bernstein-style parameterisation used in our experiments, we set $c_1=1$ and $c_2=2/3$, obtaining
\begin{equation}
\gamma_B
=
\sqrt{
\frac{2\widehat V\log(1/\delta)}{n}
}
+
\frac{2\log(1/\delta)}{3n},
\qquad
U_B=\widehat R+\gamma_B.
\label{eq:empbern-applied}
\end{equation}

Since the $0$--$1$ risk is bounded above by one, the numerical
implementation uses
\[
U_B^{\mathrm{clip}}=\min\{1,U_B\}.
\]

Equation~\eqref{eq:empbern} should not be applied naively to K-fold
scores as though they were independent observations. In K-fold CV,
dependence must either be incorporated explicitly, as in the next
section, or controlled by a concentration result valid under the
relevant dependence structure.

\section{Dependency-aware K-fold CUBV}

Dependence induced by overlapping training sets is a central difficulty in statistical inference from K-fold CV. Previous work has shown both the limitations of variance estimation from a single K-fold experiment and the need for additional assumptions, resampling structure, or stability conditions to obtain valid uncertainty statements \cite{nadeau2003,bengio2004,bayle2020,bates2024}. The construction below takes a complementary route: rather than estimating fold uncertainty as if the folds were independent replications, it controls the joint exponential behaviour of the fold-wise generalisation gaps.

\subsection{Joint exponential control}

The scalar $\gamma$-CUBV construction extends directly to the joint behaviour of the $K$ fold-wise generalisation gaps. Let
\begin{equation}
\bm{\Delta}
=
(\Delta_1,\ldots,\Delta_K)^\top,
\qquad
\bm{w}
=
(w_1,\ldots,w_K)^\top,
\end{equation}
where
\[
\Delta_k=R_k-\widehat R_k,
\qquad
w_k\geq0,
\qquad
\sum_{k=1}^{K}w_k=1.
\]

The corresponding weighted population and empirical CV risks are
\begin{equation}
R_{\mathrm{CV}}
=
\sum_{k=1}^{K}w_kR_k,
\qquad
\widehat R_{\mathrm{CV}}
=
\sum_{k=1}^{K}w_k\widehat R_k.
\end{equation}
Hence,
\begin{equation}
\Delta_{\mathrm{CV}}
=
R_{\mathrm{CV}}-\widehat R_{\mathrm{CV}}
=
\sum_{k=1}^{K}w_k\Delta_k
=
\bm{w}^{\top}\bm{\Delta}.
\label{eq:cv-gap-vector}
\end{equation}

Thus, although $\bm{\Delta}$ is a $K$-dimensional random vector, the quantity entering the risk bound is again a scalar generalisation gap, namely $\Delta_{\mathrm{CV}}=\bm{w}^{\top}\bm{\Delta}$. Assume that its exponential moment is controlled by a joint envelope $\Gamma_K(\lambda)$,
\begin{equation}
\E\!\left[
\exp\left\{
\lambda\bm{w}^{\top}\bm{\Delta}
\right\}
\right]
\leq
\exp\{\Gamma_K(\lambda)\},
\qquad \lambda>0.
\label{eq:joint-envelope}
\end{equation}

Applying Theorem~1 to the scalar gap
$\Delta_{\mathrm{CV}}=\bm{w}^{\top}\bm{\Delta}$ immediately gives
\begin{equation}
R_{\mathrm{CV}}
\leq
\widehat R_{\mathrm{CV}}
+
\inf_{\lambda>0}
\frac{
\Gamma_K(\lambda)+\log(1/\delta)
}{\lambda}.
\label{eq:joint-cubv}
\end{equation}

No factorisation across folds is required. In particular, the
fold-wise gaps may be dependent, with their joint behaviour absorbed into $\Gamma_K(\lambda)$. The remaining task is therefore to construct a valid envelope $\Gamma_K(\lambda)$ under an appropriate dependence model.

\subsection{Sub-Gaussian dependency model}
A standard way to model jointly concentrated random fluctuations is through sub-Gaussian control of all linear projections
\cite{vershynin2018high}. Accordingly, we assume the following
multivariate MGF domination.

\begin{assumption}[Joint sub-Gaussian gaps]
There exist a mean vector $\bm{\mu}$ and positive semidefinite sub-Gaussian proxy matrix $\bm{\Sigma}$ such that, for every $\bm{t}\in\mathbb{R}^K$,
\begin{equation}
\E\!\left[
e^{\bm{t}^{\top}(\bm{\Delta}-\bm{\mu})}
\right]
\leq
\exp\left\{
\frac{1}{2}
\bm{t}^{\top}\bm{\Sigma}\bm{t}
\right\}.
\label{eq:joint-subg}
\end{equation}
The matrix $\bm{\Sigma}$ is an exponential-moment proxy and need not coincide with the covariance matrix $\mathrm{Cov}(\bm{\Delta})$.
\end{assumption}

\begin{proposition}[Dependency-aware CUBV]
Under Assumption~2, with probability at least $1-\delta$,
\begin{equation}
\R_{\mathrm{CV}}
\leq
\Rh_{\mathrm{CV}}
+
\bm{w}^{\top}\bm{\mu}
+
\sqrt{
2\bm{w}^{\top}\bm{\Sigma}\bm{w}
\log(1/\delta)
}.
\label{eq:dep-cubv}
\end{equation}
\end{proposition}

\begin{proof}
Set $\bm{t}=\lambda\bm{w}$ in \eqref{eq:joint-subg}. Then
\[
\Gamma_K(\lambda)
\leq
\lambda\bm{w}^{\top}\bm{\mu}
+
\frac{\lambda^2}{2}
\bm{w}^{\top}\bm{\Sigma}\bm{w}.
\]
Substitution into \eqref{eq:joint-cubv} and optimisation over $\lambda$ produces \eqref{eq:dep-cubv}.
\end{proof}

The matrix term expands as
\begin{equation}
\bm{w}^{\top}\bm{\Sigma}\bm{w}
=
\sum_{k=1}^{K}w_k^2\Sigma_{kk}
+
2\sum_{i<j}w_iw_j\Sigma_{ij}.
\label{eq:cov-expansion}
\end{equation}
Thus, the uncertainty of the CV average depends not only on fold-wise uncertainty but also on cross-fold dependence. With equal fold weights $\bm{w}=K^{-1}\bm{1}$, and writing $\mu=\bm{w}^\top\bm\mu$, $\Sigma=\mathrm{Cov}(\bm\Delta)$, this specialises to the applied dependency-aware bound
\begin{equation}
U_{\Sigma}
=
\Rh_{\mathrm{CV}}
+
\bm{w}^{\top}\bm\mu
+
\sqrt{2\bm{w}^{\top}\bm\Sigma\bm{w}\log(1/\delta)},
\label{eq:dep-applied}
\end{equation}
used, with $\bm\mu$ and $\bm\Sigma$ calibrated on a disjoint Monte Carlo subset, in Section 10.

\subsection{Equicorrelation and the effective number of folds}

For equal weights $w_k=1/K$, assume
\begin{align}
\Sigma_{kk} &= \sigma^2,\\
\Sigma_{ij} &= \rho\sigma^2,\qquad i\neq j.
\end{align}
Then
\begin{equation}
\bm{w}^{\top}\bm{\Sigma}\bm{w}
=
\frac{\sigma^2}{K}
[1+(K-1)\rho].
\end{equation}
When $\bm{w}^{\top}\bm{\mu}=0$ (or after accounting separately for the bias term), Proposition~1 reduces to
\begin{equation}
\R_{\mathrm{CV}}
\leq
\Rh_{\mathrm{CV}}
+
\sqrt{
\frac{
2\sigma^2[1+(K-1)\rho]\log(1/\delta)
}{K}
}.
\label{eq:equicorr}
\end{equation}

Define
\begin{equation}
\keff
=
\frac{K}{1+(K-1)\rho}.
\label{eq:keff}
\end{equation}
Equation~\eqref{eq:equicorr} becomes
\begin{equation}
\R_{\mathrm{CV}}
\leq
\Rh_{\mathrm{CV}}
+
\sqrt{
\frac{
2\sigma^2\log(1/\delta)
}{\keff}
}.
\label{eq:keff-bound}
\end{equation}

Several limiting cases are immediate:
\begin{align}
\rho=0
&\quad\Rightarrow\quad
\keff=K,\\
\rho\rightarrow1
&\quad\Rightarrow\quad
\keff\rightarrow1.
\end{align}
Hence, adding highly dependent folds need not provide the $K^{-1/2}$ uncertainty reduction associated with independent replications.

\begin{remark}[Interpretation of $\rho$]
The theoretical parameter $\rho$ in \eqref{eq:equicorr} describes dependence between fold-wise \emph{generalisation gaps}, not merely correlation between observed fold accuracies. The latter may be used as an empirical proxy only if its relationship to gap dependence is established. This distinction is fundamental because the variance of K-fold CV cannot, in general, be estimated unbiasedly from a single K-fold experiment \cite{bengio2004,nadeau2003}.
\end{remark}

\subsection{Safe uncertainty inflation when the dependency matrix is estimated}

A direct plug-in rule $\bm{\Sigma}\leftarrow\widehat{\bm{\Sigma}}$ does not
automatically preserve the coverage of \eqref{eq:dep-cubv}. A conservative
alternative is to estimate a PSD \emph{upper confidence matrix}. This viewpoint is related to the broader matrix-concentration literature, in which operator-norm or semidefinite controls are used to obtain finite-sample guarantees for random matrices \cite{tropp2015}.

\begin{definition}[PSD upper confidence matrix]
Let $\eta\in(0,1)$. A random symmetric matrix $\bm{U}_{\eta}$ is called a
$(1-\eta)$ PSD upper confidence matrix for the sub-Gaussian proxy
$\bm{\Sigma}$ when
\begin{equation}
\Pp\!\left(
\bm{\Sigma}\preceq\bm{U}_{\eta}
\right)\geq 1-\eta.
\label{eq:psd-upper}
\end{equation}
\end{definition}
\footnote{Here, $\bm{A}\preceq\bm{B}$ denotes the Loewner partial order for
symmetric matrices, meaning that
$\bm{B}-\bm{A}$ is positive semidefinite. Equivalently,$\bm{x}^{\top}\bm{A}\bm{x}\leq\bm{x}^{\top}\bm{B}\bm{x}$
for every vector $\bm{x}$.}

\begin{theorem}[Safe Dependency-CUBV with estimated dependence]
Assume that the conditions of Proposition~1 hold and that an auxiliary
calibration procedure produces $\bm{U}_{\eta}$ satisfying
\eqref{eq:psd-upper}. Then, with probability at least $1-\delta-\eta$,
\begin{equation}
\R_{\mathrm{CV}}
\leq
\Rh_{\mathrm{CV}}
+
\bm{w}^{\top}\bm{\mu}
+
\sqrt{
2\bm{w}^{\top}\bm{U}_{\eta}\bm{w}
\log(1/\delta)
}.
\label{eq:safe-dep}
\end{equation}
\end{theorem}

\begin{proof}
Let
\[
\mathcal{E}_{\Sigma}
=
\{
\bm{\Sigma}\preceq\bm{U}_{\eta}
\}.
\]
On $\mathcal{E}_{\Sigma}$,
\[
\bm{w}^{\top}\bm{\Sigma}\bm{w}
\leq
\bm{w}^{\top}\bm{U}_{\eta}\bm{w}.
\]
Proposition~1 fails with probability at most $\delta$ when the true proxy is
used, while $\mathcal{E}_{\Sigma}$ fails with probability at most $\eta$.
The union bound therefore gives total failure probability at most
$\delta+\eta$.
\end{proof}

This theorem separates two uncertainty budgets:
$\delta$ controls the risk-tail event and $\eta$ controls uncertainty in the
dependency estimate. It therefore avoids treating an estimated covariance
matrix as if it were known exactly.

\begin{corollary}[Operator-norm inflation]
If an estimator satisfies
\begin{equation}
\Pp\!\left(
\|\widehat{\bm{\Sigma}}-\bm{\Sigma}\|_{\mathrm{op}}
\leq \varepsilon_{\eta}
\right)
\geq1-\eta,
\end{equation}
then
\begin{equation}
\bm{U}_{\eta}
=
\widehat{\bm{\Sigma}}
+
\varepsilon_{\eta}\bm{I}
\end{equation}
is a valid PSD upper confidence matrix, and
\begin{equation}
\R_{\mathrm{CV}}
\leq
\Rh_{\mathrm{CV}}
+
\bm{w}^{\top}\bm{\mu}
+
\sqrt{
2\left(
\bm{w}^{\top}\widehat{\bm{\Sigma}}\bm{w}
+
\varepsilon_{\eta}\|\bm{w}\|_2^2
\right)
\log(1/\delta)
}.
\label{eq:operator-safe}
\end{equation}
\end{corollary}

\begin{proof}
\footnote{Here, $\|\bm{A}\|_{\mathrm{op}}$ denotes the operator (spectral) norm of $\bm{A}$,
defined as $\|\bm{A}\|_{\mathrm{op}}=\sup_{\|\bm{x}\|_2=1}\|\bm{A}\bm{x}\|_2$
For symmetric $\bm{A}$, this is equal to $\|\bm{A}\|_{\mathrm{op}}=\max_i |\lambda_i(\bm{A})|$
where $\lambda_i(\bm{A})$ denotes the $i$th eigenvalue of $\bm{A}$.}
The operator-norm event implies
$-\varepsilon_{\eta}\bm{I}\preceq
\widehat{\bm{\Sigma}}-\bm{\Sigma}
\preceq\varepsilon_{\eta}\bm{I}$,
hence
$\bm{\Sigma}\preceq
\widehat{\bm{\Sigma}}+\varepsilon_{\eta}\bm{I}$.
The result follows from Theorem~2.
\end{proof}

\subsubsection{A finite-sample illustration under bounded calibration gaps}

The preceding result is deliberately abstract: it only requires a valid
operator-norm radius. To make the construction explicit, suppose an
\emph{independent calibration experiment} provides $B$ independent centred
gap vectors
\[
\bm{X}^{(b)}=\bm{\Delta}^{(b)}-\bm{\mu},
\qquad b=1,\ldots,B,
\]
with $|X^{(b)}_k|\leq1$, and suppose the covariance matrix is an admissible
sub-Gaussian proxy. Define
\begin{equation}
\widehat{\bm{\Sigma}}
=
\frac{1}{B}
\sum_{b=1}^{B}
\bm{X}^{(b)}\bm{X}^{(b)\top}.
\end{equation}
For every matrix entry, the random variable
$X_i^{(b)}X_j^{(b)}$ lies in $[-1,1]$. Hoeffding's inequality and a union
bound over the $K^2$ entries imply that, with probability at least $1-\eta$,
\begin{equation}
\max_{i,j}
\left|
\widehat{\Sigma}_{ij}-\Sigma_{ij}
\right|
\leq
t_{\eta,B,K},
\qquad
t_{\eta,B,K}
=
\sqrt{
\frac{2\log(2K^2/\eta)}{B}
}.
\label{eq:entrywise-radius}
\end{equation}
Since $\|\bm{A}\|_{\mathrm{op}}\leq K\|\bm{A}\|_{\max}$ for a
$K\times K$ matrix,
\begin{equation}
\varepsilon_{\eta}
=
K t_{\eta,B,K}
\label{eq:crude-op-radius}
\end{equation}
is a valid, though generally loose, finite-sample radius.

This construction is primarily a proof-of-principle. In real K-fold CV the
generalisation gaps are latent, so independent calibration gap vectors are not
available from the K validation scores alone. This is consistent with the
negative result of Bengio and Grandvalet \cite{bengio2004}. Practically useful
estimators therefore require additional information, for example an untouched
audit set, repeated learning experiments with an external test population,
a justified model-based covariance structure, or a problem-specific
algorithmic-stability bound. Repeated CV by itself should be viewed as a
diagnostic source of dependence information rather than a universally valid
covariance estimator, a point revisited empirically in Section~10.

\subsection{A dependence-robust H\"older construction}

There is a second route that avoids independence without specifying a covariance matrix. H\"older-type inequalities have previously been used to derive concentration and correlation bounds under dependence \cite{pelekis2017}; here we use the same underlying inequality to combine fold-wise exponential envelopes without factorising the dependent K-fold MGF. Write
\begin{equation}
\E\exp\left\{
\lambda\sum_{k=1}^{K}w_k\Delta_k
\right\}
=
\E\prod_{k=1}^{K}
e^{\lambda w_k\Delta_k}.
\end{equation}
Let $p_k>1$ satisfy
\begin{equation}
\sum_{k=1}^{K}\frac{1}{p_k}=1.
\end{equation}
H\"older's inequality implies
\begin{equation}
\E\prod_{k=1}^{K}
e^{\lambda w_k\Delta_k}
\leq
\prod_{k=1}^{K}
\left(
\E e^{p_k\lambda w_k\Delta_k}
\right)^{1/p_k}.
\end{equation}
If each fold satisfies
\begin{equation}
\E e^{t\Delta_k}
\leq
e^{\gamma_k(t)},
\end{equation}
then
\begin{equation}
\Gamma_K(\lambda)
\leq
\sum_{k=1}^{K}
\frac{1}{p_k}
\gamma_k(p_k\lambda w_k).
\label{eq:holder-cgf}
\end{equation}
Consequently,
\begin{equation}
\R_{\mathrm{CV}}
\leq
\Rh_{\mathrm{CV}}
+
\inf_{\lambda>0}
\frac{
\sum_k p_k^{-1}\gamma_k(p_k\lambda w_k)
+
\log(1/\delta)
}{\lambda}.
\label{eq:holder-bound}
\end{equation}

For equal weights and $p_k=K$,
\begin{equation}
\Gamma_K(\lambda)
\leq
\frac{1}{K}
\sum_{k=1}^{K}\gamma_k(\lambda).
\end{equation}
The resulting bound is conservative but does not obtain a fictitious gain by factorising dependent folds.

\section{PAC-Bayesian $\gamma$-CUBV}

PAC-Bayesian bounds control the generalisation of data-dependent distributions over predictors through a prior--posterior complexity term, with foundational formulations due to McAllester and subsequent refinements and expositions \cite{mcallester1999pac,seeger2002,alquier2024}. Let $P$ be a prior distribution over predictors that is independent of the observed sample, and let $Q$ be a data-dependent posterior.

\begin{assumption}[Uniform prior-averaged exponential control]
For each $\lambda\in\Lambda$,
\begin{equation}
\E_S\E_{h\sim P}
\exp\left\{
\lambda[
\R(h)-\Rh_S(h)
]
\right\}
\leq
e^{\gamma(\lambda)}.
\label{eq:pb-mgf}
\end{equation}
\end{assumption}

\begin{theorem}[PAC-Bayesian $\gamma$-CUBV]
Under Assumption~3, for any $\delta\in(0,1)$ and fixed $\lambda\in\Lambda$, with probability at least $1-\delta$ over $S$, simultaneously for all posteriors $Q\ll P$,
\begin{equation}
\R(Q)
\leq
\Rh(Q)
+
\frac{
\KL(Q\|P)+
\gamma(\lambda)+
\log(1/\delta)
}{\lambda}.
\label{eq:pac-gamma}
\end{equation}
Here
\begin{align}
\R(Q)&=\E_{h\sim Q}\R(h),\\
\Rh(Q)&=\E_{h\sim Q}\Rh(h).
\end{align}
\end{theorem}

\begin{proof}
For a realised sample $S$, the Donsker--Varadhan change-of-measure inequality \cite{donsker1975} gives
\begin{align}
\lambda \E_{h\sim Q}[
\R(h)-\Rh_S(h)]
\leq
\KL(Q\|P)
+
\log
\E_{h\sim P}
e^{\lambda[\R(h)-\Rh_S(h)]}.
\end{align}
By Markov's inequality and \eqref{eq:pb-mgf}, with probability at least $1-\delta$,
\begin{equation}
\log
\E_{h\sim P}
e^{\lambda[\R(h)-\Rh_S(h)]}
\leq
\gamma(\lambda)+\log(1/\delta).
\end{equation}
Combining both inequalities gives \eqref{eq:pac-gamma}.
\end{proof}

The bound separates four conceptually distinct components:
\begin{equation}
\R(Q)
\leq
\underbrace{\Rh(Q)}_{\text{observed error}}
+
\underbrace{\frac{\KL(Q\|P)}{\lambda}}_{\text{model complexity}}
+
\underbrace{\frac{\gamma(\lambda)}{\lambda}}_{\text{data uncertainty}}
+
\underbrace{\frac{\log(1/\delta)}{\lambda}}_{\text{confidence}}.
\label{eq:decomposition}
\end{equation}

The optimisation of $\lambda$ requires care. If a continuum of values is chosen after observing the data, a valid union bound, discretisation, hyper-prior, or a theorem allowing data-dependent $\lambda$ should be used rather than silently taking a post-hoc infimum. In the square-root parameterisation \cite{maurer2004pac} used in the applied experiments of Section 10, for isotropic Gaussians $P=\mathcal N(0,\sigma_P^2 I)$ and $Q=\mathcal N(\widehat\theta,\sigma_Q^2 I)$ around a trained parameter vector $\widehat\theta\in\mathbb{R}^m$,
\begin{equation}
\KL(Q\|P)
=
\frac{1}{2}
\left[
\frac{\|\widehat\theta\|^2}{\sigma_P^2}
+
m
\left(
\frac{\sigma_Q^2}{\sigma_P^2}
-1
-\log
\frac{\sigma_Q^2}{\sigma_P^2}
\right)
\right].
\label{eq:kl-gaussian}
\end{equation}
For the finite-sample PAC--Bayesian exponential control, we adopt the refinement in~\cite{maurer2004pac}, which provides the corresponding exponential envelope and yields, for a sample of size $n$, the following bound.\footnote{See Appendix~\ref{app:pac-bayes-derivation} for a complete derivation.}
\begin{equation}
U_{PB}
=
\Rh(Q)
+
\sqrt{
\frac{
\KL(Q\|P)
+
\log(2\sqrt n/\delta)
}{
2n
}
}.
\label{eq:pac-applied}
\end{equation}
Thus, Eq.~\eqref{eq:pac-applied} provides the closed-form PAC--Bayesian specialisation used in our experiments within the general exponential-control principle of Theorem~3. Here, $\Rh(Q)$ denotes the empirical Gibbs risk and is approximated by Monte Carlo sampling $\theta^{(s)}\sim Q$, $s=1,\ldots,S$. For $K$-fold validation, a posterior is constructed around the classifier trained in each fold, the fold confidence level is adjusted to $\delta/K$, and the resulting fold-wise upper risks are averaged.

\section{Heterogeneous information-fusion CUBV}

Combining multiple sources or domains by non-uniform weighting has a substantial literature, including theoretical analyses of multi-source adaptation showing that the choice of source weights can be critical for generalisation \cite{mansour2008}. Our objective is different but related: the weights below are chosen using a CUBV criterion that combines empirical source risk with uncertainty and cross-source dependence.

Consider $M$ information sources or domains with source-specific empirical and population risks,
\begin{align}
\Rh_m,\qquad \R_m,\qquad m=1,\ldots,M.
\end{align}
Let $\bm{a}$ be non-negative fusion weights satisfying
\begin{equation}
a_m\geq0,
\qquad
\sum_{m=1}^{M}a_m=1.
\end{equation}
Define
\begin{align}
\R_F &= \sum_{m=1}^{M}a_m\R_m,\\
\Rh_F &= \sum_{m=1}^{M}a_m\Rh_m.
\end{align}

For the vector of source-specific gaps
\begin{equation}
\bm{\Delta}_F=
(\R_1-\Rh_1,\ldots,\R_M-\Rh_M)^\top,
\end{equation}
let
\begin{equation}
\Gamma_F(\lambda)
=
\log\E
e^{\lambda\bm{a}^{\top}\bm{\Delta}_F}.
\end{equation}
Then
\begin{equation}
\R_F
\leq
\Rh_F+
\inf_{\lambda>0}
\frac{
\Gamma_F(\lambda)+\log(1/\delta)
}{\lambda}.
\label{eq:fusion-general}
\end{equation}

Under a zero-mean joint sub-Gaussian envelope with proxy matrix $\bm{\Sigma}_F$,
\begin{equation}
\R_F
\leq
\Rh_F+
\sqrt{
2\bm{a}^{\top}\bm{\Sigma}_F\bm{a}
\log(1/\delta)
}.
\label{eq:fusion-subg}
\end{equation}

Equation~\eqref{eq:fusion-subg} suggests a bound-aware fusion rule:
\begin{equation}
\bm{a}^{\star}
=
\arg\min_{\bm{a}\in\Delta_M}
\left[
\bm{a}^{\top}\widehat{\bm{R}}
+
\sqrt{
2\log(1/\delta)
\,\bm{a}^{\top}\bm{\Sigma}_F\bm{a}
}
\right],
\label{eq:fusion-opt}
\end{equation}
where
\begin{equation}
\Delta_M=
\left\{
\bm{a}\in\mathbb{R}^M:
a_m\geq0,\;
\sum_m a_m=1
\right\}.
\end{equation}

This criterion differs fundamentally from empirical-risk-optimal fusion. A source can have low empirical error but high uncertainty or strong redundancy with other sources. Bound-aware fusion penalises both effects.

\section{Summary of the proposed methods}

Table~\ref{tab:methods} summarises the main members of the framework developed in Part I; Section~10 shows how four of them (Hoeffding-CUBV, V-CUBV/Empirical-Bernstein-CUBV, Dependency-CUBV, and PAC-$\gamma$-CUBV) are put to work as a criterion for bounded feature discriminability.

\begin{table*}[t]
\centering
\caption{Representative members of the $\gamma$-CUBV family.}
\label{tab:methods}
\begin{tabular}{p{3.2cm}p{5.0cm}p{7.2cm}}
\toprule
Method & Uncertainty model & Main purpose\\
\midrule
Hoeffding-CUBV &
$\gamma(\lambda)=\lambda^2\sigma^2/2$ with worst-case bounded-loss proxy &
Baseline range-based concentration; recovers a standard CUBV-type upper confidence penalty.\\[1mm]

V-CUBV / Bernstein-CUBV &
Sub-gamma envelope with variance factor $v$ and scale $c$ &
Distinguishes models or experiments having similar mean error but different uncertainty.\\[1mm]

Dependency-CUBV &
Joint sub-Gaussian proxy matrix $\bm{\Sigma}$ for fold gaps &
Accounts explicitly for overlap and dependence among K-fold experiments; introduces $\keff$.\\[1mm]

H\"older-CUBV &
Fold-wise exponential envelopes combined through H\"older's inequality &
Provides a conservative construction without multiplying moment-generating functions of dependent folds.\\[1mm]

PAC-$\gamma$-CUBV &
Exponential envelope plus $\KL(Q\|P)$ complexity &
Separates predictive error, model complexity, data uncertainty, and confidence.\\[1mm]

Fusion-CUBV &
Joint uncertainty matrix across information sources &
Selects fusion weights by minimising an upper bound on risk rather than empirical error alone.\\
\bottomrule
\end{tabular}
\end{table*}

\begin{algorithm}[t]
\caption{Generic $\gamma$-CUBV evaluation}
\label{alg:gamma}
\begin{algorithmic}[1]
\Require Data $S$, learning algorithm $A$, confidence $\delta$, K-fold partition, selected uncertainty model.
\For{$k=1,\ldots,K$}
    \State Fit $h_k\gets A(T_k)$.
    \State Compute validation risk $\Rh_k$ on $V_k$.
\EndFor
\State Compute $\Rh_{\mathrm{CV}}=\sum_k w_k\Rh_k$.
\State Specify or estimate a valid exponential envelope:
\Statex \hspace{\algorithmicindent} $\gamma(\lambda)$ for a scalar bound, or
\Statex \hspace{\algorithmicindent} $\Gamma_K(\lambda)$ / $\bm{\Sigma}$ for dependent folds.
\State Compute the uncertainty penalty at level $\delta$.
\State Return $U_\delta=\Rh_{\mathrm{CV}}+\text{uncertainty penalty}$.
\end{algorithmic}
\end{algorithm}

\section{Bounded feature discriminability}

We now turn the exponential-envelope construction of Sections~3--8 towards a concrete applied question. Let
\begin{equation}
S=\{(z_i,y_i)\}_{i=1}^n
\end{equation}
be a sample drawn from an unknown population distribution $\mathcal D$, where $z_i\in\mathbb R^d$ is a low-dimensional feature vector, $d$ denotes the feature dimension, and $y_i\in\{0,1\}$. Given this representation $Z$, we train a simple classifier $h$, calculate an empirical risk $\Rh$, and assess the representation through an upper confidence bound
\begin{equation}
U(\delta)=\Rh+B(\delta),
\end{equation}
where $B$ is one of the uncertainty corrections of Sections~4--6, derived from the general $\gamma$-CUBV theorem. A feature is useful not merely because $\Rh$ is low, but because the future risk compatible with the observed evidence is also low. We call
\begin{equation}
\mathrm{Tightness}
=
U-\R_{\mathrm{true}}
\end{equation}
the excess width of the bound relative to the true population risk (in simulations $\R_{\mathrm{true}}$ is approximated from a large independent population sample), and we evaluate every bound by both coverage, $\Pp(\R_{\mathrm{true}}\le U)$, and tightness.

For comparing resubstitution and K-fold we define
\begin{equation}
\Delta_{\mathrm{comp}}
=
\left(
U_{K\text{-fold}}-\R_{\mathrm{true},K}
\right)
-
\left(
U_{\mathrm{resub}}-\R_{\mathrm{true,resub}}
\right).
\end{equation}
A positive $\Delta_{\mathrm{comp}}$ favours resubstitution only if the resubstitution bound preserves the target coverage,
\begin{equation}
\Pp
\left(
\R_{\mathrm{true,resub}}
\le
U_{\mathrm{resub}}
\right)
\ge 1-\delta.
\end{equation}

\subsection{Validation paradigms}

\subsubsection{Resubstitution}

A classifier is trained on the complete sample $S$, $\widehat h=\mathcal A(S)$, where $\mathcal A$ denotes the learning algorithm. The resubstitution risk is
\begin{equation}
\Rh_{\mathrm{resub}}
=
\frac{1}{n}
\sum_{i=1}^{n}
\mathbf 1\{
\widehat h(z_i)\ne y_i
\}.
\end{equation}
Because the same data are used to fit and evaluate the classifier, this quantity is generally optimistic and cannot be interpreted as an unbiased estimate of future risk. The question studied here is whether a sufficiently conservative correction, in the sense of Sections~4--6, can transform it into a useful upper-risk estimate.

\subsubsection{K-fold cross-validation}

The sample is divided into $K$ disjoint validation folds. For fold $k$, the classifier $\widehat h^{(-k)}$ is trained on the other $K-1$ folds and evaluated on fold $k$. The K-fold empirical risk is
\begin{equation}
\Rh_K
=
\frac{1}{K}
\sum_{k=1}^{K}
\Rh_k.
\end{equation}
The fold estimators are not independent because the training sets overlap (Section~5). In heterogeneous small samples, fold composition may also vary substantially, causing individual sub-populations to be poorly represented.

\subsection{Instantiated upper bounds}

For resubstitution we use the scalar analogue of the dependency-aware bound with
\begin{equation}
G_{\mathrm{resub}}
=
\R_{\mathrm{true}}
-
\Rh_{\mathrm{resub}}.
\end{equation}
The four principal bound families instantiated for the experiments below are: Hoeffding-CUBV, $U_H=\Rh+\gamma_H$ with $\gamma_H$ as in \eqref{eq:hoeffding}; Empirical-Bernstein CUBV, $U_B=\Rh+\gamma_B$ with $\gamma_B$ as in \eqref{eq:empbern-applied}; the dependency-aware MGF bound $U_\Sigma$ of \eqref{eq:dep-applied}, with the gap mean and covariance calibrated on a Monte Carlo subset disjoint from the evaluation subset, to prevent using the same simulated experiments both to estimate the dependence correction and to measure its performance; and the PAC-Bayes bound $U_{PB}$ of \eqref{eq:pac-applied}, which avoids the artificial construction of a finite bank of random linear classifiers and targets the Gibbs classifier rather than exactly the deterministic posterior-mean classifier. In addition, we report a uniform VC-Hoeffding reference for the resubstitution classifier, a PSD-inflated dependency construction motivated by Section~5.4, and a H\"older-Hoeffding K-fold bound derived from Section~5.5. The VC-Hoeffding reference controls the entire affine-linear hypothesis class and therefore explicitly pays for data-dependent classifier fitting. The PSD-inflated rule is included as an intentionally conservative illustration of covariance-estimation uncertainty, while H\"older-CUBV avoids factorising dependent fold MGFs. Detailed derivations are given in the appendices.

\subsection{Feature and subset scores}

The upper-risk framework applies directly to individual features or low-dimensional feature subsets. For a single feature $z_j$, let $h_j$ be a classifier trained using only that feature; define
\begin{equation}
U_j(\delta)
=
\Rh_j+B_j(\delta).
\end{equation}
For binary classification with balanced classes, the chance-level risk is $\R_{\mathrm{chance}}=0.5$, and a natural sufficient condition for bounded discriminability is
\begin{equation}
U_j(\delta)<0.5.
\end{equation}
More generally, for a feature subset $S$,
\begin{equation}
U_S(\delta)
=
\Rh_S+B_S(\delta),
\end{equation}
and one may rank candidate representations by $U_S$ rather than by empirical accuracy alone, or define a bounded discriminability score
\begin{equation}
\mathrm{BDI}(S;\delta)
=
1-U_S(\delta).
\end{equation}
This construction is conceptually distinct from attribution. SHAP or LIME may indicate that a feature contributes strongly to a model prediction; bounded discriminability asks a different question: whether the classification information associated with that feature is supported by a controlled upper bound on future risk.

\section{Experimental evaluation}
We designed a sequence of controlled Monte Carlo experiments to study
different aspects of the proposed $\gamma$-CUBV framework. Rather than
varying several sources of difficulty simultaneously, the experiments
progressively isolate the effects of sample size, structural complexity,
fold dependence, and heterogeneous information fusion.

All binary classification experiments use genuinely trained
$\ell_2$-regularised logistic-regression classifiers after feature
standardisation. Population risk is approximated using a large
independent synthetic test sample. Unless otherwise stated, we use
$K=5$, confidence level $1-\delta=0.95$, and 60 Monte Carlo
repetitions.

\subsection{Common simulation framework}
\subsubsection{Multimodal data-generating process}

The baseline synthetic problem consists of balanced binary classes
generated from class-conditional Gaussian mixtures. The total number
of Gaussian components across both classes is denoted by $N_c$, with
$N_c/2$ components assigned to each class.

Within each class, one dominant component has mixture weight
$\pi_{\mathrm{major}}=0.6$, while the remaining probability mass is
distributed across the minority components. The dominant class means
are separated along the first coordinate, whereas the minority modes
are placed closer to the decision boundary and displaced in the
remaining coordinates.

Unless otherwise stated, the baseline generator uses
\[
\pi_{\mathrm{major}}=0.6,\qquad
a=1.5,\qquad
\sigma_{\mathrm{cluster}}=1.2.
\]

This construction produces heterogeneous multimodal classes while
preserving a globally exploitable linear signal. It therefore allows
the effect of finite-sample uncertainty and fold composition to be
studied without forcing the linear classifier into a near-chance
non-linear regime.

\subsubsection{Learning and population-risk estimation}

All experiments use an $\ell_2$-regularised logistic-regression
classifier preceded by feature standardisation. A new classifier is
genuinely trained for every Monte Carlo repetition and, in the
cross-validation experiments, independently within every training
fold.

For each generated training sample, population risk is approximated
using an independent test sample of $10\,000$ observations drawn from
the same data-generating distribution. For the full-sample estimator,
the classifier is trained on all $n$ observations and its
resubstitution risk is computed on the same training sample. For
$K$-fold cross-validation, we use stratified $K=5$ folds and average
the held-out empirical risks across folds.

For the calibrated MGF/Dependency constructions, Monte Carlo
repetitions are divided into disjoint calibration and evaluation
subsets. The calibration repetitions are used to estimate the mean
generalisation gap and, for the fold-dependent construction, the
covariance matrix of the fold-wise gaps. The resulting quantities are
then kept fixed when evaluating the bound on the held-out Monte Carlo
repetitions. This separation prevents the reported evaluation sample
from being reused to calibrate the dependency correction.

\subsubsection{Evaluation criteria}

We evaluate each bound along several complementary dimensions. Let
$U$ denote the reported upper risk, $\widehat R$ the corresponding
empirical risk, and $R_{\mathrm{true}}$ the population risk estimated
on the independent test sample.

Empirical coverage is defined as
\begin{equation}
\widehat{\mathrm{Cov}}
=
\frac{1}{B_{\mathrm{eval}}}
\sum_{b=1}^{B_{\mathrm{eval}}}
\mathbf{1}
\left\{
R_{\mathrm{true}}^{(b)}
\leq
U^{(b)}
\right\},
\end{equation}
and is compared with the nominal confidence level $1-\delta=0.95$.

Bound tightness is measured by
\begin{equation}
T = U-R_{\mathrm{true}},
\end{equation}
so that smaller positive values indicate more informative upper
bounds, provided that coverage is maintained.

We also report the generalisation gap
\begin{equation}
G = R_{\mathrm{true}}-\widehat R,
\end{equation}
and the bound penalty
\begin{equation}
P = U-\widehat R.
\end{equation}

Two additional criteria quantify whether a bound is operationally
informative. The non-triviality rate is
\begin{equation}
\Pr(U<1),
\end{equation}
whereas, for the balanced binary problems considered here, the
bounded-discriminability certification rate is
\begin{equation}
\Pr(U<0.5).
\end{equation}
The latter measures how often the available data are sufficient to
certify, at the prescribed confidence level, a future classification
risk below chance.

Because the reported quantities are estimated from repeated Monte Carlo
experiments, we additionally quantify simulation uncertainty. For
continuous Monte Carlo summaries, such as mean upper risk or mean
tightness, we report 95\% confidence intervals of the form
\[
\bar q \pm 1.96\,\frac{s_q}{\sqrt{B_{\mathrm{eval}}}},
\]
where $s_q$ is the empirical standard deviation across the independent
evaluation repetitions. Because coverage and certification are
estimated as binomial proportions, their corresponding 95\% confidence
intervals are computed using the Wilson score interval~\cite{wilson1927}. These intervals quantify Monte Carlo uncertainty and should not be confused with the
generalisation confidence level $1-\delta$ defining the
$\gamma$-CUBV risk bound itself.

\subsection{Experiment I: Effect of sample size}

We first isolate the effect of sample size while keeping the underlying
classification problem fixed. We set $d=2$, $N_c=4$, and vary $n\in\{50,100,200,500,1000,2000,5000\}$. For every sample size, 60 independent Monte Carlo repetitions are generated. The underlying population distribution, classifier family, and data-generating parameters remain unchanged across the sample-size sweep. Consequently, this experiment isolates the reduction in
statistical uncertainty produced by increasing the amount of available
data. For resubstitution, we evaluate Hoeffding-CUBV, Empirical-Bernstein CUBV, VC-Hoeffding, PAC-Bayes, and the calibrated MGF construction. For $K$-fold validation, we evaluate Hoeffding-CUBV, Empirical-Bernstein CUBV, PAC-Bayes, calibrated MGF/Dependency, Safe-MGF/Dependency, and H"older-Hoeffding CUBV. The main-text figures focus on the resubstitution results for clarity; the corresponding $K$-fold comparisons, including bound tightness and empirical coverage of the dependence-aware constructions, are reported in Appendix~\ref{app:kfold-comparisons}.

\subsubsection{Results}

The population risk remained essentially stable across the complete sample-size sweep, from approximately $0.229$ at $n=50$ to $0.220$ at $n=5000$. This confirms that the experiment primarily modifies the amount of statistical evidence available to the learner rather than the intrinsic difficulty of the classification problem.

Figure~\ref{fig:exp1_upper_bounds} shows the evolution of the resubstitution upper risks as the available sample size increases. The shaded regions represent 95\% Monte Carlo confidence intervals for the estimated mean upper risk. These intervals quantify uncertainty due to the finite number of Monte Carlo evaluation repetitions and should
not be confused with the $1-\delta=0.95$ generalisation confidence level defining the risk bounds themselves.

\begin{figure}[t]
    \centering
    \includegraphics[width=\columnwidth]
    {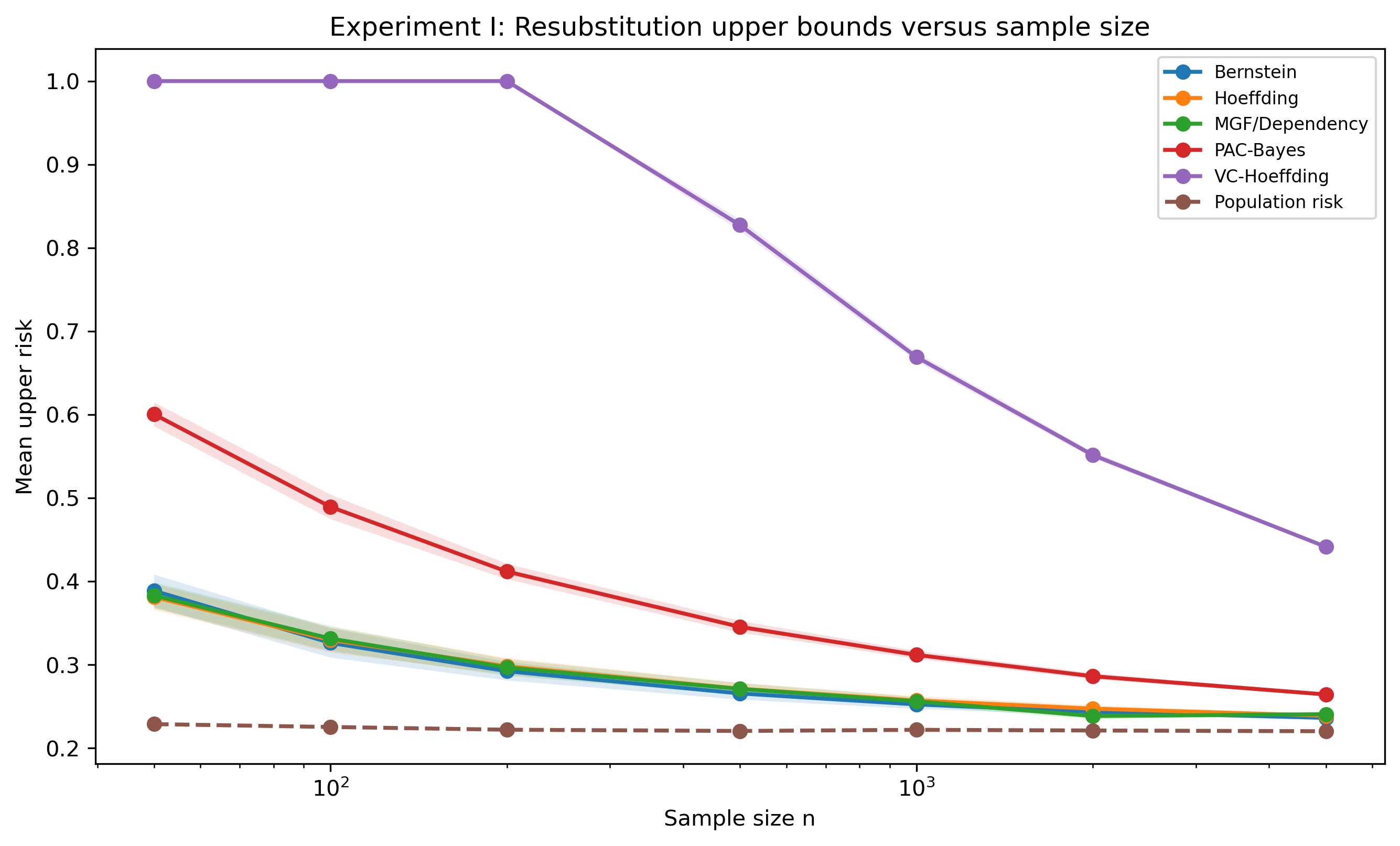}
    \caption{Mean resubstitution upper risk as a function of sample size.
    Shaded regions denote 95\% Monte Carlo confidence intervals for the
    estimated mean across independent evaluation repetitions. The dashed
    curve denotes the population risk estimated from the independent test
    sample. Monte Carlo confidence intervals quantify simulation
    uncertainty and are distinct from the $1-\delta=0.95$ generalisation
    confidence level defining the bounds.}
    \label{fig:exp1_upper_bounds}
\end{figure}

For the deterministic classifier, Hoeffding, empirical-Bernstein, and
the calibrated MGF constructions became progressively tighter as the
sample size increased. Under resubstitution, the mean Hoeffding upper
risk decreased from $0.381$ at $n=50$ to $0.239$ at $n=5000$,
whereas the empirical-Bernstein bound decreased from $0.389$ to
$0.236$. The calibrated MGF bound followed a similar trend, decreasing
from $0.383$ to $0.241$. Their corresponding mean tightness values
decreased from approximately $0.15$--$0.16$ at $n=50$ to approximately
$0.016$--$0.020$ at $n=5000$.

The same qualitative behaviour was observed under $K$-fold
cross-validation. At $n=50$, the mean Hoeffding,
empirical-Bernstein, and calibrated MGF upper risks were $0.395$,
$0.406$, and $0.401$, respectively. At $n=5000$, the corresponding
values were $0.239$, $0.237$, and $0.240$. Thus, the difference between
resubstitution and $K$-fold progressively diminished as the amount of
data increased. Table~\ref{tab:experiment1_summary} summarises representative results
at the two extremes of the sample-size range. It highlights both the
reduction in upper risk and the substantial decrease in tightness as
$n$ increases.

\begin{table*}[t]
\centering
\caption{Representative results for Experiment I. Each entry reports
mean upper risk / mean tightness.}
\label{tab:experiment1_summary}
\begin{tabular}{lcccc}
\hline
& \multicolumn{2}{c}{$n=50$}
& \multicolumn{2}{c}{$n=5000$} \\
Method & Resub. & $K$-fold & Resub. & $K$-fold \\
\hline
Hoeffding
& 0.381 / 0.152
& 0.395 / 0.165
& 0.239 / 0.019
& 0.239 / 0.019 \\
Emp. Bernstein
& 0.389 / 0.160
& 0.406 / 0.175
& 0.236 / 0.016
& 0.237 / 0.016 \\
MGF/Dependency
& 0.383 / 0.154
& 0.401 / 0.170
& 0.241 / 0.020
& 0.240 / 0.020 \\
Safe-MGF/Dependency
& --
& 1.000 / 0.769
& --
& 1.000 / 0.780 \\
PAC-Bayes
& 0.600 / 0.370
& 0.994 / 0.762
& 0.264 / 0.043
& 0.319 / 0.097 \\
VC-Hoeffding
& 1.000 / 0.771
& --
& 0.441 / 0.221
& -- \\
H\"older-Hoeffding
& --
& 0.609 / 0.378
& --
& 0.261 / 0.040 \\
\hline
\end{tabular}
\end{table*}

PAC-Bayes produced wider bounds, particularly in the $K$-fold setting. Under resubstitution, its mean upper risk decreased from
$0.600$ at $n=50$ to $0.264$ at $n=5000$. Under $K$-fold, the corresponding values decreased from $0.994$ to $0.319$. In contrast, the uniform VC-Hoeffding construction was substantially more conservative: its resubstitution upper risk remained equal to one for $n\leq 200$, was $0.827$ at $n=500$, $0.669$ at $n=1000$, and $0.441$ at $n=5000$. The dependence-robust H\"older-Hoeffding construction yielded informative but wider $K$-fold bounds than the standard Hoeffding specialisation, decreasing from $0.609$ at $n=50$ to $0.261$ at $n=5000$. By contrast, the present PSD-inflated Safe-MGF/Dependency implementation remained trivial, with $U=1$, across the complete sample-size range. This illustrates the price paid by the conservative PSD inflation used to protect against covariance-estimation uncertainty.

Empirical coverage was generally close to or above the nominal
$0.95$ confidence level. Isolated deviations occurred for
empirical-Bernstein under $K$-fold at $n=1000$, where coverage was
$0.933$, and for the calibrated MGF construction at $n=2000$, where
coverage decreased to $0.867$ under resubstitution and $0.900$ under
$K$-fold. Since coverage is estimated from a finite number of held-out
Monte Carlo repetitions, these point estimates should be interpreted
together with their binomial uncertainty. Wilson 95\% confidence
intervals for empirical coverage are therefore reported in the
supplementary analysis. Beyond coverage and tightness, the bounds can also be evaluated by whether they provide sufficiently strong evidence to certify
better-than-chance prediction. For the balanced binary problem, this
corresponds to the event $U<0.5$.

Figure~\ref{fig:exp1_certification} shows the resulting certification
probability as a function of sample size. Because certification is a
binary event at each Monte Carlo repetition, uncertainty is represented
using 95\% Wilson confidence intervals. Bounded-discriminability certification was achieved rapidly by
Hoeffding, empirical-Bernstein, and calibrated MGF. Under
resubstitution, PAC-Bayes reached complete certification by $n=200$,
whereas under $K$-fold it required larger sample sizes. The uniform
VC-Hoeffding construction certified better-than-chance performance only
at the largest sample size, $n=5000$. The certification analysis thus
reveals a distinction that is not visible from empirical accuracy
alone: two procedures may achieve similar predictive risk while
providing substantially different levels of statistical evidence that
their future risk lies below chance.

\begin{figure}[t]
    \centering
    \includegraphics[width=\columnwidth]
    {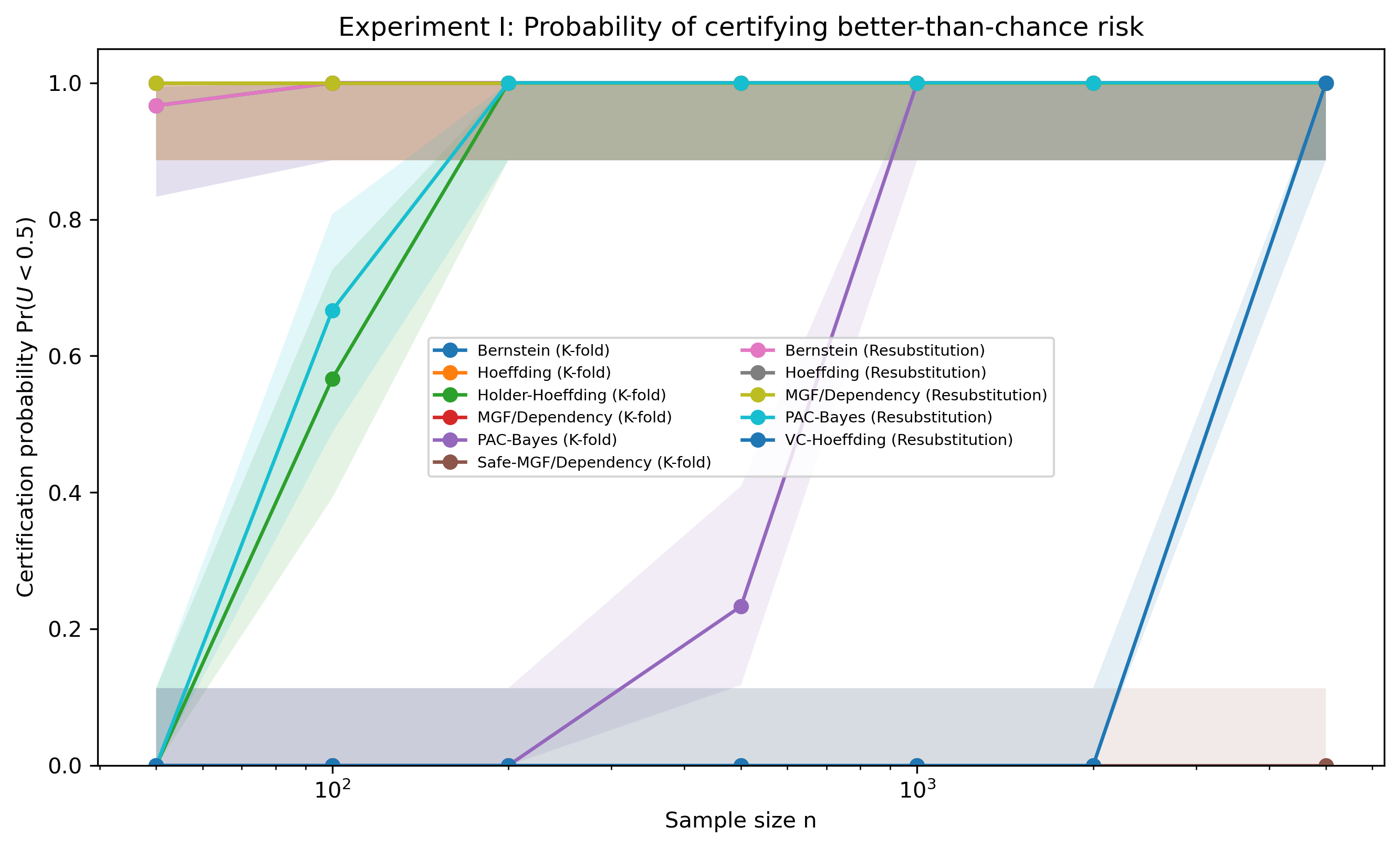}
    \caption{Probability of certifying better-than-chance future risk,
    defined by $U<0.5$, as a function of sample size. Shaded regions
    denote 95\% Wilson confidence intervals for the Monte Carlo
    certification probability.}
    \label{fig:exp1_certification}
\end{figure}

Additional comparisons between resubstitution and $K$-fold validation are reported in Appendix~\ref{app:kfold-comparisons}. In particular, the appendix shows the evolution of bound tightness for the two validation paradigms and
the empirical coverage of the dependence-aware $K$-fold constructions.

\subsection{Experiment II: Increasing structural complexity}

We next study how the proposed bounds behave as the underlying
classification problem becomes progressively more difficult. Rather
than modifying sample size alone, we construct an explicit five-level
complexity ladder, denoted by $L1$--$L5$, while keeping the feature
dimension fixed at $d=2$. The levels are designed to progressively
depart from an easily separable linear problem. $L1$ corresponds to a
simple linear configuration, $L2$ introduces multimodal linear
structure, and $L3$--$L5$ progressively increase non-linearity and
overlap. The highest level, $L5$, produces a problem for which the fixed
linear classifier approaches chance performance. For each complexity level, we evaluate the same sample-size grid as in Experiment I, using 60 Monte Carlo repetitions per configuration. All evaluation criteria and Monte Carlo confidence intervals are computed as described in the common simulation framework. To isolate the effect of structural complexity, the main figures focus on resubstitution bounds, while the representative numerical comparison includes both resubstitution and $K$-fold validation. This allows us to assess how structural complexity affects not only bound tightness but also the differences between the two evaluation paradigms.

\subsubsection{Results}

Figure~\ref{fig:exp2_true_risk} first verifies that the predefined
complexity ladder translates into a progressively more difficult
classification problem for the fixed linear learner. Population risk
increases monotonically from $L1$ to $L5$ across all sample sizes.

\begin{figure}[t]
    \centering
    \includegraphics[width=\columnwidth]
    {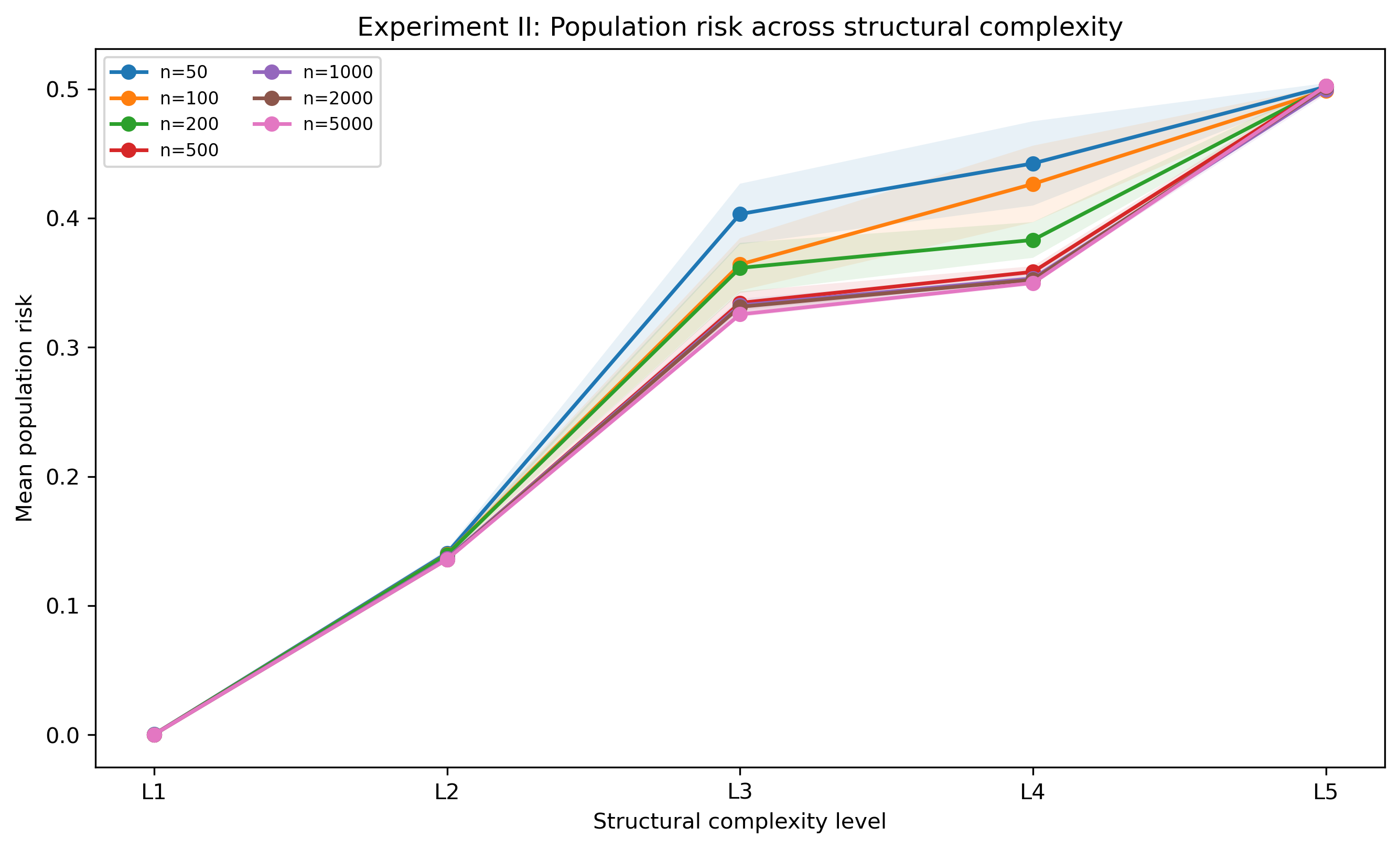}
    \caption{Mean population risk across the structural complexity
    ladder $L1$--$L5$ for different sample sizes. Shaded regions denote
    95\% Monte Carlo confidence intervals. The progression from $L1$ to
    $L5$ produces a monotonic increase in population risk for the fixed
    linear classifier.}
    \label{fig:exp2_true_risk}
\end{figure}

At $n=200$, the mean population risk increases from approximately
$3\times10^{-4}$ at $L1$ to $0.140$ at $L2$, $0.362$ at $L3$,
$0.383$ at $L4$, and $0.501$ at $L5$. This progression separates two
different sources of difficulty. For the intermediate levels, increasing
sample size reduces population risk, indicating that part of the
difficulty is attributable to finite-sample estimation. In contrast,
at $L5$ the population risk remains close to $0.5$ across the complete
sample-size range, indicating a structural mismatch between the linear
classifier and the underlying data geometry. 

We next examine the resubstitution risk bounds at the representative sample size $n=200$. As shown in Fig.~\ref{fig:exp2_upper_bounds}, Hoeffding-CUBV, empirical-Bernstein CUBV, and the calibrated MGF/Dependency construction track the increasing problem difficulty while remaining comparatively close to the population risk. At $L1$, their mean upper risks are $0.087$, $0.010$, and $0.002$, respectively. At $L3$, they increase to approximately $0.435$, $0.440$, and $0.434$, and at $L5$
to $0.555$, $0.565$, and $0.576$.

\begin{figure}[t]
    \centering
    \includegraphics[width=\columnwidth]
    {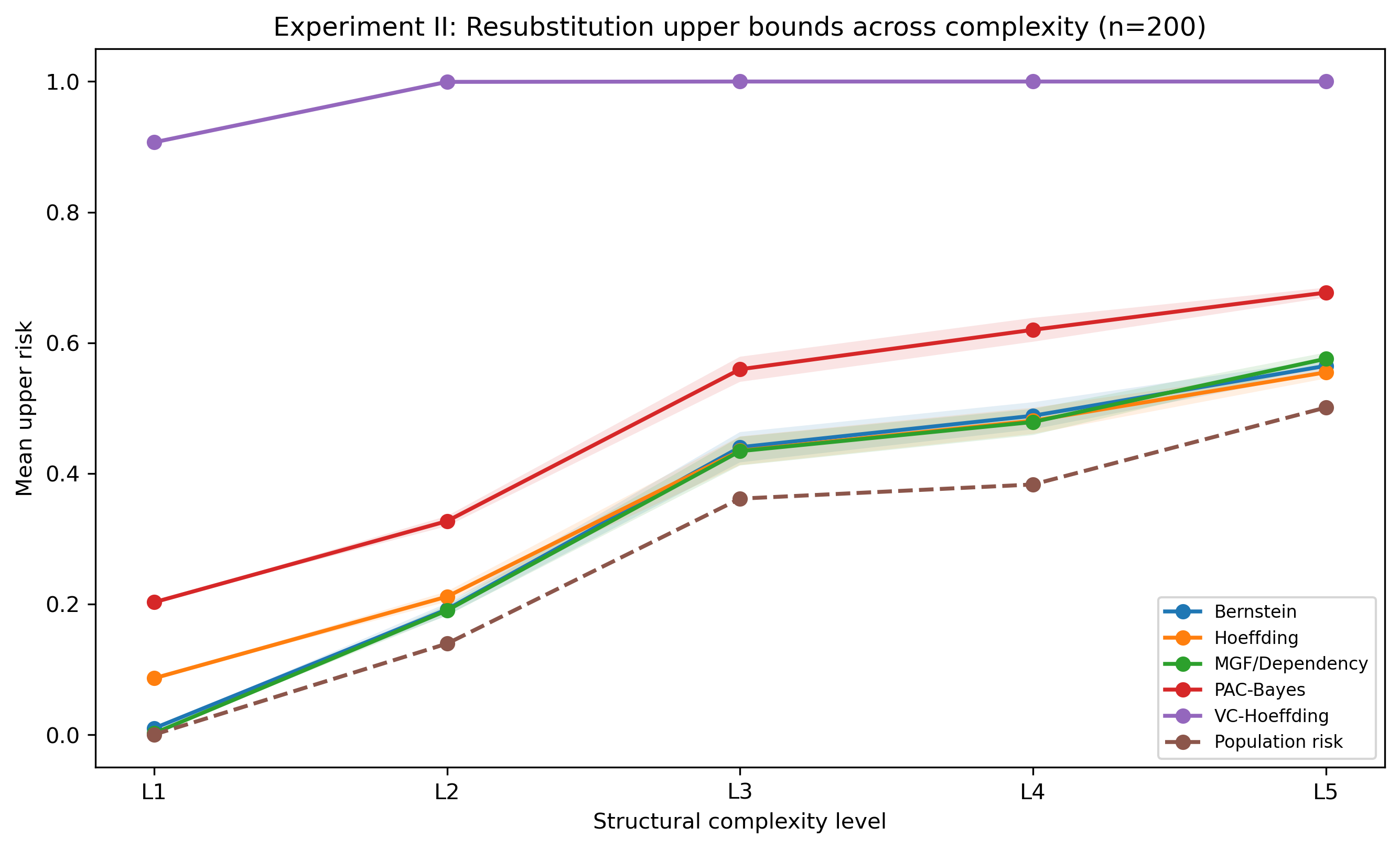}
    \caption{Mean resubstitution upper risk across structural complexity
    levels at $n=200$. Shaded regions denote 95\% Monte Carlo confidence
    intervals. The dashed curve denotes the corresponding population
    risk.}
    \label{fig:exp2_upper_bounds}
\end{figure}

PAC-Bayes remains more conservative throughout the complexity ladder. At $n=200$, its mean resubstitution upper risk increases from $0.203$ at $L1$ to $0.327$ at $L2$, $0.560$ at $L3$, $0.620$ at $L4$, and $0.677$ at $L5$. The VC-Hoeffding construction is substantially more conservative, reaching approximately one from $L2$ onward and therefore becoming largely uninformative in the more difficult regimes. The greater conservatism of PAC-Bayes can be further explained by separating the effect of Gibbs randomisation from the complexity-dependent penalty. At $n=200$,¡ the increase in population risk induced by Gibbs randomisation is negligible for $L1$ and $L2$, reaches¡ approximately $0.010$ at $L3$ and $0.036$ at $L4$, and becomes negligible again at $L5$. By contrast, the PAC-Bayes penalty remains close to $0.20$ across all five complexity levels. Thus, the greater conservatism observed in these simulations is primarily attributable to the complexity and confidence penalty rather than to posterior randomisation itself. Table~\ref{tab:experiment2_summary} also reveals systematic differences between resubstitution and $K$-fold validation. For the classical concentration-based constructions and calibrated MGF/Dependency, resubstitution generally produces tighter bounds, with differences becoming more pronounced as structural complexity increases. PAC-Bayes exhibits a substantially larger gap between the two paradigms, reflecting the stronger penalties associated with its fold-wise implementation. In contrast, the PSD-inflated Safe-MGF/Dependency construction remains trivial ($U=1$) across all complexity levels, illustrating the substantial cost of conservative covariance uncertainty protection. The $K$-fold bounds target the average population risk of the fold-trained classifiers, rather than the final classifier trained on the complete sample.

\begin{table*}[t]
\centering
\caption{Results for Experiment II at $n=200$. Each entry reports mean
upper risk / mean tightness, averaged over 30 evaluation repetitions.}
\label{tab:experiment2_summary}
\begin{tabular}{lccccc}
\hline
Method & $L1$ & $L2$ & $L3$ & $L4$ & $L5$ \\
\hline
\multicolumn{6}{l}{\textit{Resubstitution}} \\
\hline
Bernstein
& 0.010 / 0.010
& 0.192 / 0.053
& 0.440 / 0.079
& 0.488 / 0.105
& 0.565 / 0.064 \\

Hoeffding
& 0.087 / 0.086
& 0.212 / 0.072
& 0.435 / 0.073
& 0.481 / 0.098
& 0.555 / 0.054 \\

MGF/Dependency
& 0.002 / 0.002
& 0.190 / 0.051
& 0.434 / 0.073
& 0.479 / 0.096
& 0.576 / 0.075 \\

PAC-Bayes
& 0.203 / 0.203
& 0.327 / 0.187
& 0.560 / 0.188
& 0.620 / 0.201
& 0.677 / 0.177 \\

VC-Hoeffding
& 0.907 / 0.907
& 0.999 / 0.860
& 1.000 / 0.638
& 1.000 / 0.617
& 1.000 / 0.499 \\

\hline
\multicolumn{6}{l}{\textit{$K$-fold validation ($K=5$)}} \\
\hline
Bernstein
& 0.010 / 0.010
& 0.199 / 0.059
& 0.458 / 0.092
& 0.513 / 0.111
& 0.603 / 0.103 \\

Hoeffding
& 0.087 / 0.086
& 0.217 / 0.077
& 0.452 / 0.086
& 0.505 / 0.103
& 0.593 / 0.093 \\

MGF/Dependency
& 0.003 / 0.003
& 0.196 / 0.056
& 0.452 / 0.085
& 0.503 / 0.101
& 0.632 / 0.132 \\

PAC-Bayes
& 0.464 / 0.464
& 0.591 / 0.450
& 0.833 / 0.455
& 0.894 / 0.468
& 0.956 / 0.456 \\

Holder-Hoeffding
& 0.194 / 0.193
& 0.324 / 0.184
& 0.559 / 0.193
& 0.612 / 0.210
& 0.700 / 0.200 \\

Safe-MGF/Dependency
& 1.000 / 1.000
& 1.000 / 0.860
& 1.000 / 0.633
& 1.000 / 0.598
& 1.000 / 0.499 \\

\hline
\end{tabular}
\end{table*}

The effect of structural complexity is even more apparent when the resubstitution bounds are interpreted through bounded-discriminability certification. Figure~\ref{fig:exp2_certification} reports the probability of certifying $U<0.5$ at $n=200$. Hoeffding-CUBV, empirical-Bernstein CUBV, and calibrated MGF/Dependency certify better-than-chance future performance with probability one at $L1$ and $L2$. At $L3$, their certification probability decreases to approximately $0.867$. At $L4$, certification falls further, to $0.600$ for Hoeffding and empirical-Bernstein and $0.667$ for MGF/Dependency. At $L5$, where the population risk is itself approximately $0.5$, certification essentially disappears. PAC-Bayes exhibits the same qualitative transition but loses certification earlier: its certification probability falls from one at $L1$ and $L2$ to $0.167$ at $L3$ and zero at $L4$ and $L5$.
VC-Hoeffding does not certify $U<0.5$ at any complexity level for $n=200$.

\begin{figure}[t]
    \centering
    \includegraphics[width=\columnwidth]
    {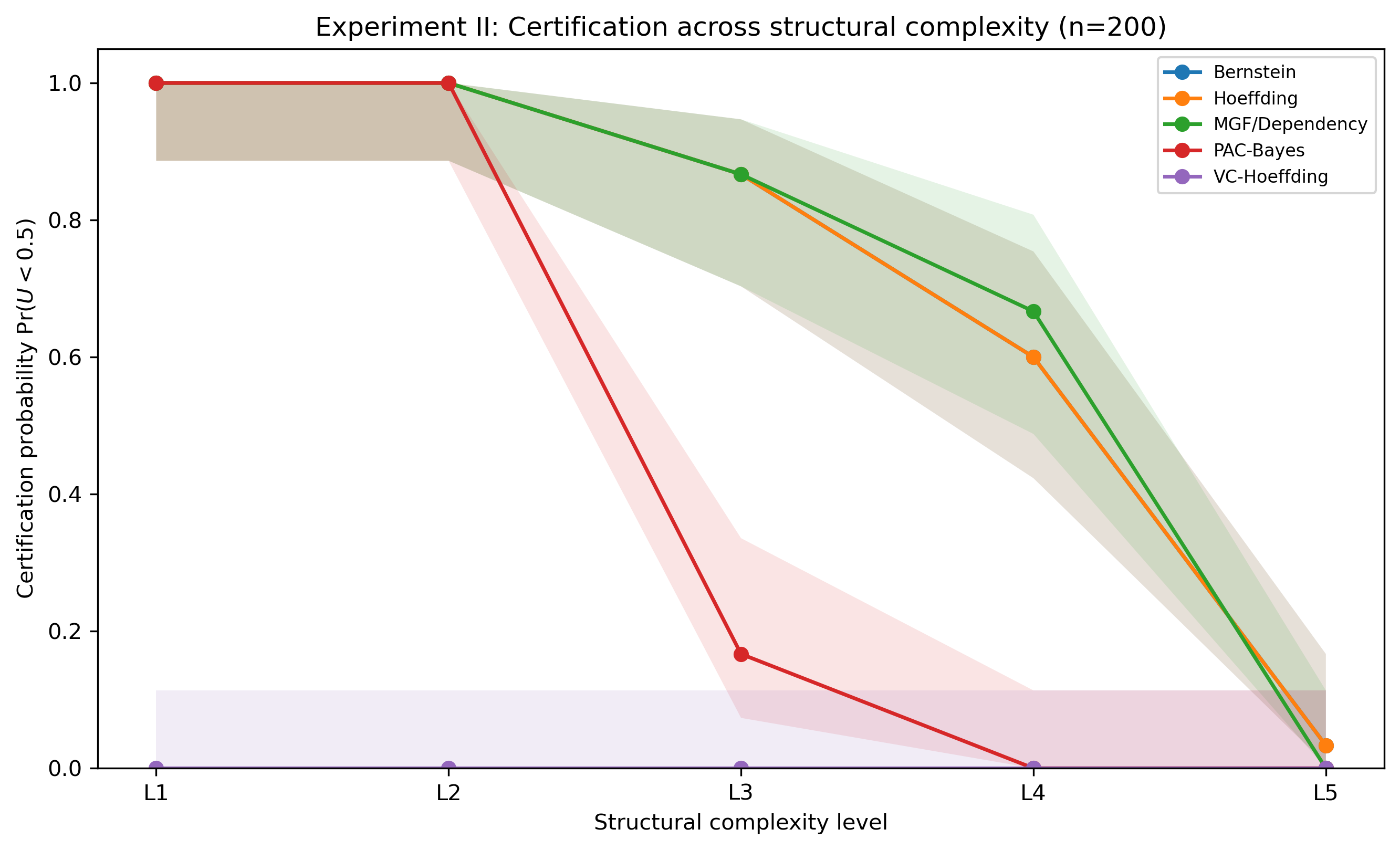}
    \caption{Probability of certifying better-than-chance future risk
    from the resubstitution bounds across structural complexity levels
    at $n=200$. Shaded regions denote 95\% Wilson confidence intervals
    for the Monte Carlo certification probability.}
    \label{fig:exp2_certification}
\end{figure}

These results illustrate an important distinction between predictive performance and statistical evidence. As the problem becomes more complex, empirical classification may still remain better than chance for several regimes, while the available data may no longer support a high-confidence claim that future risk is below the chance threshold. The bounded-discriminability criterion therefore exposes the point at which increasing structural complexity exhausts the evidential strength of the available sample.

\subsection{Experiment III: Multimodality versus non-linearity}

We next disentangle two different notions of structural complexity: multimodality and non-linearity. The objective is to determine whether increasing the number of mixture components is itself sufficient to make the classification problem harder, or whether the geometry induced by those components is the dominant factor. We compare two families of synthetic problems. In the \emph{linear-multimodal} family, the number of Gaussian components is increased while preserving a globally exploitable linear structure. In the \emph{nonlinear-alternating} family, the components are arranged so that the class labels increasingly alternate across the feature space, progressively violating the inductive bias of the linear logistic classifier. For both geometries, we consider $N_c\in\{2,4,6,8,10\}$, and the same sample-size grid used in the previous experiments. Each configuration is evaluated over 60 Monte Carlo repetitions. The main comparison below focuses on the resubstitution bounds at the representative sample size $n=200$.

\subsubsection{Results}

Figure~\ref{fig:exp3_true_risk} reveals a clear distinction between multimodality and non-linearity. In the linear-multimodal family, population risk remains well below chance despite increasing the number
of mixture components. At $n=200$, mean population risk is approximately $0.108$ for $N_c=2$, increases to $0.221$ for $N_c=4$, and then remains close to $0.20$ for $N_c=6$, $8$, and $10$. By contrast, in the nonlinear-alternating family, the problem undergoes a sharp transition. The population risk is only $0.034$ for $N_c=2$, but increases to $0.492$ at $N_c=4$ and remains essentially at chance for all larger values of $N_c$. Thus, increasing multimodality alone does not necessarily destroy linear discriminability; the decisive factor is whether the induced class geometry remains compatible with the classifier.

\begin{figure}[t]
    \centering
    \includegraphics[width=\columnwidth]
    {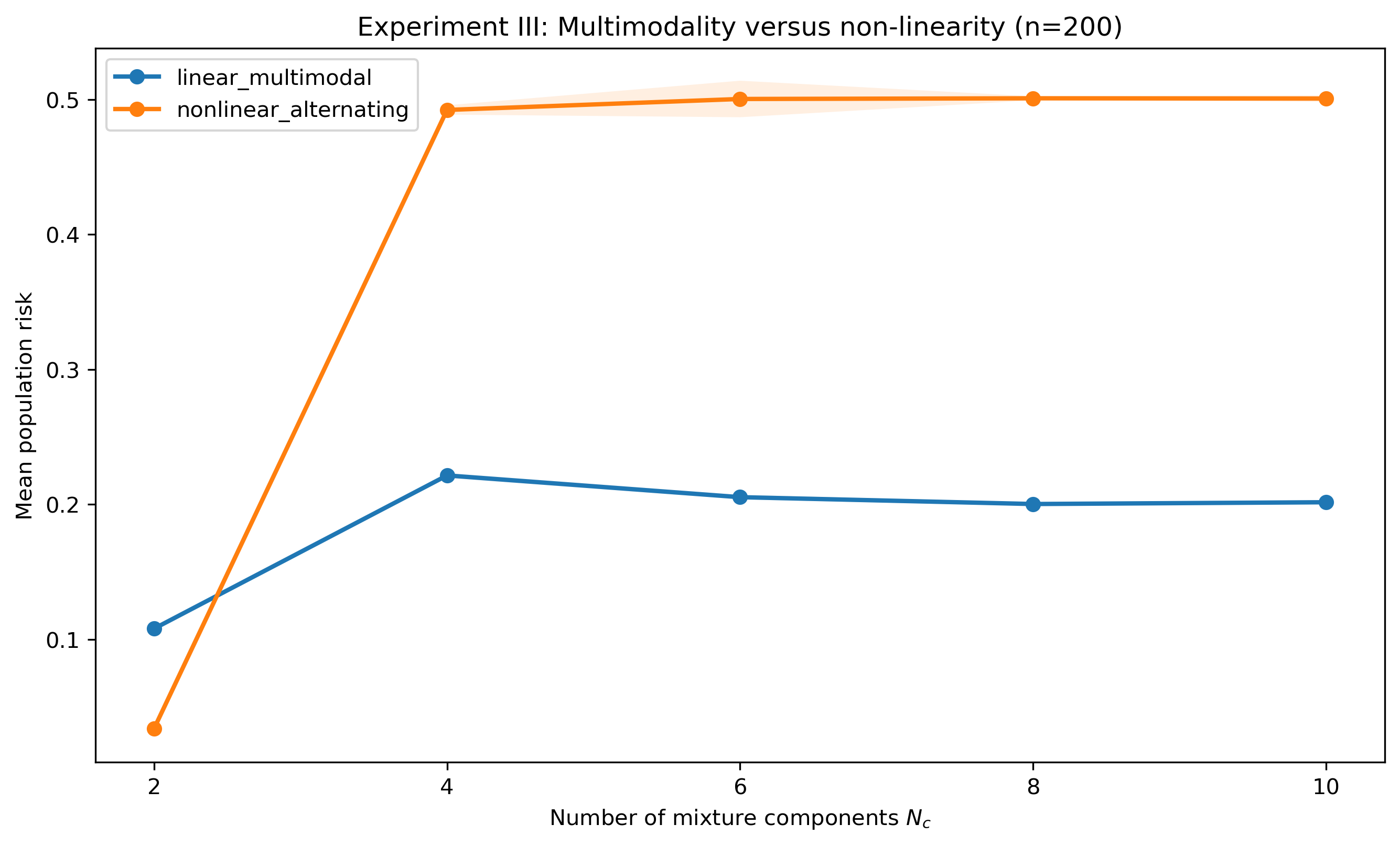}
    \caption{Mean population risk at $n=200$ as the number of mixture
    components increases under linear-multimodal and nonlinear-alternating
    geometries. Shaded regions denote 95\% Monte Carlo confidence
    intervals.}
    \label{fig:exp3_true_risk}
\end{figure}

The same distinction is reflected in the resubstitution bounds.
Figure~\ref{fig:exp3_hoeffding} shows the Hoeffding-CUBV upper risk.
For the linear-multimodal family, the upper bound remains below $0.31$
throughout the complete range of $N_c$. In the nonlinear-alternating
family, however, the upper risk rises above $0.53$ from $N_c=4$
onward, closely reflecting the transition of the underlying population
risk toward chance.

\begin{figure}[t]
    \centering
    \includegraphics[width=\columnwidth]
    {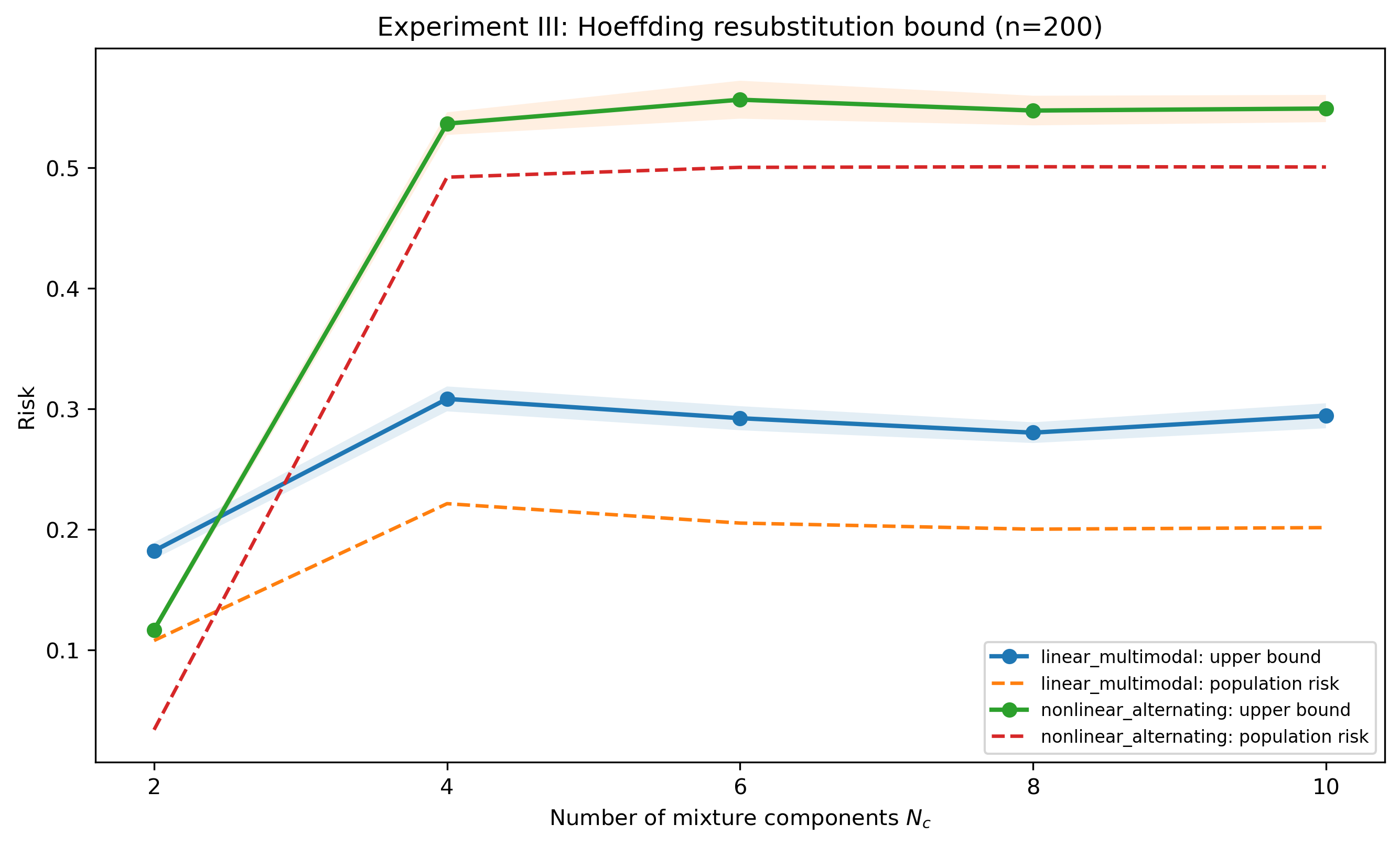}
    \caption{Hoeffding-CUBV resubstitution upper risk at $n=200$ for
    linear-multimodal and nonlinear-alternating geometries. Solid curves
    denote mean upper risks, dashed curves denote the corresponding
    population risks, and shaded regions denote 95\% Monte Carlo
    confidence intervals.}
    \label{fig:exp3_hoeffding}
\end{figure}

At $N_c=4$, for example, the linear-multimodal population risk is
$0.221$ and the corresponding Hoeffding upper risk is $0.308$. Under
the nonlinear-alternating geometry, the population risk increases to
$0.492$ and the Hoeffding upper risk to $0.537$. For $N_c\geq 6$,
the nonlinear-alternating population risk remains approximately $0.5$,
and the corresponding upper bounds remain above the chance threshold. The distinction becomes particularly clear under bounded-discriminability certification. Figure~\ref{fig:exp3_certification} shows the probability of certifying $U<0.5$.

\begin{figure}[t]
    \centering
    \includegraphics[width=\columnwidth]
    {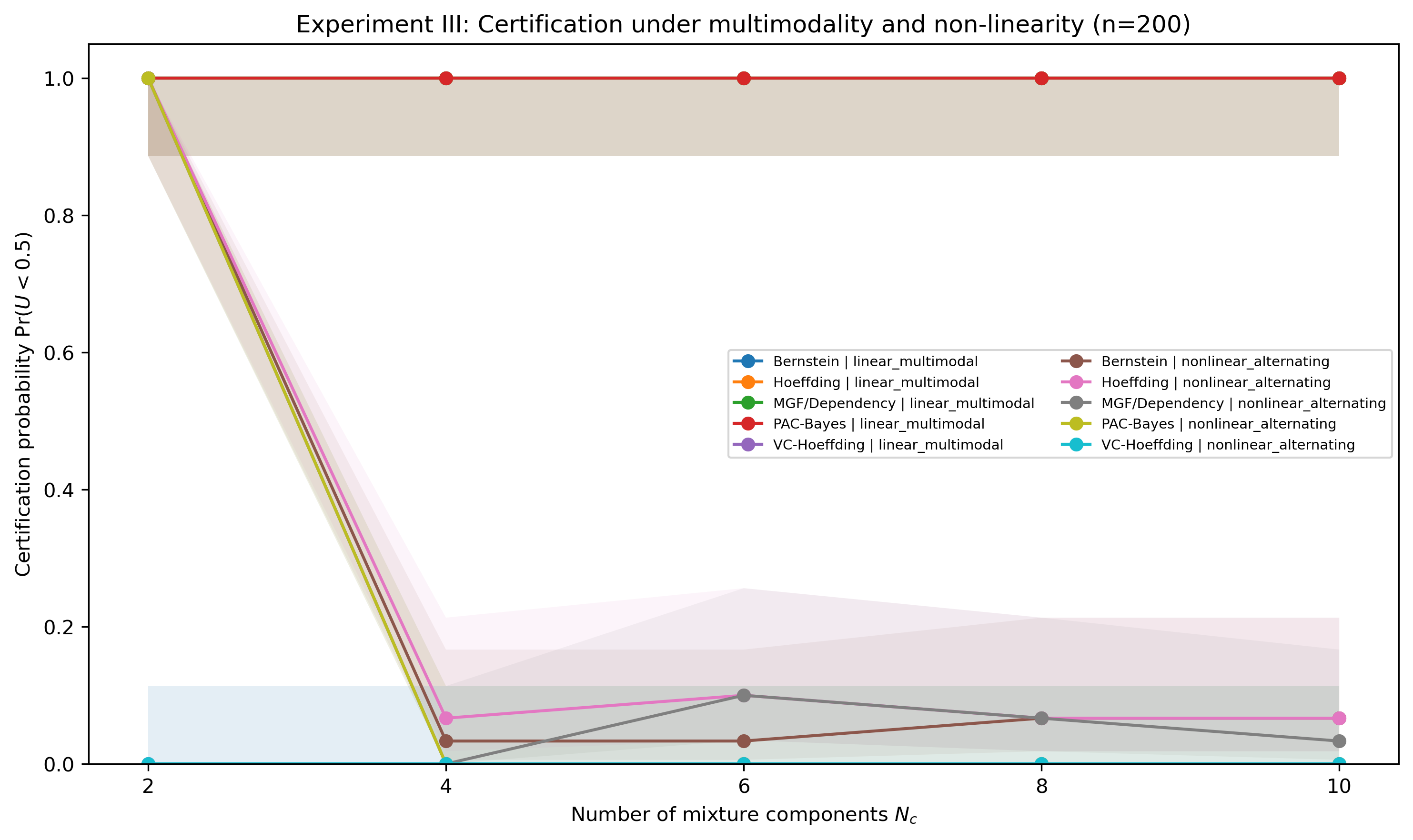}
    \caption{Probability of certifying better-than-chance future risk
    at $n=200$ under linear-multimodal and nonlinear-alternating
    geometries. Shaded regions denote 95\% Wilson confidence intervals
    for the Monte Carlo certification probability.}
    \label{fig:exp3_certification}
\end{figure}

For the linear-multimodal family, Hoeffding-CUBV,
empirical-Bernstein CUBV, calibrated MGF/Dependency, and PAC-Bayes
achieve certification probability one across the complete range
$N_c=2,\ldots,10$. VC-Hoeffding remains too conservative to certify
$U<0.5$. For the nonlinear-alternating family, certification collapses once the geometry becomes incompatible with the linear learner. Although all
non-VC constructions certify the easy $N_c=2$ problem with probability
one, certification is close to zero from $N_c=4$ onward. At $N_c=4$,
for example, the certification probability is $0.067$ for Hoeffding-CUBV, $0.033$ for empirical-Bernstein CUBV, and zero for MGF/Dependency and PAC-Bayes.

These results show that multimodality and non-linearity should not be
treated as equivalent sources of difficulty. A representation may
contain many modes while still preserving strong global
discriminability for a linear classifier. Conversely, a comparatively
small number of components can destroy linear discriminability when
their arrangement induces an incompatible decision geometry.
Bounded-discriminability certification captures this distinction
directly: certification is retained under increasing multimodality but
vanishes when the representation ceases to support reliable
better-than-chance prediction for the chosen learner.

\subsection{Experiment IV: Bound-aware fusion of heterogeneous multimodal sources}
\label{sec:expIV}

We next evaluate the information-fusion component of the proposed
framework under controlled source heterogeneity. The objective is to
determine whether a fusion rule that optimises the resulting
generalisation bound can provide more informative risk guarantees than
simpler weighting strategies when the available sources represent
different heterogeneous observations of the same classification task.

We consider $M=4$ information sources defined over the same subjects
and the same balanced binary labels. Each source is itself multimodal:
both classes are represented by two Gaussian components. Before
source-specific rotation, the four component centres follow the
interlaced configuration
\begin{equation}
\begin{aligned}
\mathcal C_0:\quad
&\bm{\mu}_{01}=(-a,+b),
&
\bm{\mu}_{02}=(+a,-b),
\\
\mathcal C_1:\quad
&\bm{\mu}_{11}=(+a,+b),
&
\bm{\mu}_{12}=(-a,-b).
\end{aligned}
\label{eq:expIV-interlaced}
\end{equation}
Hence, the two modes belonging to each class are spatially interlaced
with those of the opposite class.

The binary label of each subject is shared across sources, whereas the
within-class Gaussian component is sampled independently in each
source. The source distributions are therefore non-homogeneous even
though the underlying binary task and subject labels are common. In
addition to the source-specific mixture proportions, the separation,
dispersion, and orientation vary across sources. The configurations
used in the simulation are
\begin{center}
\begin{tabular}{lccccc}
\toprule
Source & $a$ & $b$ & $\sigma$ & Rotation & Mode-1 probability\\
\midrule
1 & 1.50 & 1.30 & 1.05 & $0^\circ$   & 0.65\\
2 & 1.45 & 1.35 & 1.08 & $15^\circ$  & 0.60\\
3 & 1.55 & 1.25 & 1.10 & $-15^\circ$ & 0.68\\
4 & 1.40 & 1.40 & 1.08 & $30^\circ$  & 0.58\\
\bottomrule
\end{tabular}
\end{center}

For every sample size we compare five fusion rules: uniform weighting, empirical-risk weighting, variance-only weighting, Fusion-CUBV, and Safe-Fusion-CUBV. The five weighting strategies are formally specified in Appendix G, including their calibration objectives, simplex constraints, and the covariance-inflation procedure used by Safe-Fusion-CUBV. Each source is classified independently using the same L2-regularised logistic-regression learner. Fusion therefore takes place at the risk/evidence level rather than by concatenating the source features. For each $n$, 60 Monte Carlo repetitions are generated. The first 30 repetitions are used exclusively to estimate the calibration risk, mean generalisation-gap vector, covariance matrix, and fusion weights. These weights are then frozen before evaluation on the remaining 30 repetitions. Population risk is approximated using an independent test sample of size $10\,000$. Continuous Monte Carlo summaries are reported with 95\% confidence intervals for the mean, while empirical coverage uses Wilson 95\% binomial intervals.

\subsubsection{Results}

Figure~\ref{fig:exp4_upper_bound} shows the fused upper risks as a
function of sample size. Excluding the deliberately conservative
Safe-Fusion-CUBV construction, the bounds become progressively tighter
as the available sample size increases.

\begin{figure}[t]
    \centering
    \includegraphics[width=\columnwidth]
    {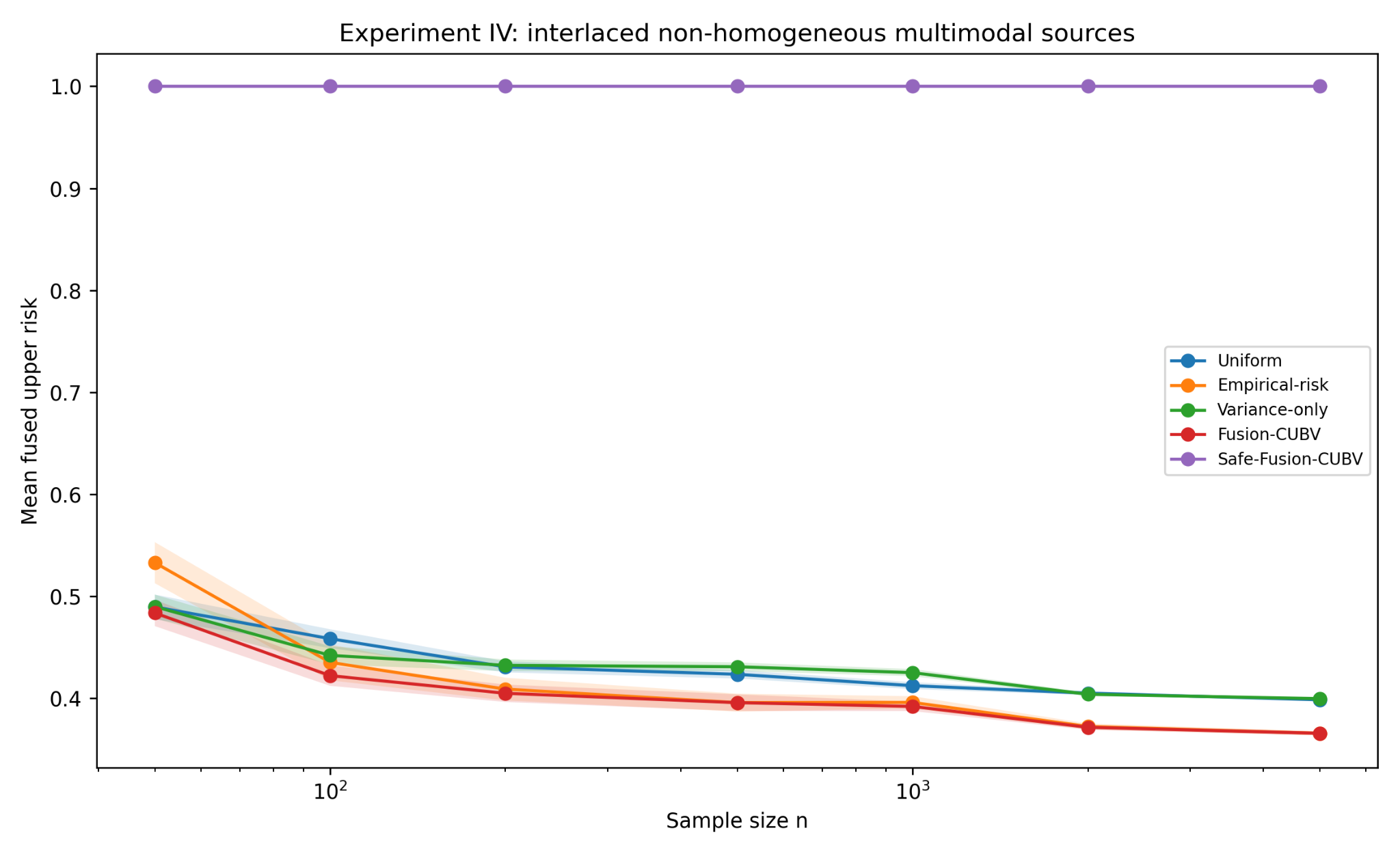}
    \caption{Mean fused upper risk as a function of sample size for the
    five information-fusion strategies. Shaded regions denote 95\%
    Monte Carlo confidence intervals.}
    \label{fig:exp4_upper_bound}
\end{figure}

Fusion-CUBV consistently produces a smaller mean upper risk than
uniform and variance-only fusion throughout the sample-size sweep. At
$n=50$, the mean upper risk is $0.484$ for Fusion-CUBV, compared with
$0.490$ for both uniform and variance-only weighting. At $n=200$, the
corresponding values are $0.405$, $0.431$, and $0.432$, respectively.
At $n=2000$, Fusion-CUBV attains $0.372$, compared with $0.405$ for
uniform weighting and $0.404$ for variance-only weighting. The comparison with empirical-risk weighting exhibits a different pattern. At $n=50$, empirical-risk weighting attains a mean upper risk of $0.533$, substantially above the $0.484$ obtained by Fusion-CUBV. At $n=100$, the corresponding values are $0.435$ and $0.422$. From $n=200$ onward, however, the difference becomes small, and the two procedures coincide at $n=5000$, where both attain a mean upper risk of approximately $0.366$. This result should not be interpreted as Fusion-CUBV uniformly
minimising population risk. Empirical-risk weighting generally attains
the smaller true fused risk: for example, at $n=50$ the mean true risk
is $0.357$ for empirical-risk weighting and $0.389$ for Fusion-CUBV.
The smaller Fusion-CUBV upper risk therefore reflects a different
objective: empirical performance is balanced against the estimated mean
generalisation gap and the covariance structure of the source-specific
gaps. Because all strategies are evaluated on the same Monte Carlo
repetitions, paired differences provide a more sensitive comparison.
Figure~\ref{fig:exp4_paired} reports $U_{\mathrm{other}}-U_{\mathrm{Fusion}}$, so positive values indicate a smaller upper risk for Fusion-CUBV.

\begin{figure}[t]
    \centering
    \includegraphics[width=\columnwidth]
    {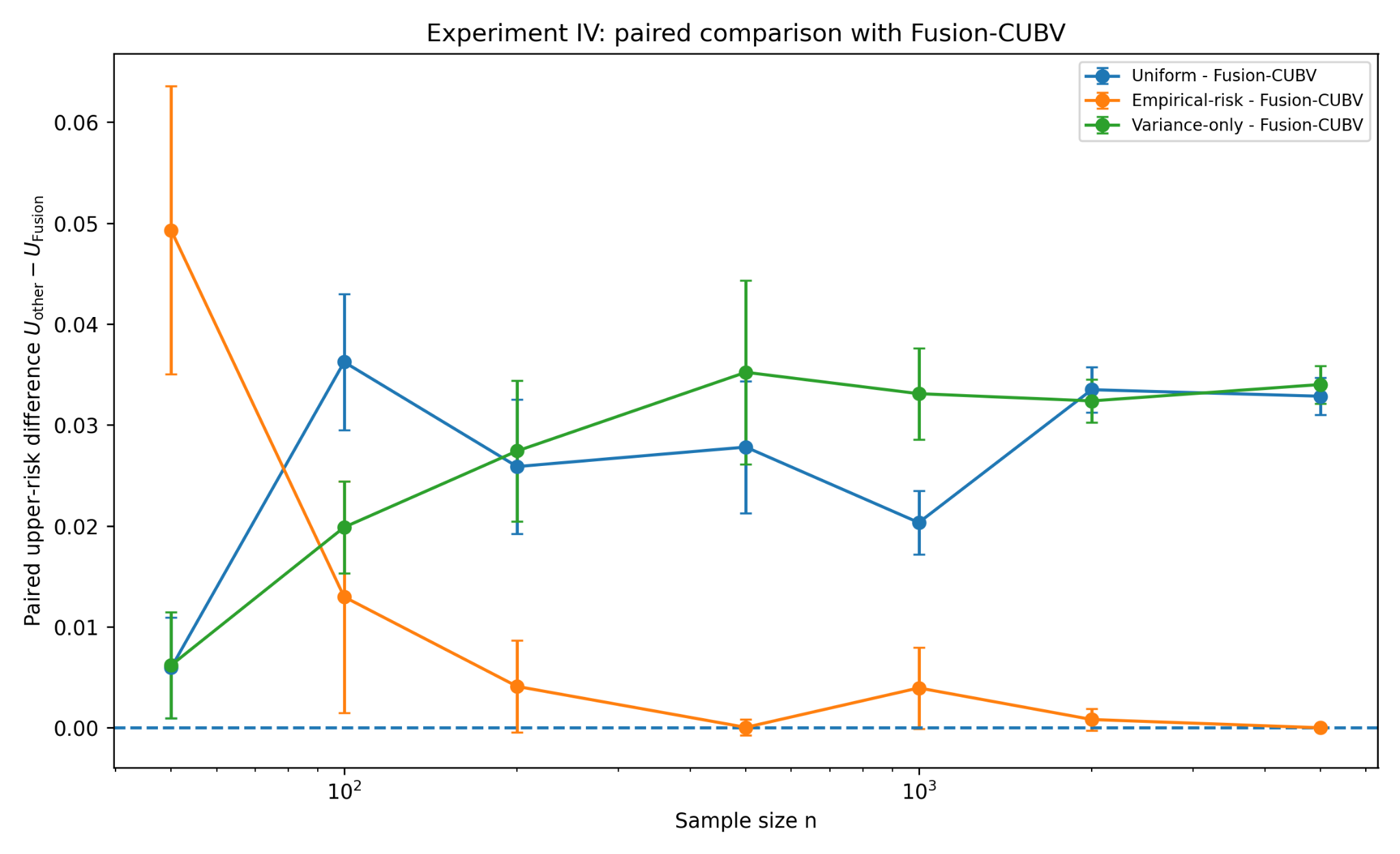}
    \caption{Paired difference in fused upper risk relative to
    Fusion-CUBV. Positive values indicate a larger upper risk for the
    comparison method. Error bars denote 95\% Monte Carlo confidence
    intervals for the paired differences.}
    \label{fig:exp4_paired}
\end{figure}

Relative to empirical-risk weighting, the paired mean difference is
$0.0493$ at $n=50$ (95\% CI $[0.0350,0.0636]$) and $0.0130$ at
$n=100$ (95\% CI $[0.0015,0.0244]$). From $n=200$ onward, the
empirical-risk--Fusion-CUBV interval includes zero in several
configurations, and the two methods are numerically identical at
$n=5000$. In contrast, both uniform and variance-only fusion remain
above Fusion-CUBV throughout the complete sample-size sweep, with
paired 95\% intervals that remain on the positive side of zero. 
\begin{figure}[t]
    \centering
    \includegraphics[width=\columnwidth]
    {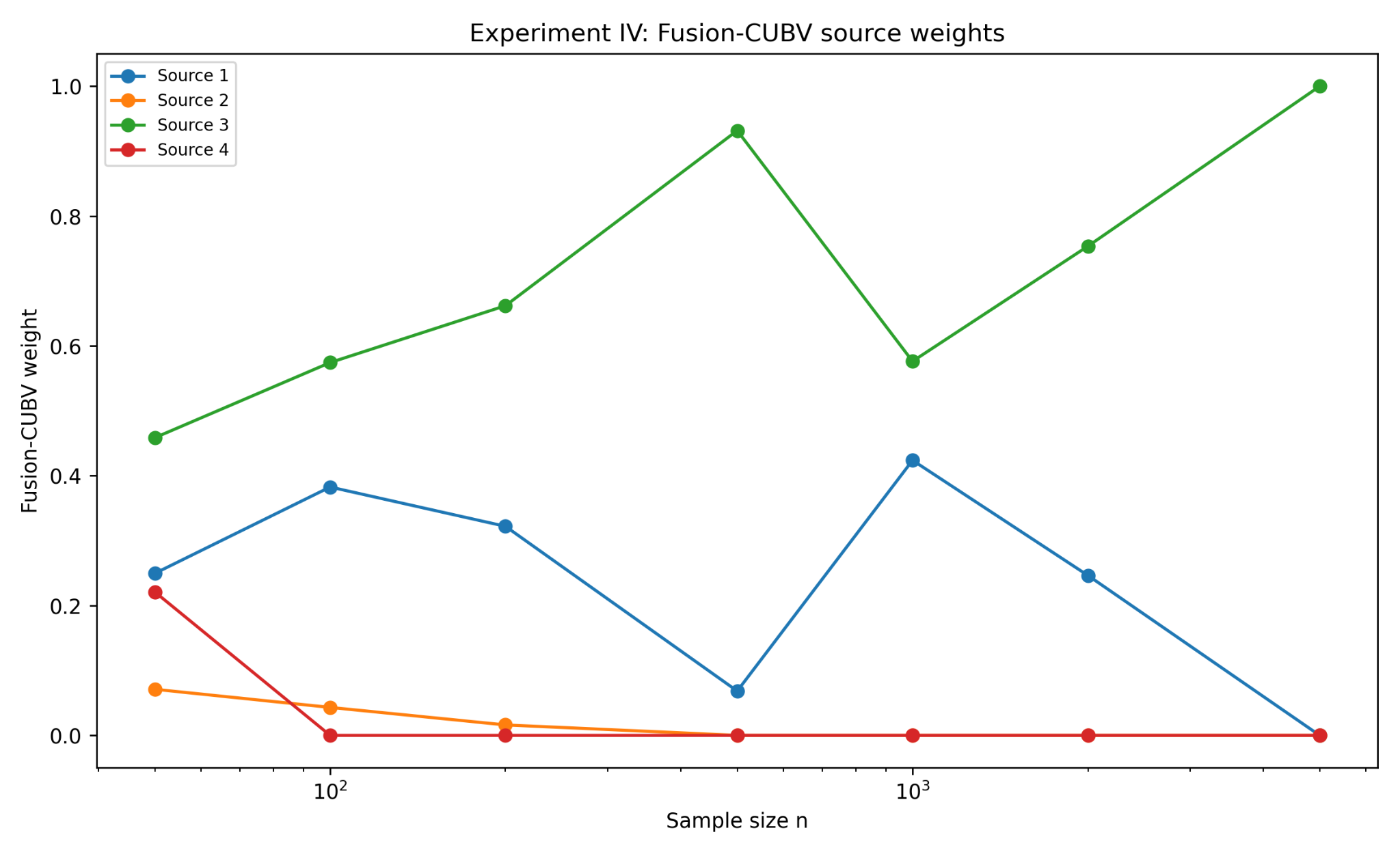}
    \caption{Fusion-CUBV source weights at each sample size. Weights are
    estimated using only the calibration repetitions and frozen before
    evaluation. Because calibration is performed independently for each
    $n$, the curves should not be interpreted as a deterministic
    trajectory of source relevance with sample size.}
    \label{fig:exp4_weights}
\end{figure}

The learned source weights further clarify the mechanism. Figure~\ref{fig:exp4_weights} shows the Fusion-CUBV weights at each sample size.  At $n=50$, Fusion-CUBV assigns approximately $(0.250,\;0.071,\;0.459,\;0.221)$
to Sources 1--4. At $n=200$, the corresponding weights are approximately $(0.322,\;0.016,\;0.662,\;0)$, whereas at $n=1000$ Sources~1 and~3 receive approximately $0.424$ and $0.576$. At $n=5000$, the solution concentrates on Source~3. Thus, the small- and moderate-sample solutions combine sources when this improves the calibrated certificate, while the large-sample solution approaches the empirical-risk rule, which selects Source~3 throughout this experiment.

\begin{table*}[t]
\centering
\caption{Representative results for Experiment IV. Each entry reports
mean fused upper risk / empirical coverage.}
\label{tab:experiment4_summary}
\begin{tabular}{lccc}
\toprule
Method & $n=50$ & $n=500$ & $n=5000$ \\
\midrule
Uniform
& 0.490 / 1.000
& 0.424 / 0.967
& 0.399 / 0.900 \\
Empirical-risk
& 0.533 / 1.000
& 0.396 / 0.967
& 0.366 / 1.000 \\
Variance-only
& 0.490 / 1.000
& 0.431 / 0.933
& 0.400 / 0.800 \\
Fusion-CUBV
& 0.484 / 1.000
& 0.396 / 0.967
& 0.366 / 1.000 \\
Safe-Fusion-CUBV
& 1.000 / 1.000
& 1.000 / 1.000
& 1.000 / 1.000 \\
\bottomrule
\end{tabular}
\end{table*}

Table~\ref{tab:experiment4_summary} reports representative upper-risk and coverage results at three sample sizes. Empirical coverage remains high for Fusion-CUBV and empirical-risk weighting over the sample-size sweep, with one isolated value of
$0.967$ at $n=500$ and coverage equal to one at the remaining tested
sample sizes. Uniform and variance-only weighting show lower coverage
at $n=5000$, reaching $0.900$ and $0.800$, respectively. These values
are based on 30 independent evaluation repetitions and should therefore
be interpreted together with their Wilson intervals; in particular,
they should not be read as precise estimates of asymptotic coverage.

The Safe-Fusion-CUBV construction behaves markedly differently. Its
upper risk remains equal to one across the complete sample-size range,
with mean tightness approximately $0.60$. This reflects the substantial
price paid by the conservative positive-semidefinite covariance
inflation used to protect against covariance-estimation uncertainty.
The construction preserves empirical coverage but is uninformative in
this experimental regime. The better-than-chance threshold $U<0.5$ also becomes informative in this heterogeneous setting. Fusion-CUBV has a mean upper risk below
$0.5$ already at $n=50$, whereas empirical-risk weighting remains above
this threshold on average at the smallest sample size. This distinction
concerns the risk certificate rather than the population risk itself.

Overall, Experiment IV shows that bound-aware fusion is not equivalent
to selecting the source with the smallest empirical error. Under
non-homogeneously sampled multimodal sources, Fusion-CUBV combines
sources when uncertainty is substantial and can produce a materially
smaller upper-risk certificate than empirical-risk weighting in the
small-sample regime. As $n$ increases, the difference diminishes and
the solution approaches empirical-risk source selection, consistent
with the decreasing role of sampling uncertainty.

\subsection{Experiment V: Increasing the number of information sources}

We next investigate the effect of increasing the number of available
information sources while keeping the fusion rule fixed to Fusion-CUBV.
The objective is to determine whether the addition of progressively
weaker or noisier sources necessarily degrades the resulting risk
certificate, or whether the bound-aware weighting mechanism can
effectively suppress uninformative sources.

A nested source construction is used. Source 1 corresponds to the
original representation, whereas additional sources are progressively
attenuated and corrupted by increasing noise. The complete source
ladder spans signal strengths from $1.0$ to $0.4$ and noise standard
deviations from $0$ to $1.0$. Importantly, increasing the number of
sources genuinely adds new sources without redefining those already
present. Thus, $M=2$ uses the first two sources,$M=4$ uses the first four sources, and so forth up to $M=10$. We consider $M\in\{2,4,6,8,10\}$ and the same sample-size grid used in the previous experiments. Each configuration is evaluated over 60 Monte Carlo repetitions using the same independent calibration/evaluation protocol.

\subsubsection{Results}

Figure~\ref{fig:exp5_bound_vs_n} shows the Fusion-CUBV upper risk as
a function of sample size for different numbers of information sources.
The curves are almost indistinguishable across $M$, indicating that
adding progressively weaker sources does not materially degrade the
resulting risk certificate. This behaviour is particularly clear at $n=200$. As shown in
Fig.~\ref{fig:exp5_vs_M}, the mean upper risk remains essentially
constant as the number of sources increases: $0.2967,\;0.2967,\;0.2965,\;0.2962,\;0.2965$, for $M=2,4,6,8,$ and $10$, respectively. The corresponding population
risk also remains stable, ranging approximately from $0.222$ to
$0.224$.

\begin{figure}[t]
    \centering
    \includegraphics[width=\columnwidth]
    {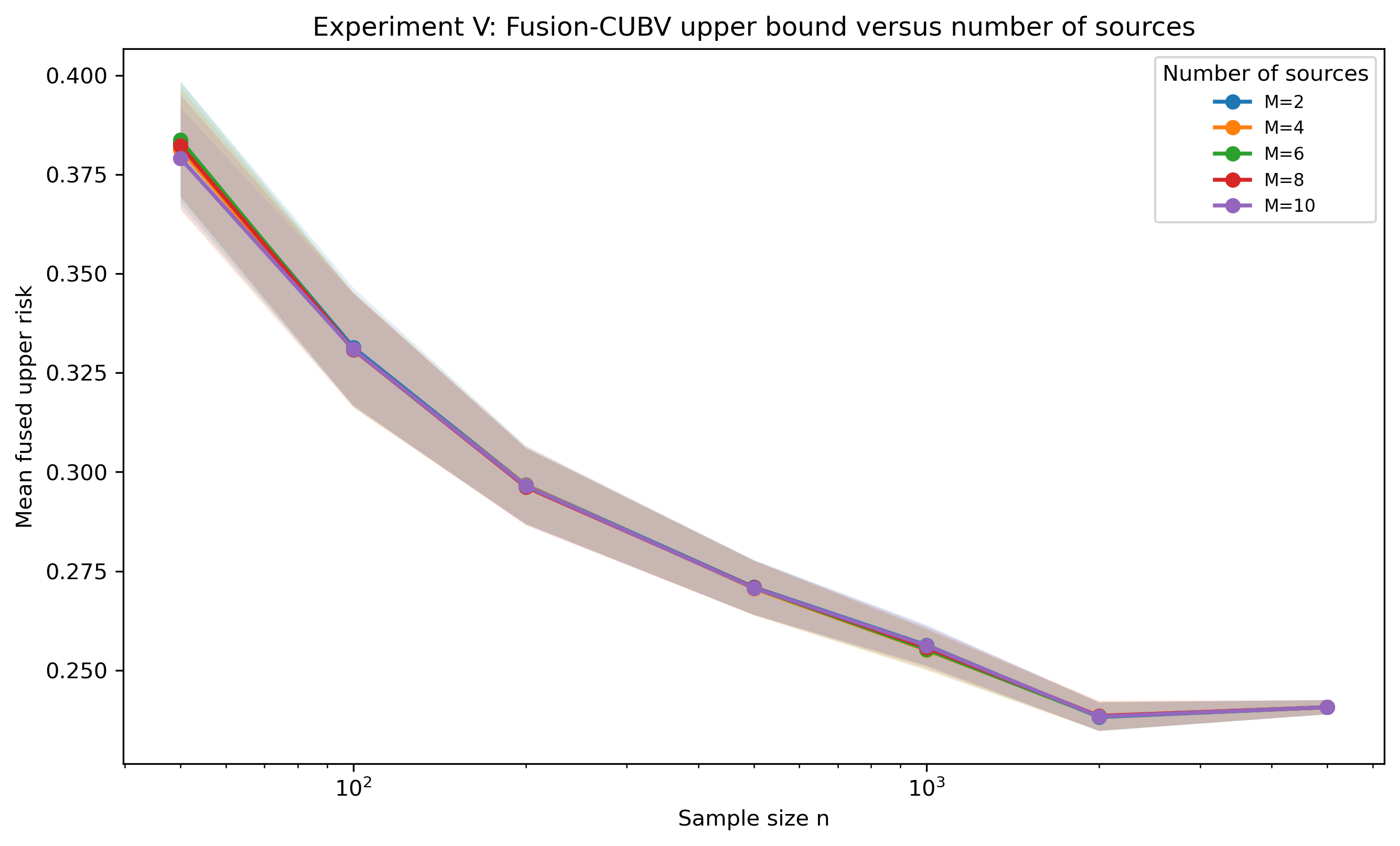}
    \caption{Fusion-CUBV mean upper risk as a function of sample size
    for different numbers of information sources. Shaded regions denote
    95\% Monte Carlo confidence intervals. The near-overlap of the
    curves shows that the bound remains stable as progressively weaker
    sources are added.}
    \label{fig:exp5_bound_vs_n}
\end{figure}

\begin{figure}[t]
    \centering
    \includegraphics[width=\columnwidth]
    {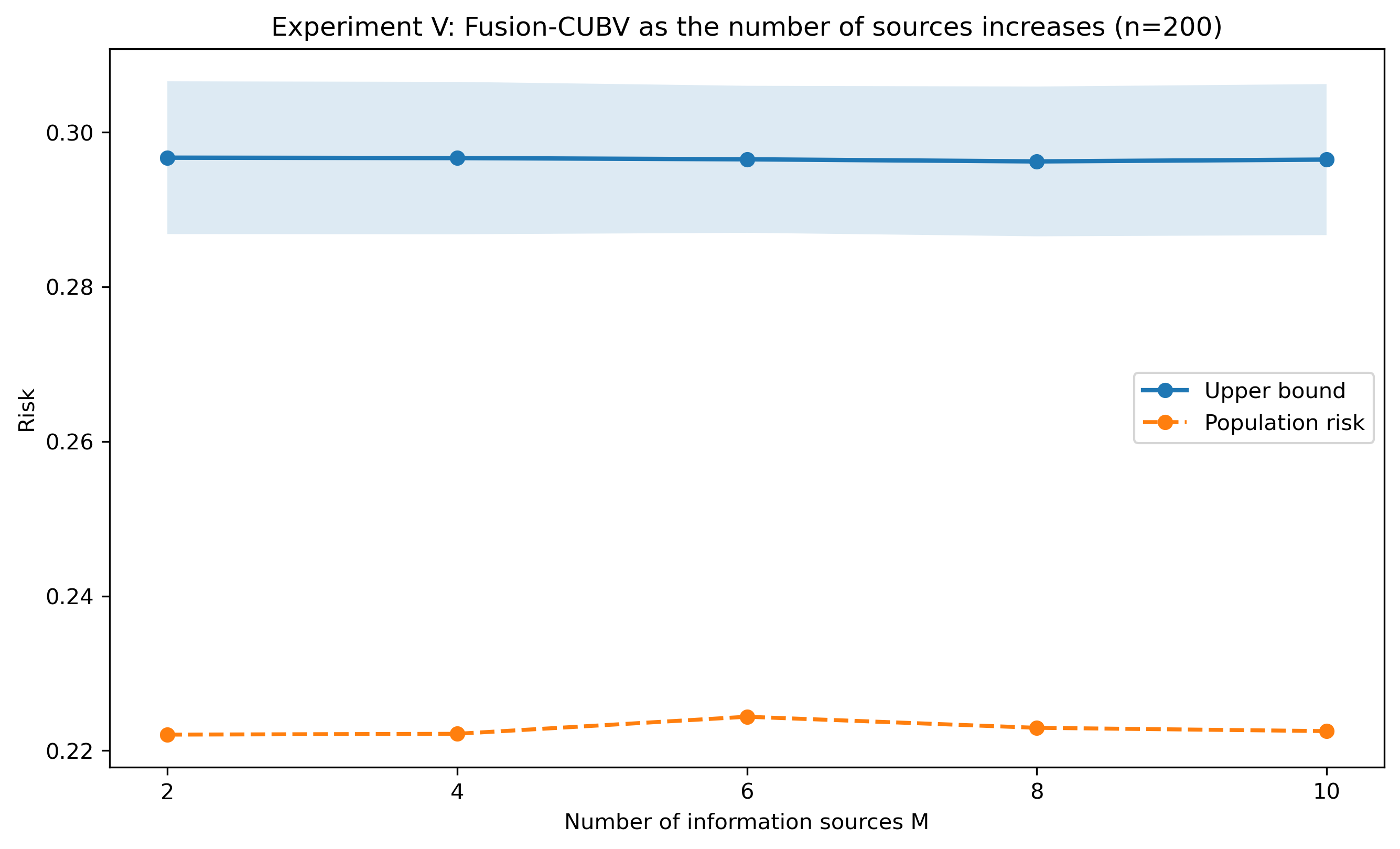}
    \caption{Fusion-CUBV upper risk and corresponding population risk as
    the number of information sources increases at $n=200$. Shaded
    regions denote 95\% Monte Carlo confidence intervals for the mean
    upper risk.}
    \label{fig:exp5_vs_M}
\end{figure}

The bound tightness exhibits the same robustness. At $n=200$, mean
tightness varies only from approximately $0.072$ to $0.075$ across the
complete range of $M$. The effective variance $a^\top\Sigma a$ also remains of the same order of magnitude, with only small variations as additional sources are introduced. The stability of the bound is explained by the learned Fusion-CUBV
weights. The optimisation does not force every available source to contribute. Instead, weak or redundant sources can receive zero or near-zero weight. At $n=200$, for example, the learned weights for $M=2$ are approximately $(1,0)$, whereas for $M=4$ they are $(0.991,0,0,0.009)$. For $M=6$, the solution distributes a small fraction of the total weight over selected additional sources,$(0.906,0,0,0.050,0,0.045)$, while for $M=10$ approximately $96\%$ of the weight is again assigned to the primary source. Thus, increasing the number of available sources does not imply that the fused predictor must use all of them. Instead, Fusion-CUBV acts as a selective weighting mechanism that can effectively suppress
uninformative or excessively noisy sources. Empirical coverage remains equal to one at $n=200$ for all considered values of $M$. Similarly, the bounded-discriminability certification probability remains equal to one across the complete source range.
These quantities therefore provide little additional separation between the different values of $M$ in this experiment.

Overall, the results demonstrate that the proposed fusion criterion is
robust to the presence of additional weak information sources. Rather
than suffering systematic degradation as $M$ increases, Fusion-CUBV
preserves both the population risk and the resulting upper-risk
certificate by assigning negligible weight to sources that do not
improve the bound.

\subsection{Experiment VI: Structurally heterogeneous information sources}

The final experiment evaluates Fusion-CUBV under a more demanding
multi-source setting in which additional sources are not merely noisier
versions of the same representation, but exhibit genuinely different
class geometries. Each additional source is generated from a distinct
structural profile, with progressively increasing non-linearity and
multimodal complexity.

The first source follows a linear Gaussian classification model with substantial class overlap, yielding a population risk of approximately $0.22$. The second source introduces multimodal structure while retaining useful discriminative information, whereas subsequent sources exhibit increasingly complex non-linear geometries that are difficult to capture with the linear classifiers considered here. Thus, increasing the number of sources simultaneously increases the structural heterogeneity of the available representations and introduces sources with markedly different predictive risks. We consider $M\in\{2,4,6,8,10\}$, together with the sample-size grid used in
previous experiments. As in the previous fusion experiments, Fusion-CUBV weights are learned exclusively from the calibration repetitions and subsequently frozen before evaluation on the independent Monte Carlo repetitions.

\subsubsection{Results}

Figure~\ref{fig:exp6_bound_vs_n} shows the fused upper risk as a
function of sample size for different numbers of structurally
heterogeneous sources. Despite the progressively more non-linear
source profiles, the resulting curves remain remarkably close
across all values of $M$. The upper risk decreases systematically with sample size, from
approximately $0.34$ at $n=50$ to approximately $0.19$ at $n=5000$,
with only minor differences across the number of sources. At
$M=10$, for example, the mean upper risk decreases from $0.3411$
to $0.1859$ over this range. Thus, increasing source heterogeneity
does not produce a systematic deterioration of the fused risk
certificate, even when several additional sources exhibit
substantially poorer individual predictive performance. The corresponding population risks are shown in
Fig.~\ref{fig:exp6_true_risk}. Across the complete sample-size range,
the mean true fused risk remains between approximately $0.17$
and $0.20$ for all values of $M$. For $M=10$, it decreases from
$0.1979$ at $n=50$ to $0.1733$ at $n=5000$.

\begin{figure}[t]
    \centering
    \includegraphics[width=\columnwidth]
    {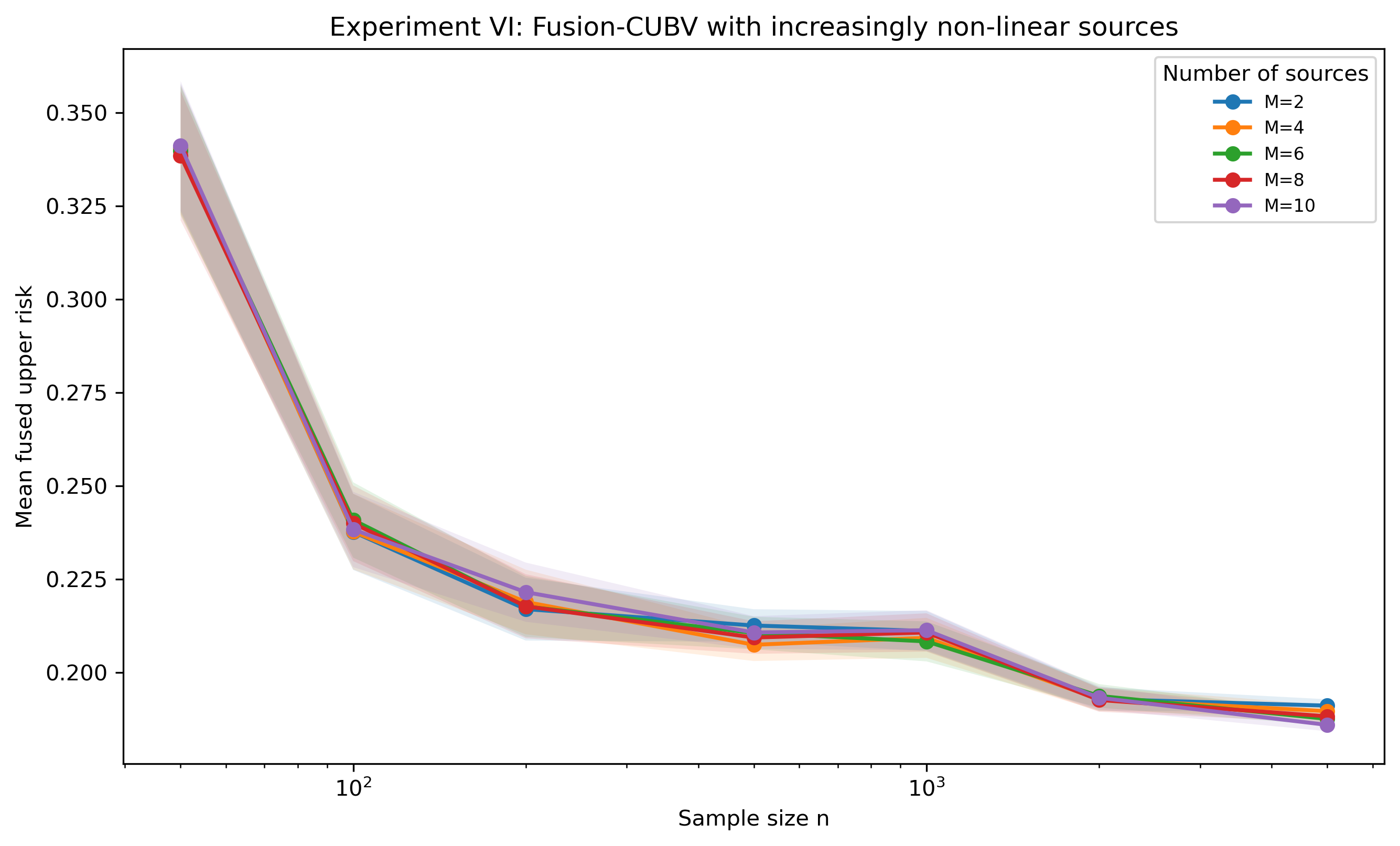}
    \caption{Fusion-CUBV mean upper risk as a function of sample size
    for increasing numbers of structurally heterogeneous information
    sources. Shaded regions denote 95\% Monte Carlo confidence
    intervals.}
    \label{fig:exp6_bound_vs_n}
\end{figure}

\begin{figure}[t]
    \centering
    \includegraphics[width=\columnwidth]
    {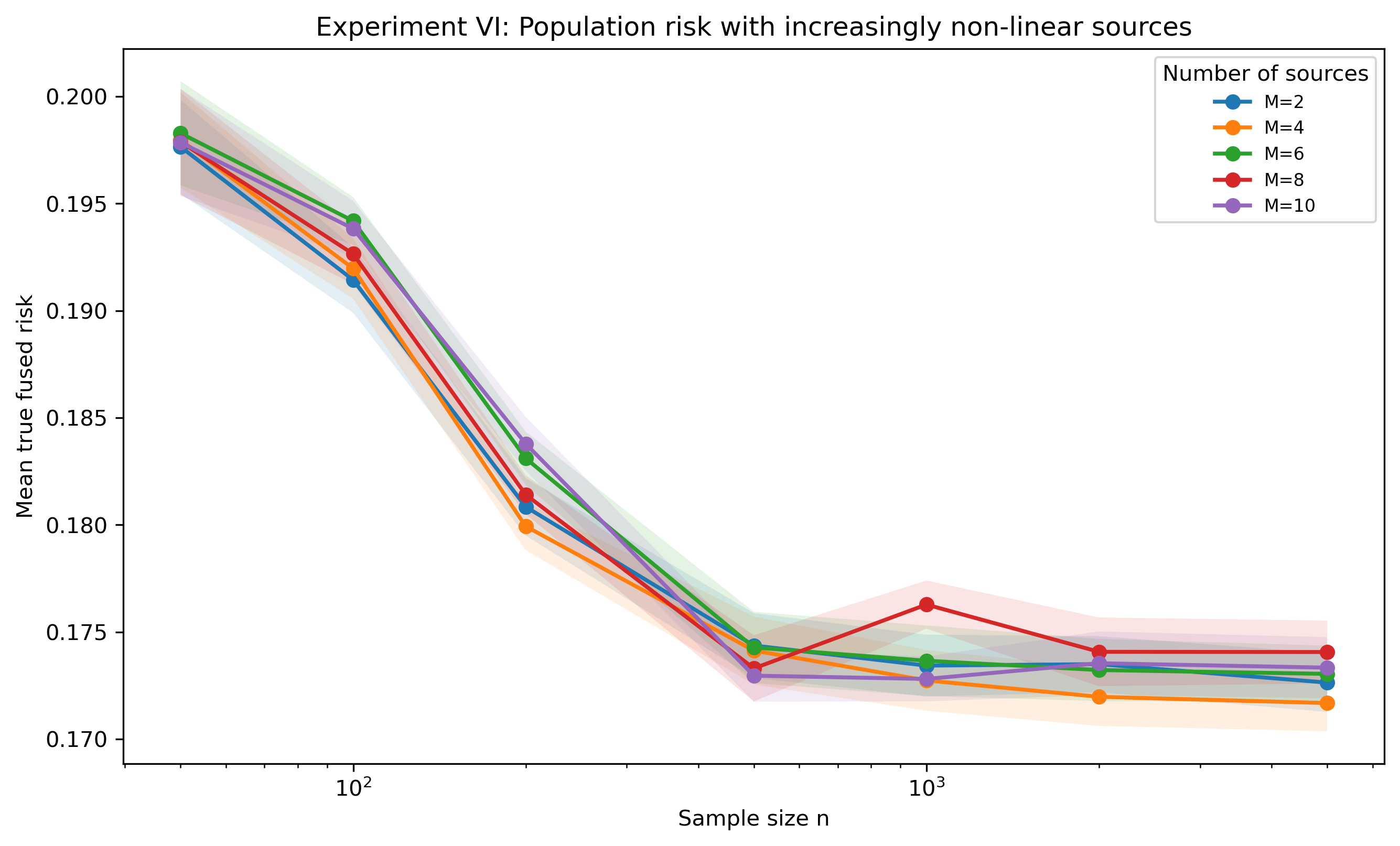}
    \caption{Mean population risk of the Fusion-CUBV predictor as a
    function of sample size for increasing numbers of structurally
    heterogeneous sources. Shaded regions denote 95\% Monte Carlo
    confidence intervals.}
    \label{fig:exp6_true_risk}
\end{figure}

These results indicate that the incorporation of increasingly
non-linear sources does not necessarily increase the fused
population risk. Importantly, the stability is not explained by
an artificially negligible baseline risk, but by the ability of
the fusion mechanism to allocate weight to the more informative
sources while suppressing contributions from poorly performing
representations. The stability of the result is particularly clear at the
representative sample size $n=200$. As shown in Fig.~\ref{fig:exp6_vs_M}, the mean fused upper risk remains between
approximately $0.217$ and $0.221$ as the number of sources
increases from $M=2$ to $M=10$. The corresponding true fused
risk remains between approximately $0.180$ and $0.184$.

\begin{figure}[t]
    \centering
    \includegraphics[width=\columnwidth]
    {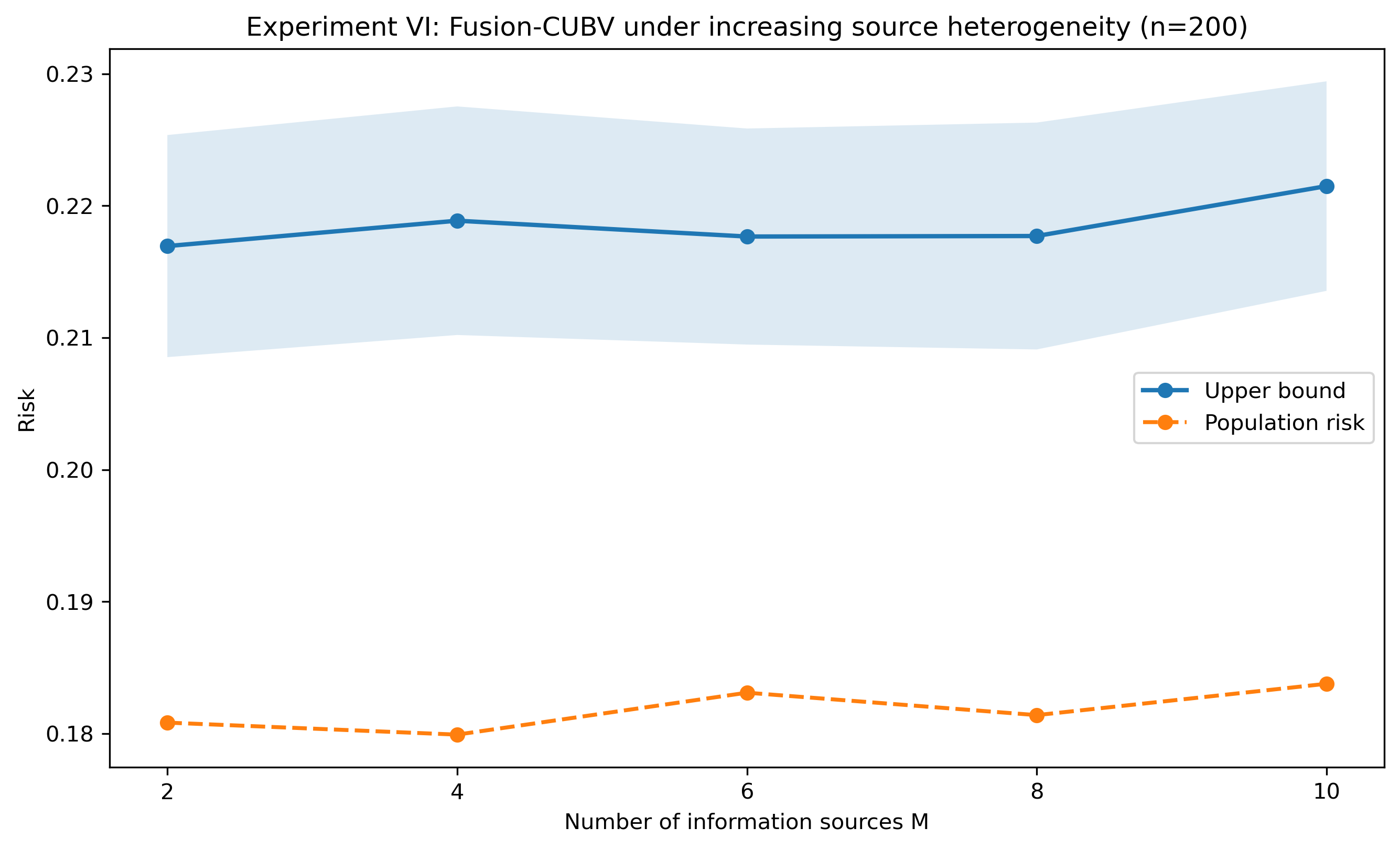}
    \caption{Fusion-CUBV upper risk and corresponding population risk
    as the number of structurally heterogeneous sources increases at
    $n=200$. Shaded regions denote 95\% Monte Carlo confidence
    intervals for the mean upper risk.}
    \label{fig:exp6_vs_M}
\end{figure}

The mean tightness at $n=200$ remains relatively small across the
complete source range, varying between approximately $0.035$
and $0.039$. The effective variance $a^\top\Sigma a$ also remains
stable, indicating that the fusion mechanism does not accumulate
substantial additional uncertainty as structurally heterogeneous
sources are introduced.

An inspection of the learned fusion weights provides further
insight into this behaviour. At $n=200$ and $M=10$, Fusion-CUBV
assigns approximately $0.176$ of the total weight to the first
source and $0.824$ to the second, while assigning negligible
weight to the remaining sources. The corresponding individual
population risks are approximately $0.22$ and $0.18$ for the
first two sources, whereas the subsequent non-linear sources
exhibit risks close to $0.5$.  Consequently, the stability of the fused risk is primarily
explained by adaptive source selection and weighting rather
than by an improvement obtained through combining all available
representations. As the sample size increases, the weights
become increasingly concentrated on the second source, which
exhibits the lowest individual risk in this experimental setting.

At $n=200$, empirical coverage ranges from approximately $0.933$ to $0.967$ across the different source counts. The
bounded-discriminability certification probability, $\Pr(U<0.5)$, is equal to one for every considered value of $M$ at this sample size. Thus, although empirical coverage is not perfect in every configuration, all evaluated repetitions yield non-trivial certificates below the chance-level risk threshold.

These results extend the findings of Experiment V. There,
robustness was observed when progressively weaker and noisier
versions of the same underlying representation were added.
Here, the same qualitative behaviour persists when the
additional sources differ in their actual class geometry.
Fusion-CUBV therefore appears able to preserve a stable risk
certificate in the presence of structurally heterogeneous and
increasingly non-linear sources. Importantly, this experiment illustrates robustness to the
inclusion of poorly performing representations, rather than
demonstrating that all heterogeneous sources contribute
complementary predictive information. The results support the
interpretation of Fusion-CUBV as a risk-aware fusion mechanism
capable of selectively suppressing sources whose inclusion
would otherwise deteriorate the aggregate risk.

\section{Why can resubstitution be tighter?}

The behaviour observed in the preceding experiments can be understood
in terms of two distinct effects: reduced effective training size under
partitioning and instability in the representation of heterogeneous
sub-populations across folds. Suppose the sample contains $N_c$ latent sub-populations. Under $K$-fold cross-validation, each trained model sees only
$ n_{\mathrm{train}}=\frac{K-1}{K}n$ observations. More importantly, this reduction need not be uniform across latent modes. When some sub-populations are weakly represented, individual folds may contain only a small number of observations from a particular mode, or may omit it almost entirely. A crude structural occupancy indicator is $\eta=\frac{n}{dN_c}$, which can be interpreted as the number of observations available per dimension and mixture component. Small $\eta$ indicates that individual sub-populations may be poorly supported. However, $\eta$ should not be interpreted as a complete measure of problem complexity. As shown in Experiments II and III, increasing the number of mixture components does not necessarily increase classification difficulty when the global decision geometry remains compatible with the learner. The relevant mechanism is therefore better represented as:
\begin{equation*}
\begin{aligned}
&\text{heterogeneity}
+\text{small }n
+\text{sparse mode occupancy}
\\
&\qquad\Longrightarrow
\text{fold variability}
\Longrightarrow
\text{model variability}
\\
&\qquad\Longrightarrow
\text{wider K-fold uncertainty}.
\end{aligned}
\end{equation*}

Resubstitution avoids this partition-induced loss of support by training
the final classifier on the complete sample. Its empirical risk is
optimistically biased, but when this optimism is controlled by a valid
upper correction, the resulting risk certificate can be narrower than
the corresponding cross-validation bound. The conclusion is therefore not that resubstitution should universally replace $K$-fold cross-validation. Rather, \emph{partitioning can itself be a source of uncertainty, and in small
heterogeneous samples a properly corrected full-sample risk can be more
informative than a cross-validated risk.} Two mechanisms contribute to this effect. First, overlapping cross-validation folds induce statistical dependence, reducing the effective number of independent validation estimates through $K_{\mathrm{eff}}$, as formalised in Section~5. Second, heterogeneous sub-populations can be represented unevenly across folds, producing additional variability in the fitted models. These mechanisms are conceptually distinct: fold dependence reduces the effective amount of
independent validation information, whereas fold-composition
variability changes the training distribution seen by each fitted
classifier. In small heterogeneous samples, both effects can operate
simultaneously.

\section{Implications for neuroimaging}

Neuroimaging provides a natural motivating domain for the proposed framework \cite{gorriz2026latent}. Cross-validation is widely used in brain decoding and biomarker assessment, yet small neuroimaging samples can lead to highly variable performance estimates, and variability across folds may substantially understate the uncertainty associated with model generalisation \cite{varoquaux2017,varoquaux2018}. Even after dimensionality reduction or deep feature extraction, diagnostic classes are unlikely to form homogeneous populations. Differences in age, sex, acquisition protocol, disease subtype, comorbidity, disease stage, scanner characteristics, and biological variability can induce latent sub-populations and structurally heterogeneous feature distributions.

The experiments presented here suggest that the relevant issue is not multimodality alone. As shown in Experiments II and III, multiple modes can remain compatible with a simple classifier when the global decision geometry is preserved, whereas comparatively modest structural changes can substantially reduce discriminability when they induce model mismatch. A more realistic working hypothesis for neuroimaging is therefore
\begin{equation*}
\begin{aligned}
&\text{small }n
+\text{heterogeneous representations}
\\
&\qquad
+\text{possible model mismatch}.
\end{aligned}
\end{equation*}

In such settings, conventional $K$-fold validation may be sensitive not only to sample size but also to which latent sub-populations are represented in each training fold. Bounded discriminability poses a
different question: rather than asking only whether a representation achieves high cross-validated accuracy, one asks whether its future risk can be bounded tightly enough to support a discriminability claim. For a single latent feature $z_j$, one may associate a certified upper risk $U_j(\delta)$, whereas for a feature subset $S$ one may consider $U_S(\delta)$. This leads naturally to the selection criterion $S^\star=\arg\min_S U_S(\delta)$, instead of selecting features solely according to $S^\star= \arg\max_S \mathrm{Accuracy}(S)$. This distinction is particularly relevant in neuroimaging, where two representations may attain similar empirical accuracy while supporting very different levels of statistical evidence about future discriminability. The bounded-discriminability criterion therefore provides a way to distinguish predictive performance from the strength of the corresponding generalisation certificate.

A real-data proof of concept is reported in Appendix~\ref{app:real-neuroimaging}. We considered unsupervised whole-brain latent representations learned from structural MRI in a balanced NOR--AD classification problem. Each latent coordinate was evaluated independently using deterministic and PAC--Bayesian upper-risk constructions, while a completely held-out subject-level audit cohort was used only as an external empirical reference. The real-data experiment reproduces several of the behaviours suggested by the synthetic analysis. In particular, individual latent coordinates
with comparable empirical discrimination can require substantially different amounts of data before their upper risks fall below the balanced binary chance-risk level. Moreover, combining all coordinates does not necessarily strengthen the corresponding certificate: although the joint latent representation retains similar empirical predictive performance, its PAC--Bayesian complexity penalty is larger than that of the individual one-dimensional representations. A label-permutation negative control provides an additional check. With the learned latent representation held fixed, destroying the association between latent coordinates and diagnosis moves the deterministic upper risks above the chance-risk threshold. The complete experimental protocol, audit analysis, sample-size curves, and permutation results are provided in Appendix~\ref{app:real-neuroimaging}. Leave-site-out validation remains an important extension for future work, as it would test whether a certificate obtained within one acquisition environment remains informative under scanner or site shift.

\section{Discussion}

The proposed framework reframes validation as an uncertainty problem.
A cross-validated accuracy or mean risk is an observed performance
estimate; by itself, it does not quantify the strength of the evidence
supporting future generalisation. The $\gamma$-CUBV formulation makes
the additional ingredients explicit by separating empirical performance
from the concentration mechanism used to control sampling uncertainty,
dependence, model complexity, and, in the multi-source setting,
information-fusion uncertainty.

The experiments support several complementary conclusions. First,
sample size controls how rapidly useful risk certificates can be
obtained. Second, structural complexity cannot be reduced to
multimodality alone: the geometry of the representation relative to the
learner is crucial. Third, partitioning introduces an additional source
of uncertainty that can make corrected full-sample estimates competitive
with, or tighter than, cross-validation bounds in some regimes.
Finally, bound-aware information fusion can remain stable when weak,
noisy, or structurally heterogeneous sources are added, because
uninformative sources need not receive substantial weight.

\subsection{Why dependence matters}

If $K$ folds represented independent replications, averaging them would
reduce uncertainty at the usual $K^{-1/2}$ rate. Standard $K$-fold
cross-validation does not generally admit this interpretation because
the fitted models are trained on strongly overlapping datasets.
Equation~\eqref{eq:keff} provides an interpretable representation of
the resulting loss of effective independent information under an
equicorrelation model. This special case should not be interpreted as implying that a single scalar $\rho$ fully characterises cross-validation dependence. The more general result is the positive-semidefinite covariance-proxy formulation, which allows non-uniform dependence across folds.

A second, conceptually distinct mechanism arises from heterogeneous
data. In small samples, latent sub-populations may be represented
unevenly across training folds, so that different folds induce
meaningfully different fitted models. Dependence between validation
estimates and fold-composition variability are therefore distinct
phenomena: the former reduces effective independent information,
whereas the latter changes the learning problem encountered by each
model. Both can contribute simultaneously to cross-validation
uncertainty.

\subsection{Accuracy versus evidence}

The experiments illustrate that predictive performance and statistical
evidence are not interchangeable. Two procedures can achieve similar
empirical risks while producing substantially different upper-risk
certificates. Within the proposed framework, the strength of the
certificate depends not only on empirical performance but also on
sampling uncertainty, dependence, and model complexity. This distinction is particularly important for bounded discriminability. A representation may achieve an apparently useful classification accuracy while the available sample remains insufficient to certify $U<0.5$ for a balanced binary problem. Conversely, as sample size increases, the same predictive problem may cross from an empirically promising regime into one in which better-than-chance future performance can be certified at the prescribed confidence level. The certification experiments therefore provide information that raw
accuracy alone cannot supply: they identify when the observed predictive
signal is accompanied by a sufficiently strong future-risk guarantee.

\subsection{Multimodality, geometry, and model mismatch}

The structural experiments refine an important point from our previous
work. Multimodality alone is not synonymous with classification
difficulty. In the linear-multimodal experiments, increasing the number
of mixture components did not destroy discriminability as long as a
globally exploitable decision geometry remained present. By contrast, the nonlinear-alternating configurations rapidly drove the fixed linear classifier toward chance performance even with relatively few components. Thus, the relevant quantity is not simply the number of latent modes but the interaction between sample size, latent heterogeneity, and the inductive bias of the learner. This distinction also clarifies the interpretation of simple structural indicators such as $\eta=\frac{n}{dN_c}$. Such quantities may indicate sparse support of latent sub-populations, but they should not be treated as complete measures of problem complexity. A representation with many modes may still be easy for a given learner, whereas a lower-dimensional or less multimodal representation may be intrinsically difficult if its geometry is poorly matched to the classifier.

\subsection{Information fusion versus uncertainty fusion}

Information-fusion methods commonly optimise predictive performance
after combining several sources. The optimisation in Eq.~\eqref{eq:fusion-opt} addresses a related but different objective: it also accounts for the uncertainty associated with each source and the dependence among their generalisation gaps.

Experiment IV shows why this distinction matters. Under
non-homogeneously sampled multimodal sources, Fusion-CUBV can produce a
substantially smaller upper-risk certificate than empirical-risk
weighting even when the latter attains a lower true fused risk. At
small sample sizes, the uncertainty and covariance terms materially
alter the selected weights and favour multi-source combinations. As
the sample size increases, this difference diminishes and the
Fusion-CUBV solution approaches empirical-risk source selection. Thus,
predictive risk and evidence-adjusted future risk need not induce the
same fusion rule.

From a practical perspective, these results suggest that explicitly accounting for uncertainty may improve upon conventional fusion strategies based exclusively on empirical performance. The present experiment uses independent Monte Carlo calibration to estimate generalisation-gap statistics, providing a controlled setting in which the effect of uncertainty-aware weighting can be isolated. Although such calibration is not directly available in a single real-world dataset, the underlying principle remains relevant: empirical risk alone does not account for estimation variability or dependence between information sources. Extensions based on observable validation statistics, including K-fold cross-validation, could provide practical approximations, provided that the dependence induced by overlapping training sets and adaptive weight selection is appropriately addressed.

The subsequent source-number experiments provide an additional
observation. Adding progressively weaker or noisier sources did not
systematically deteriorate the Fusion-CUBV certificate. The same
qualitative robustness was observed when the additional sources
possessed genuinely different and increasingly non-linear geometries.
In both cases, the optimisation could assign negligible weight to
sources that did not improve the resulting certificate. These results suggest that the relevant question in multi-source learning is not merely whether an additional source contains predictive information, but whether it contributes sufficiently independent and reliable information to improve the certified future risk.

The Safe-Fusion-CUBV construction illustrates the opposite extreme.
The conservative positive-semidefinite inflation used to protect
against covariance-estimation uncertainty produced trivial bounds in
the present simulations. This does not invalidate the safe construction,
but it demonstrates the practical cost of obtaining protection from a
deliberately conservative covariance upper bound.

\subsection{Relationship to the original K-fold CUBV}

The present framework is intended as a generalisation of the original
K-fold CUBV formulation \cite{gorriz2026kfold}, rather than as a
replacement for it. Classical concentration inequalities, PAC-Bayesian
bounds, dependency-aware constructions, and information-fusion bounds
are represented through a common exponential-envelope formulation.

The principal conceptual extension is to make the cumulant envelope
$\gamma$ explicit and allow it to encode different assumptions about
variance, dependence, model complexity, and information fusion. The
bounded-discriminability formulation then applies the resulting upper
risk to a different scientific question: whether a feature or
representation supports a statistically controlled claim about future
discriminability.

\subsection{Resubstitution, K-fold, and the coverage criterion}

The sample-size experiments reproduce an important qualitative
observation from our previous CUBV work \cite{gorriz2026kfold}:
cross-validation uncertainty is not determined solely by the nominal
sample size. Partitioning can reduce the support available to individual
fitted models and, in heterogeneous samples, alter the representation
of latent sub-populations across folds.

In the present simulations, corrected resubstitution and $K$-fold
bounds become increasingly similar as sample size grows. At smaller
sample sizes, however, differences between the two paradigms can remain
substantial. This should not be interpreted as a general recommendation
to replace cross-validation with resubstitution. Resubstitution is
optimistically biased, and a smaller reported upper bound is meaningful
only when the correction used to construct it provides adequate
coverage.

This point is visible in the Monte Carlo results. Most constructions
exhibited empirical coverage close to or above the nominal level, but
isolated deviations were observed for some calibrated or
data-dependent bounds. The experiments therefore reinforce the need to
report coverage together with tightness: numerical sharpness without
adequate coverage is not evidence of a superior bound.

\subsection{Limitations and open theoretical questions}

Several aspects require further development before the
dependency-aware and bounded-discriminability constructions can be
treated as turnkey statistical procedures.

First, $\bm{\Sigma}$ is a covariance \emph{proxy} for latent
generalisation gaps. Estimating it and subsequently inserting it into a
risk bound introduces another layer of uncertainty. Bengio and
Grandvalet \cite{bengio2004} showed that no universally unbiased
estimator of the variance of $K$-fold cross-validation exists from a
single $K$-fold experiment, while Nadeau and Bengio
\cite{nadeau2003} demonstrated the inferential consequences of ignoring
training-set-induced variability. The positive-semidefinite
upper-confidence construction developed here is intended to expose this
uncertainty rather than conceal it in a plug-in estimate.

The simulations also reveal the practical difficulty of this approach:
the present conservative PSD inflation can be sufficiently large to
produce trivial upper bounds. Developing sharper covariance
upper-confidence constructions is therefore an important theoretical
and practical direction. Second, repeated cross-validation generates additional observations but
does not automatically generate independent ones. Repeated partitions
may assist estimation of dependence, but their own correlation structure
must be accounted for. Third, the direct Hoeffding and empirical-Bernstein inequalities are fixed-hypothesis concentration results. Applying them after fitting a classifier on the same observations does not automatically provide a
post-selection guarantee. In the present experiments they are used as
comparative bounding constructions and their behaviour is assessed
through Monte Carlo coverage. A stronger theoretical treatment should
derive uniform or algorithm-dependent counterparts appropriate to the
learned classifier. 

Fourth, the calibrated MGF constructions depend on an independent calibration/evaluation separation in the present Monte Carlo design. Deployment on a single real dataset would require an analytically
justified dependence model, an external calibration sample, a hierarchical dependence estimate, or another procedure that propagates calibration uncertainty into the final guarantee. Fifth, the real-data proof of concept reported in Appendix~\ref{app:real-neuroimaging} shows that bounded-discriminability behaviour can also be observed in learned neuroimaging representations. However, this empirical result does not resolve the broader theoretical problem of post-training and post-selection validity. The next theoretical stage should therefore derive feature-level upper bounds that explicitly account for model training, fold dependence, and feature-selection
multiplicity, particularly when latent coordinates or feature subsets are
selected adaptively from the same data.

Sixth, selection introduces an additional multiplicity problem. Searching
over algorithms, hyperparameters, values of $K$, fusion rules, features,
or feature subsets and then reporting the smallest upper bound can
invalidate a nominal guarantee unless the selection step is itself
controlled. In particular, selecting $S^\star=\arg\min_S U_S(\delta)$ over many candidate representations requires an appropriate uniform, nested, or multiple-comparison correction. Finally, all bounds considered here address finite-sample generalisation under a common population. Distribution shift between the training population and a future target population is a separate problem. A more complete extension could take the form
\begin{equation*}
\R_{\mathrm{target}}
\leq
\Rh
+
B_{\mathrm{sampling}}
+
B_{\mathrm{dependence}}
+
B_{\mathrm{shift}},
\end{equation*}
where $B_{\mathrm{shift}}$ controls the discrepancy between source and
target populations. This issue is especially relevant to multi-site
neuroimaging, where scanner, acquisition, and population differences
can dominate ordinary finite-sample uncertainty.

\section{Conclusion}

We have proposed $\gamma$-CUBV, a unified exponential framework for constructing upper bounds on generalisation risk from resubstitution, cross-validation, and heterogeneous information sources. The central object is a cumulant-generating envelope $\gamma$, which translates assumptions about uncertainty into an explicit confidence penalty. Classical range-based concentration, variance-aware bounds, dependency-aware constructions, PAC-Bayesian penalties, and multi-source fusion can all be represented within the same structure.

The dependency analysis yields a particularly transparent consequence. Under an equicorrelated sub-Gaussian model, $K$ overlapping folds behave as $\keff= \frac{K}{1+(K-1)\rho}$ effective independent units of evidence. When $\rho$ is large, increasing $K$ may improve the granularity of cross-validation without providing a proportional increase in inferential information. Building on this framework, we introduced \emph{bounded discriminability} as a statistical-learning-theory criterion for validating features and representations. Rather than asking only whether a feature achieves a high empirical or cross-validated accuracy, the criterion asks whether its future classification risk can be bounded tightly enough to support a controlled discriminability claim.

The simulation study revealed several complementary behaviours. First, increasing sample size progressively tightens the resulting risk certificates and increases the probability of certifying better-than-chance future performance. Second, structural complexity cannot be identified with multimodality alone: increasing the number of modes may leave discriminability largely intact when the global decisionngeometry remains compatible with the learner, whereas comparatively modest non-linear structure can drive a fixed linear classifier toward
chance performance. Third, $K$-fold partitioning can introduce
additional uncertainty when heterogeneous sub-populations are weakly represented across folds, so that corrected full-sample estimates may be competitive with, or tighter than, cross-validation bounds in some regimes, provided that their empirical coverage remains adequate.

The information-fusion experiments further showed that Fusion-CUBV can adapt to several distinct forms of source heterogeneity. In the fixed four-source experiment, each source contained a non-homogeneously sampled interlaced multimodal representation of the same binary task. Fusion-CUBV produced tighter upper-risk certificates than uniform and variance-only weighting and, in the small-sample regime, also improved
materially upon empirical-risk weighting despite the latter attaining lower population risk. As the sample size increased, the bound-aware solution approached empirical-risk source selection. Experiments V and VI further showed that the resulting certificates remained stable when progressively weaker, noisier, or structurally incompatible sources were added. Conversely, the conservative Safe-Fusion-CUBV construction
demonstrated the practical cost of strong protection against
covariance-estimation uncertainty, producing uninformative bounds in the present simulations.

The broader message is that predictive performance and evidence should not be identified with one another, whether at the level of a fitted model or at the level of a single feature. Accuracy is an observed average. Generalisation requires an uncertainty statement compatible with the data, the learning procedure, the dependence structure, and the hypothesis class. The $\gamma$-CUBV framework, and its bounded-discriminability instantiation, are designed to make those components explicit.

The real neuroimaging proof of concept reported in Appendix~\ref{app:real-neuroimaging} provides an initial demonstration of this principle in unsupervised latent representations learned from structural MRI. Individual latent coordinates can support similar empirical classification performance while requiring substantially different sample
sizes before their upper risks become non-vacuous. Moreover, combining all latent coordinates need not strengthen the corresponding statistical evidence, because improved or preserved empirical performance can be offset by a larger model-complexity penalty.

The next step is therefore not simply to test whether bounded
discriminability can be observed in real neuroimaging data, but to determine how robust these certificates remain under external-cohort validation, multi-centre acquisition, distribution shift, and adaptive feature selection. This perspective is particularly relevant to low-dimensional latent representations in neuroimaging, where small cohorts, biological
heterogeneity, acquisition variability, and possible model mismatch are
common. A natural next step is therefore to test bounded discriminability on real neuroimaging representations and to determine whether feature-level upper-risk certificates improve reproducibility and scientific interpretability beyond conventional accuracy and cross-validation alone.

\section*{Data and code availability}
Code implementing the proposed bounds, the multimodal imbalanced generator, and all simulation experiments will be made publicly available upon publication at \href{https://github.com/SiPBA}{https://github.com/SiPBA}

\section*{Declaration of competing interest}
The authors declare that they have no known competing financial interests or personal relationships that could have appeared to influence the work reported in this paper.

\section*{Acknowledgements}
The author also acknowledges the Spanish Ministry of Science, Innovation and Universities for the rather unusual circumstance of leaving a Highly Cited Researcher in Spain — one of only 42 nationally — outside the National Research Plan for reasons unrelated to the author’s scientific record, productivity, impact, or academic merits.

The authors acknowledge the use of OpenAI's ChatGPT for English-language editing and for assistance in improving the organisation and presentation of the manuscript. All scientific content, mathematical derivations, experimental design, results, and conclusions were independently developed, verified, and approved by the authors.

\newpage

\appendix


\section{Experiment VII: bounded discriminability in real neuroimaging data}
\label{app:real-neuroimaging}

The synthetic experiments allow direct evaluation of coverage because the
data-generating distribution is known and population risk can be approximated
using arbitrarily large independent samples. We finally considered a real
neuroimaging application, for which the population risk is unknown. The goal
of this experiment is therefore different: to determine whether individual
coordinates of a learned low-dimensional representation support non-vacuous
upper-risk statements, whether those statements are consistent with an
independent audit sample, and how the resulting evidence changes with sample
size and under label randomisation.

\subsection{Dataset and development--audit design}

We considered a binary structural-MRI discrimination problem between cognitively normal subjects (NOR) and subjects diagnosed with Alzheimer's disease (AD)\footnote{Data used in the preparation of this paper were obtained from the ADNI database (\url{adni.loni.usc.edu}). Please see \cite{gorriz2026latent} for further details.}. After balancing the two diagnostic groups without replacement,
the dataset contained 376 distinct subjects, with 188 NOR and 188 AD observations. Each subject contributed a single MRI volume.

The balanced sample was stratified into a development cohort of $282$ subjects, comprising $141$ NOR and $141$ AD observations, and an independent audit cohort of 94 subjects, comprising 47 NOR and 47 AD observations. Thus,

\begin{equation*}
    \mathcal D_{\mathrm{dev}}
    \cap
    \mathcal D_{\mathrm{audit}}
    =
    \varnothing .
\end{equation*}

The audit subjects were excluded from autoencoder training, classifier fitting, bound construction, and the sample-size analyses described below. Audit risks are therefore used only as independent out-of-sample empirical references. They should not be interpreted as known population risks and cannot provide, by themselves, an empirical coverage estimate analogous to that available in the synthetic experiments.

\subsection{Unsupervised whole-brain latent representation}

A three-dimensional autoencoder was trained using only the development MRI
volumes. Diagnostic labels were not used during representation learning; the
autoencoder was optimised exclusively through reconstruction. The resulting reconstruction mean-squared error was $\mathrm{MSE}_{\mathrm{AE}} = 0.04758$. After training, the encoder was frozen and applied separately to the
development and audit cohorts, $X_{\mathrm{dev}} \longrightarrow  Z_{\mathrm{dev}}$ and $X_{\mathrm{audit}}\longrightarrow Z_{\mathrm{audit}}$, where each MRI volume was represented by an eight-dimensional latent vector,
$ Z=(z_1,z_2,\ldots,z_8)\in\mathbb R^8$. The eight coordinates are learned from the complete brain volume. They are therefore global latent variables and should not be interpreted as individual anatomical regions. As discussed in our previous analysis of latent-space projections and atlases in deep neuroimaging \cite{gorriz2026latent}, individual autoencoder coordinates may combine distributed anatomical information and should not be assigned a direct neuroanatomical interpretation without an explicit localisation analysis. In particular, $z_j$ cannot be identified directly with the hippocampus, temporal cortex, or another specific structure. The principal analysis below consequently asks a representation-level question: \emph{Does an individual unsupervised whole-brain latent coordinate contain
sufficient evidence to support better-than-chance NOR--AD discrimination?}

\subsection{Component-wise bounded discriminability}

Each latent coordinate was considered separately. For every
$j\in\{1,\ldots,8\}$, a standardised one-dimensional logistic classifier was
fitted using only $(z_j,y)_{\mathrm{dev}}$. This gives the sequence $z_j
\longrightarrow \widehat h_j \longrightarrow \widehat R_j \longrightarrow
U_j(\delta),\delta=0.05$. For balanced binary classification, the chance-risk level is $R_{\mathrm{chance}}=0.5$, so that $U_j(\delta)<0.5$ provides the operational bounded-discriminability criterion used in this experiment.

We report Hoeffding and empirical-Bernstein upper-risk corrections as
fixed-hypothesis-style reference rules. Because the logistic classifier is
fitted using the same development observations on which the resubstitution
error is computed, these direct forms do not themselves pay for data-dependent hypothesis selection. A uniform VC-Hoeffding construction is
therefore also reported as a conservative reference. PAC--Bayes provides a distinct post-training construction. For every coordinate we used
\begin{equation*}
P
=
\mathcal N(0,4^2I),
\qquad
Q_j
=
\mathcal N(\widehat\theta_j,0.1^2I),
\end{equation*}
where $\widehat\theta_j$ denotes the fitted logistic-regression parameters.
The associated upper bound targets the Gibbs risk $R_G(Q_j)= \mathbb E_{\theta\sim Q_j} R(h_\theta)$, rather than the risk of the deterministic posterior-centre classifier. Consequently, deterministic and PAC--Bayesian quantities are kept separate throughout the analysis.

\subsection{Component-wise results using all development subjects}

Table~\ref{tab:real-component-bounds} reports the results obtained from the
complete development sample, $n=282$. All eight coordinates satisfied $U_H<0.5$ and $U_B<0.5$. More importantly, all eight coordinates also satisfied the PAC--Bayesian criterion $U_{PB}(z_j)<0.5$. The resulting PAC--Bayes upper risks ranged from approximately $0.435$ to $0.499$. Thus, although the eight coordinates originate from the same unsupervised representation, they do not carry identical amounts of
classification evidence. For all coordinates, the deterministic audit risk was below the corresponding Hoeffding and Bernstein upper-risk values. Likewise, the independent Gibbs audit risks were below their PAC--Bayes upper risks. These comparisons provide an external empirical consistency check, but they are not direct tests of population-risk coverage.

The uniform VC-Hoeffding construction was substantially more conservative.
Coordinate-wise upper risks ranged from approximately $0.958$ to $1$, despite
the use of only a single latent variable. This illustrates the statistical
cost of protecting uniformly against data-dependent selection over the
complete affine-linear hypothesis class.

\begin{table*}[t]
\centering
\small
\begin{tabular}{cccccccc}
\toprule
Feature &
$\widehat R$ &
$R_{\mathrm{audit}}$ &
$U_H$ &
$U_B$ &
$\widehat R_G(Q)$ &
$R_{G,\mathrm{audit}}$ &
$U_{PB}$ \\
\midrule
$z_1$ & 0.287 & 0.330 & 0.360 & 0.360 & 0.284 & 0.338 & 0.435 \\
$z_2$ & 0.323 & 0.351 & 0.396 & 0.398 & 0.338 & 0.329 & 0.490 \\
$z_3$ & 0.316 & 0.277 & 0.388 & 0.391 & 0.311 & 0.289 & 0.463 \\
$z_4$ & 0.277 & 0.340 & 0.349 & 0.349 & 0.284 & 0.333 & 0.435 \\
$z_5$ & 0.326 & 0.298 & 0.399 & 0.402 & 0.328 & 0.316 & 0.479 \\
$z_6$ & 0.291 & 0.319 & 0.364 & 0.364 & 0.290 & 0.324 & 0.442 \\
$z_7$ & 0.298 & 0.362 & 0.371 & 0.372 & 0.307 & 0.337 & 0.458 \\
$z_8$ & 0.348 & 0.330 & 0.420 & 0.424 & 0.348 & 0.340 & 0.499 \\
\bottomrule
\end{tabular}
\caption{Experiment VII: component-wise bounded discriminability using all
$n=282$ development subjects. $\widehat R$ and
$R_{\mathrm{audit}}$ refer to the deterministic logistic classifier, whereas
$\widehat R_G(Q)$ and $R_{G,\mathrm{audit}}$ refer to the corresponding Gibbs
classifier and should therefore be compared with the PAC--Bayes upper risk
$U_{PB}$. The audit cohort is not used to construct any bound.}
\label{tab:real-component-bounds}
\end{table*}

\subsection{The full eight-dimensional representation}

For comparison, we also fitted a classifier using all eight latent
coordinates simultaneously, $(z_1,\ldots,z_8) \longrightarrow \widehat h_Z$. This analysis does \emph{not} average the latent coordinates. All eight
variables enter the logistic classifier jointly. The deterministic classifier obtained $\widehat R_Z = 0.277$, $R_{\mathrm{audit},Z}=
0.287$, and the corresponding reference upper risks were $U_H(Z)= 0.349$
$U_B(Z) = 0.349$. The uniform VC-Hoeffding upper risk was equal to $1$. PAC--Bayes gave $\widehat R_G(Q_Z)= 0.279$, $R_{G,\mathrm{audit}}(Q_Z)=0.299$, but
$U_{PB}(Z) = 0.530$. Hence the complete eight-dimensional representation did not satisfy $U<0.5$ under the PAC--Bayesian criterion, despite retaining a low empirical risk. The difference is explained by the explicit complexity term. For the one-dimensional models, $\KL(Q_j\|P)\simeq 6.4$ whereas for the joint eight-dimensional classifier,$\KL(Q_Z\|P)=28.90$. The larger representation therefore preserves predictive performance while
incurring a substantially larger PAC--Bayesian complexity penalty. This
provides a concrete example of the distinction between predictive accuracy
and statistical evidence: adding features need not produce a stronger
upper-risk certificate.

\subsection{Comparison with K-fold validation}

The complete latent representation was additionally evaluated by five-fold
cross-validation. The observed K-fold empirical risk was $\widehat R_{\mathrm{CV}} = 0.294$, with $U_{H,\mathrm{CV}} = 0.367$, $U_{B,\mathrm{CV}}
= 0.368$. For the corresponding full-development fit, $\widehat R_{\mathrm{full}} = 0.277$, $U_{H,\mathrm{full}} = 0.349$, $U_{B,\mathrm{full}} = 0.349$. Thus, in this dataset, partitioning the development cohort produced slightly larger empirical and upper risks than fitting the same representation using all available development observations. This comparison is descriptive and should not be interpreted as establishing a universal ordering between resubstitution and K-fold validation.

\subsection{Sample-size dependence and variability}

We next studied how the component-wise upper risks evolve with the amount of
labelled development data. Importantly, the autoencoder was \emph{not}
retrained in this experiment. The already learned latent representation was
kept fixed, and only the labelled classification stage was repeated at
different sample sizes. Consequently, this experiment addresses \emph{How many labelled subjects are required to obtain a non-vacuous classification certificate from an already learned representation?} rather than the different question of how autoencoder uncertainty changes when representation learning itself is repeated at smaller sample sizes. For every coordinate, sample sizes $n \in \{25,50,100,200,282\}$ were considered. At each $n<282$, 30 stratified subsamples were drawn without replacement. Within each repetition, the same subjects were used for the Hoeffding, empirical-Bernstein, and PAC--Bayes calculations. The curves report the mean upper risk and empirical 10th--90th percentile bands across repetitions. For PAC--Bayes, the empirical Gibbs risk was estimated using 1000 posterior draws per fitted classifier, reducing the Monte Carlo contribution to thebobserved variability.

We define the observed certification sample size on the evaluated grid as
\begin{equation}
n_{\mathrm{cert}}
=
\min
\left\{
n:
\overline U(n)<0.5
\right\},
\label{eq:n-cert-real}
\end{equation}
where $\overline U(n)$ denotes the mean upper risk across the repeated subsamples. Because only five sample sizes were evaluated, $n_{\mathrm{cert}}$ should be interpreted as a grid crossing, not as an exact continuous threshold.

\begin{table}[t]
\centering
\begin{tabular}{cccc}
\toprule
Feature &
$n_{\mathrm{cert}}^{H}$ &
$n_{\mathrm{cert}}^{B}$ &
$n_{\mathrm{cert}}^{PB}$ \\
\midrule
$z_1$ & 50  & 50  & 200 \\
$z_2$ & 50  & 100 & 282 \\
$z_3$ & 50  & 100 & 200 \\
$z_4$ & 50  & 50  & 200 \\
$z_5$ & 50  & 100 & 282 \\
$z_6$ & 50  & 50  & 200 \\
$z_7$ & 50  & 50  & 200 \\
$z_8$ & 100 & 100 & 282 \\
\bottomrule
\end{tabular}
\caption{First evaluated sample size at which the mean component-wise upper
risk falls below $0.5$. Hoeffding and Bernstein refer to the deterministic
classifier, whereas PAC--Bayes targets Gibbs risk.}
\label{tab:real-ncert}
\end{table}

The PAC--Bayes results show a clear hierarchy. Coordinates
$z_1,z_3,z_4,z_6,$ and $z_7$ crossed the $0.5$ threshold at $n=200$ on the
tested grid, whereas $z_2,z_5,$ and $z_8$ required the complete
$n=282$ development cohort. For example, the mean PAC--Bayes upper risk of $z_4$ evolved as $0.782,\; 0.635,\; 0.530,\; 0.468,\; 0.438$ for $n = 25,\;50,\;100,\;200,\;282$, respectively. For the weaker coordinate $z_8$, certification occurred only at the largest sample size. The important effect is therefore not merely a change in classification accuracy. Both the upper risk and its between-sample variability decrease as the number of observed subjects grows.

\begin{figure*}[t]
\centering
\includegraphics[width=.48\linewidth]{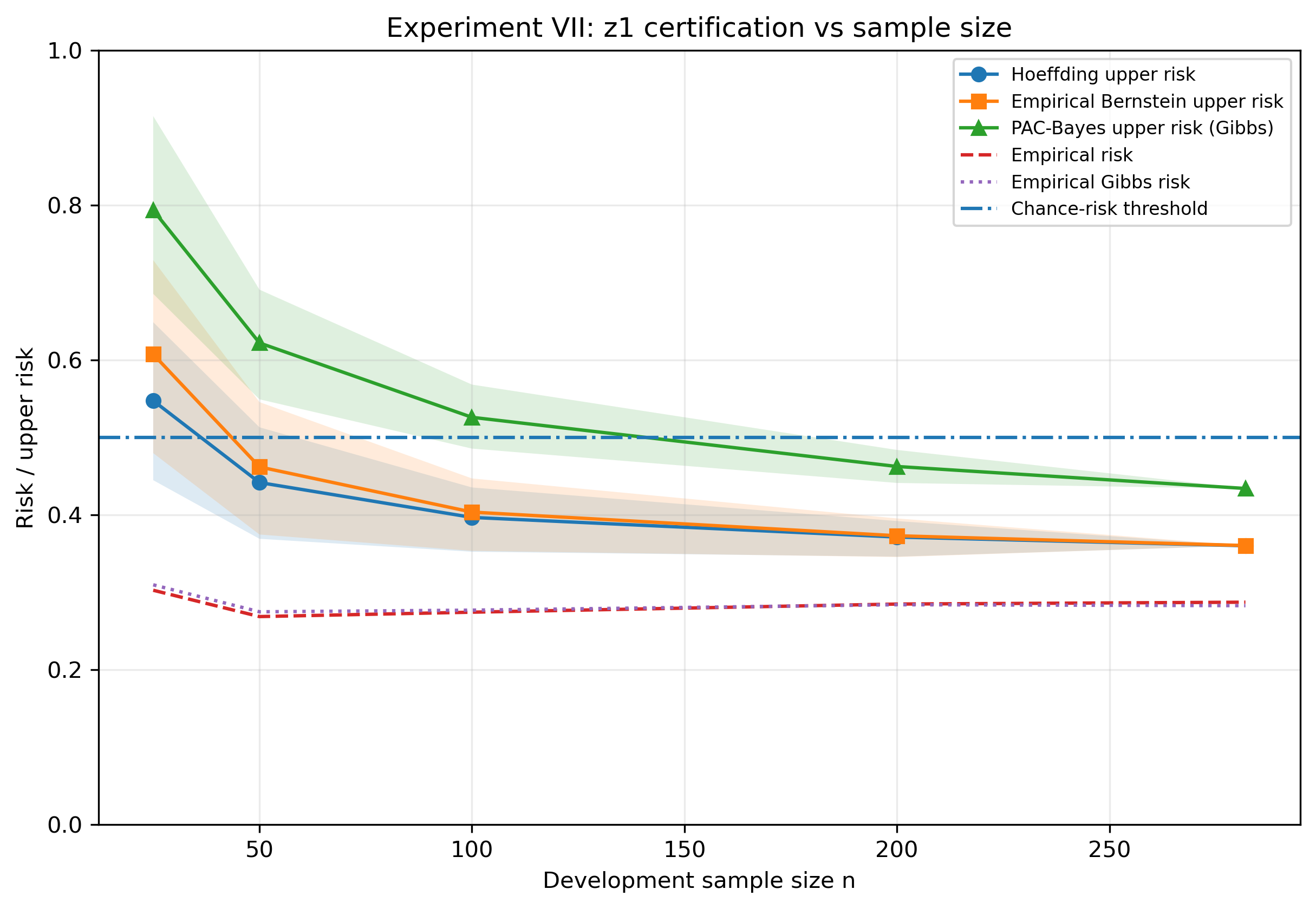}
\includegraphics[width=.48\linewidth]{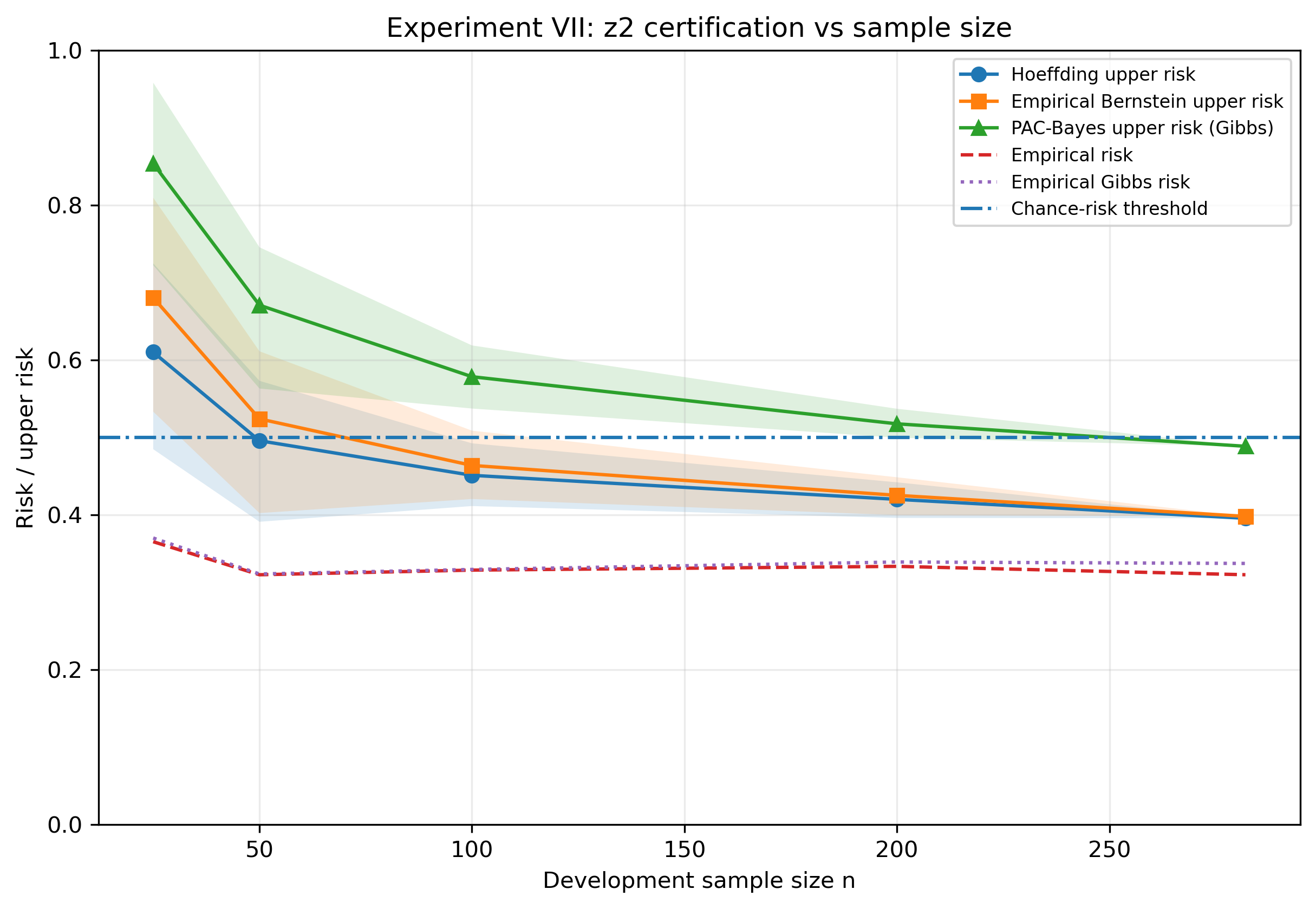}
\includegraphics[width=.48\linewidth]{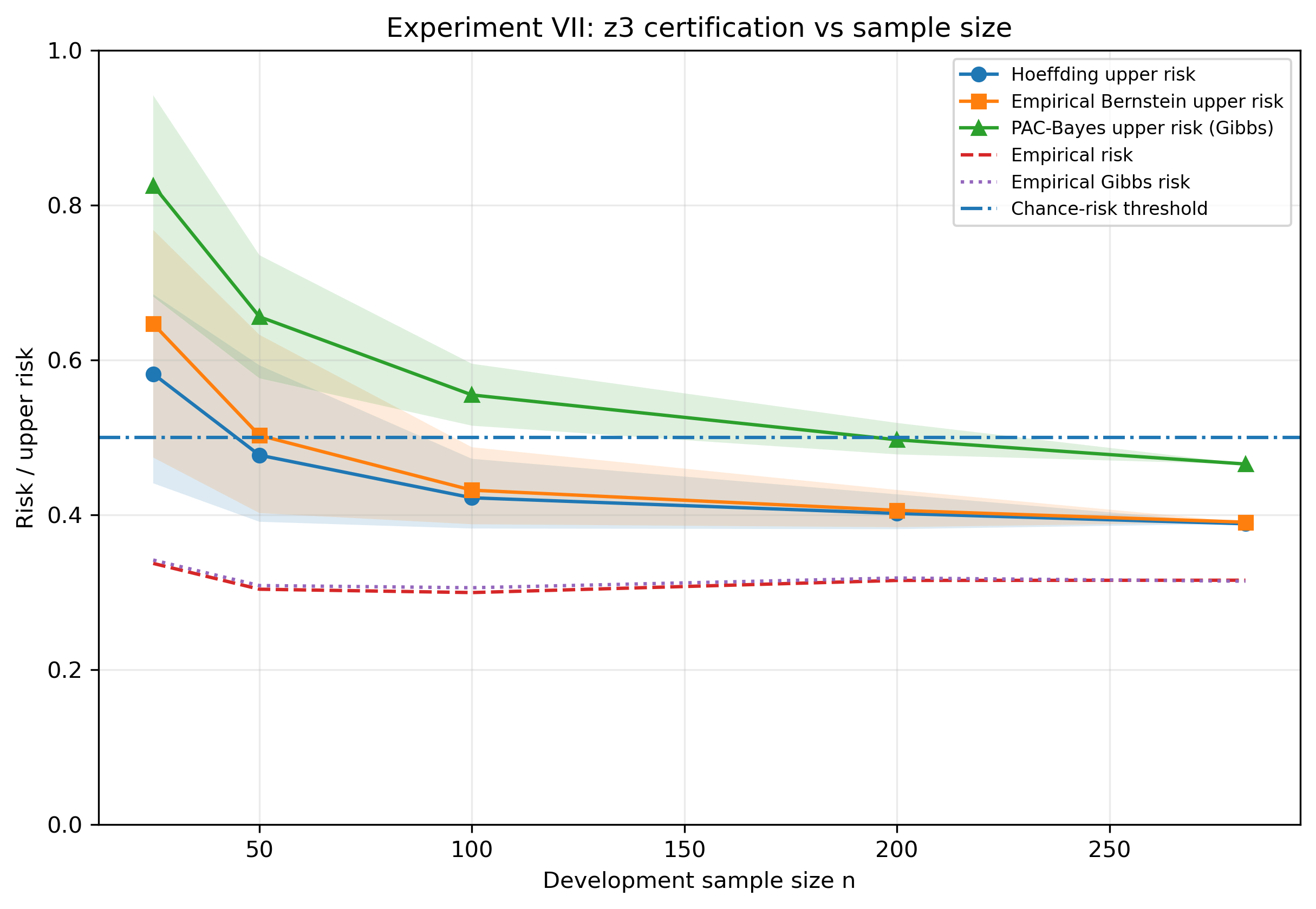}
\includegraphics[width=.48\linewidth]{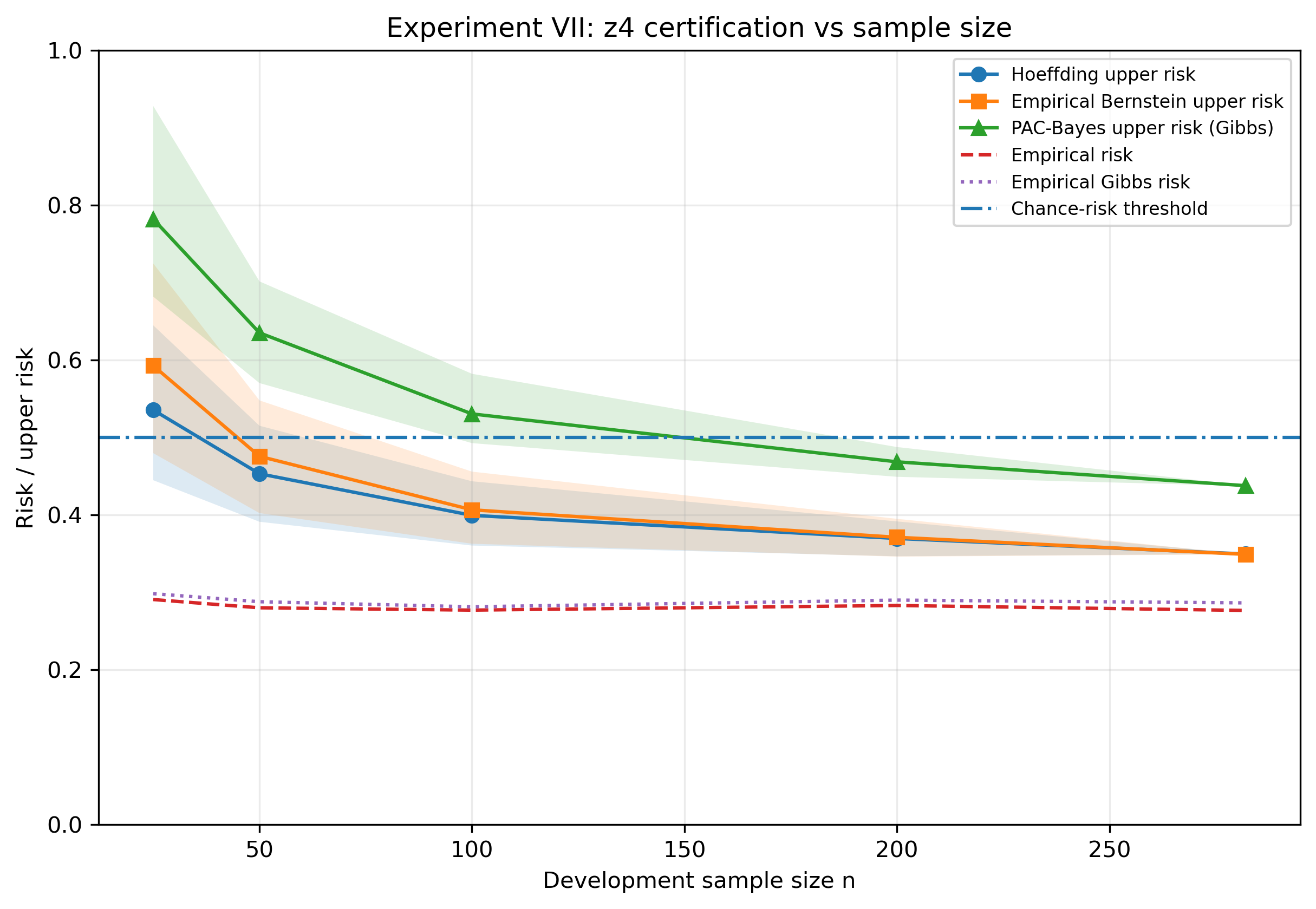}
\caption{Experiment VII: sample-size dependence of bounded discriminability for latent coordinates $z_1$--$z_4$. Solid curves report mean upper risks and the shaded regions show empirical 10th--90th percentile bands across repeated
stratified subsamples. The horizontal reference at $0.5$ denotes balanced binary chance risk. PAC--Bayes targets Gibbs risk; Hoeffding and empirical Bernstein target the deterministic classifier.}
\label{fig:real-z1-z4}
\end{figure*}

\begin{figure*}[t]
\centering
\includegraphics[width=.48\linewidth]{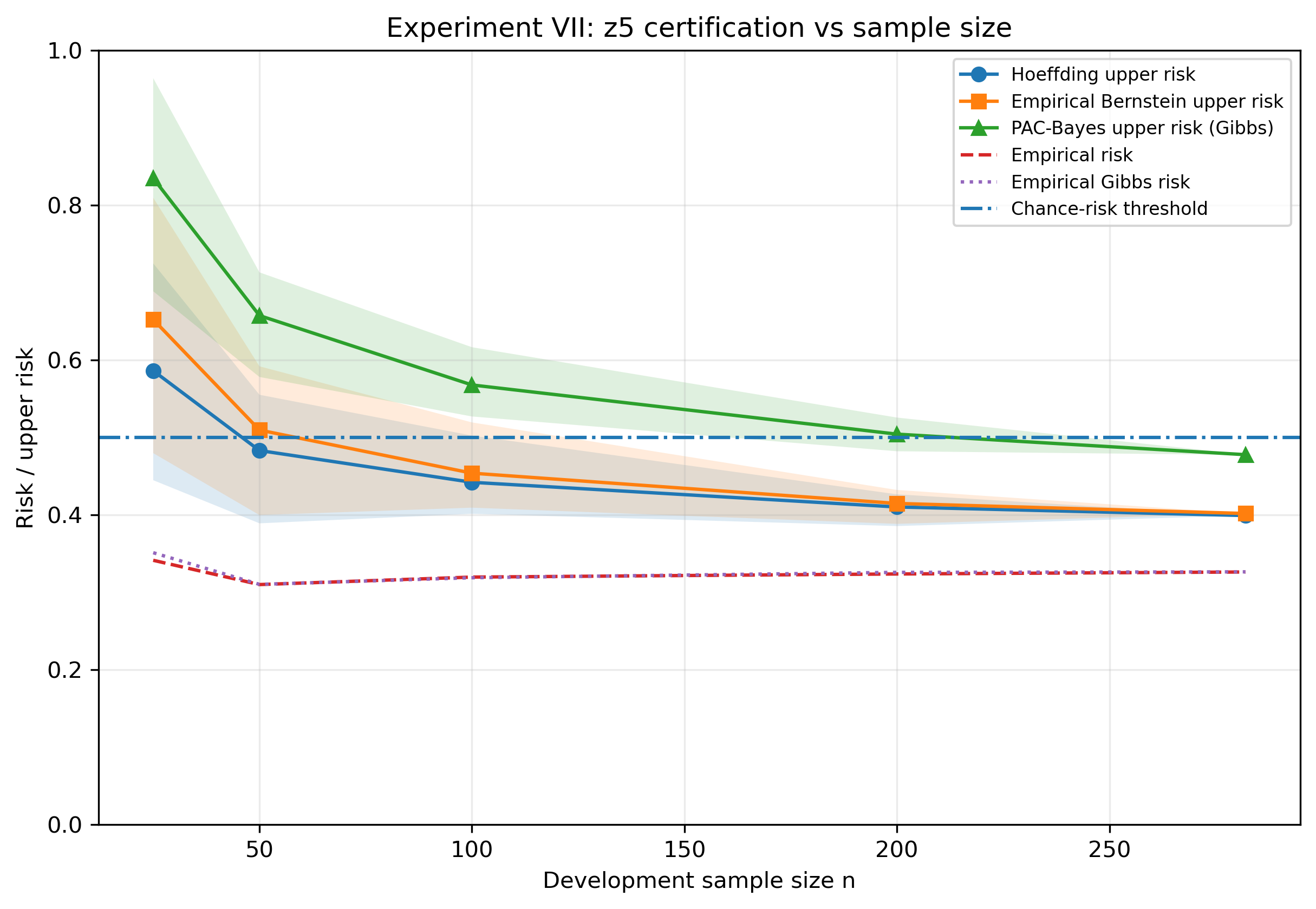}
\includegraphics[width=.48\linewidth]{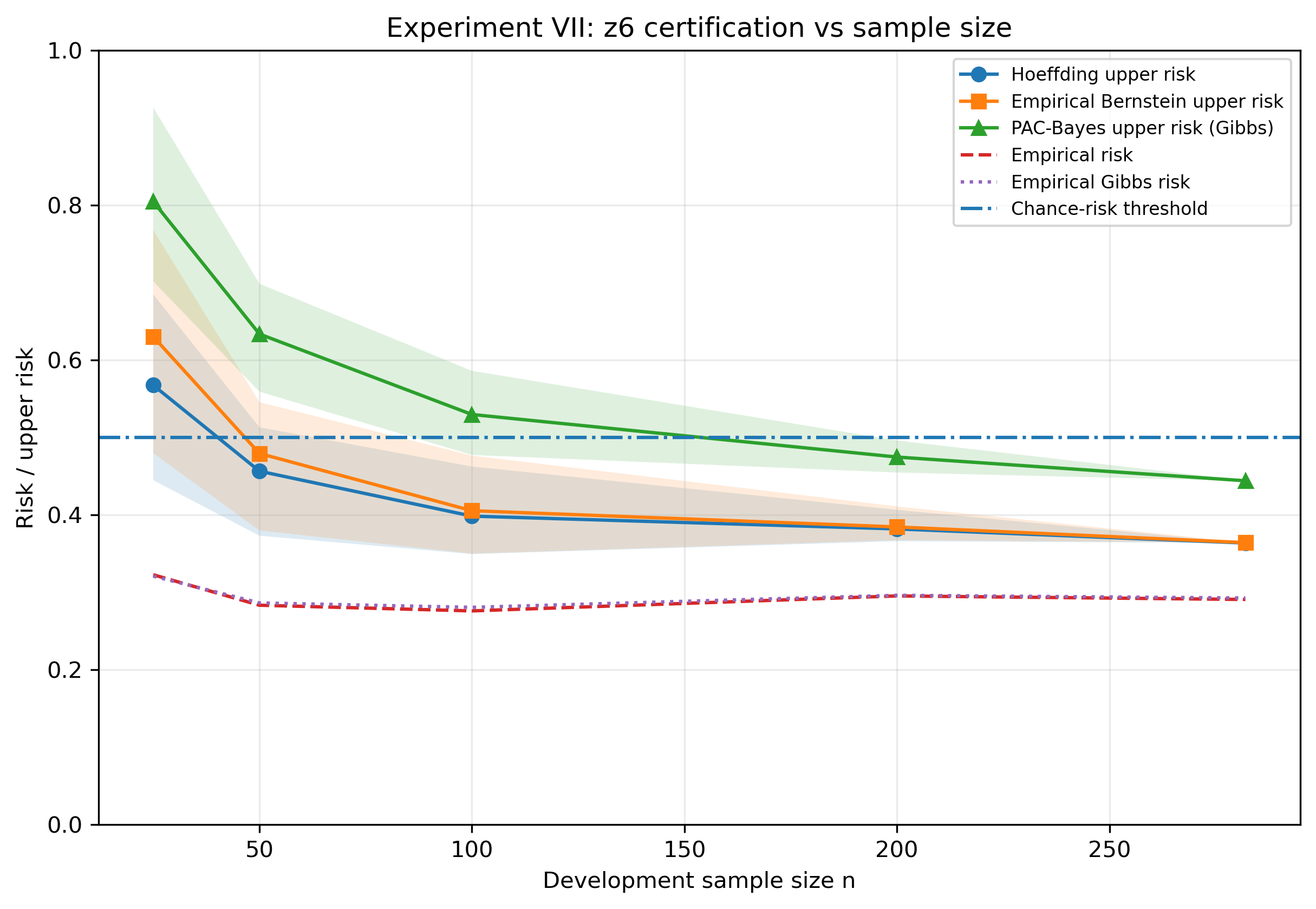}
\includegraphics[width=.48\linewidth]{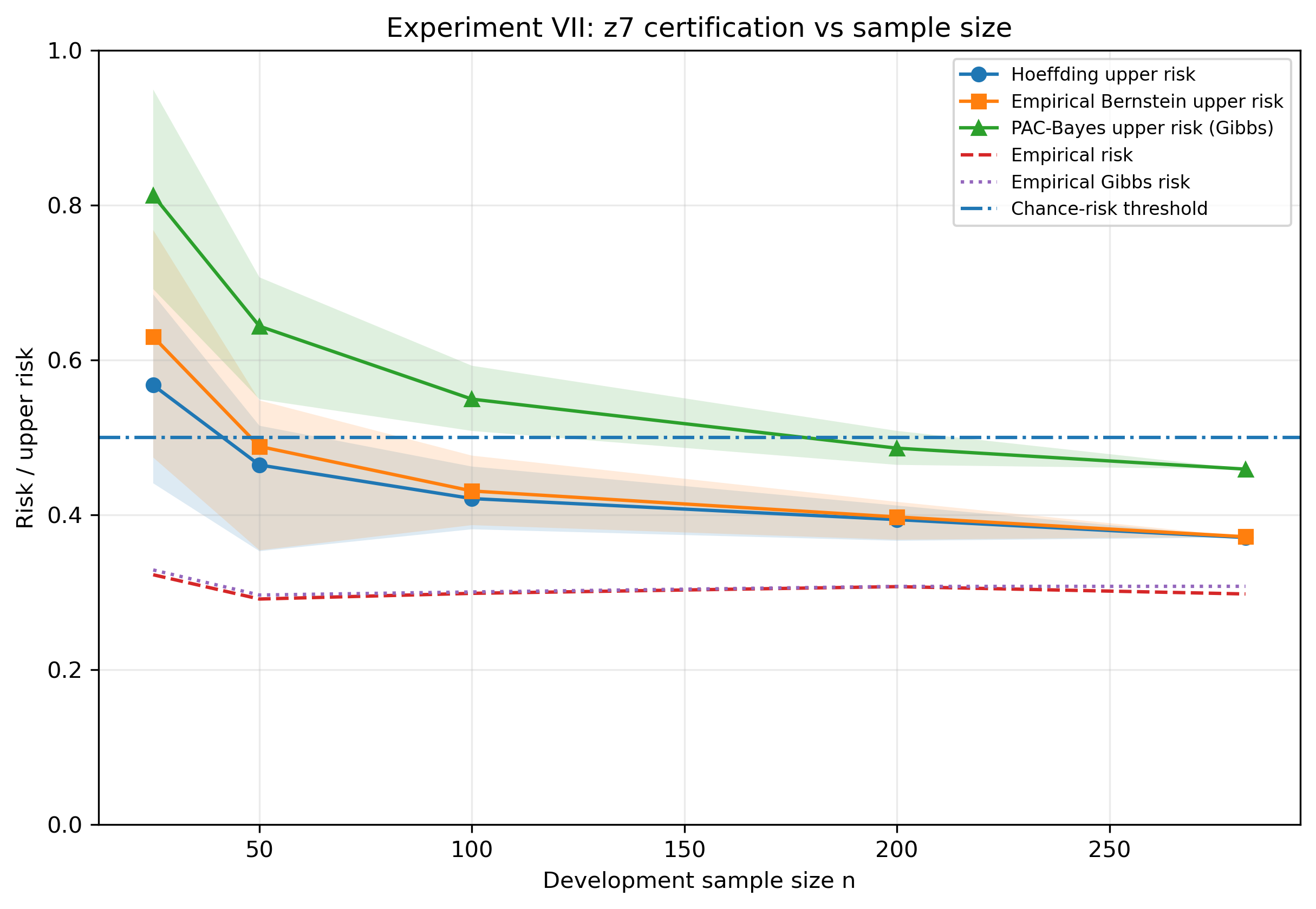}
\includegraphics[width=.48\linewidth]{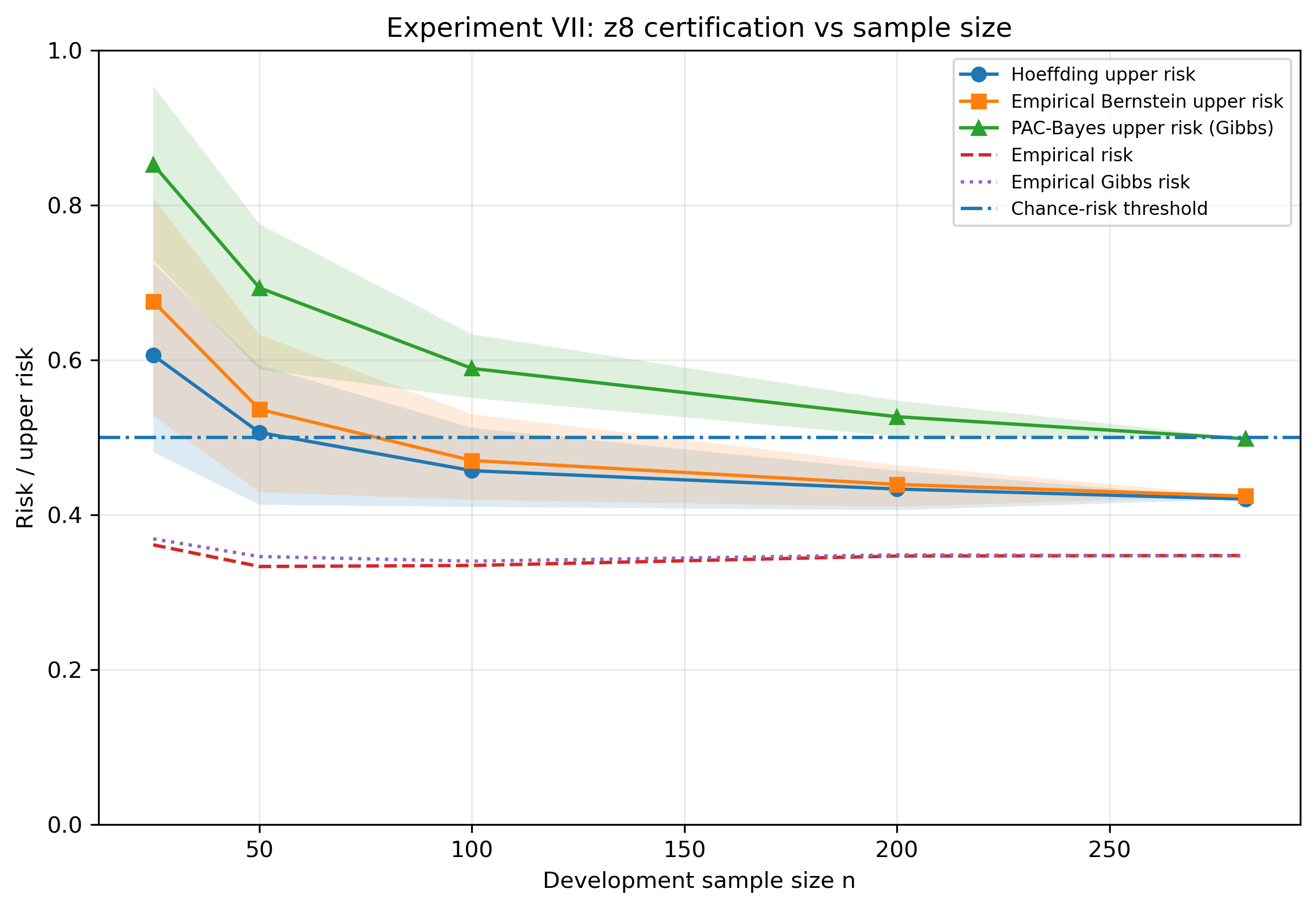}
\caption{Experiment VII: sample-size dependence of bounded discriminability
for latent coordinates $z_5$--$z_8$. The later PAC--Bayes crossing of some
coordinates illustrates that an empirical discriminative signal can appear
before the available sample supports a better-than-chance upper-risk
statement.}
\label{fig:real-z5-z8}
\end{figure*}

For comparison, the PAC--Bayes upper risk of the complete eight-dimensional representation remained above $0.5$ throughout the sample-size sweep, reaching approximately $U_{PB}(Z) \simeq 0.531$ at $n=282$. This behaviour contrasts with the individual-coordinate results and again illustrates the contribution of the complexity term.

\subsection{Permutation negative control}

A label-randomisation experiment was used as a negative control. The learned
latent representation was kept fixed while the diagnostic labels were
randomly permuted. Thus, representation geometry was preserved while the
association between latent coordinates and diagnosis was destroyed. For the true labels, coordinate-wise Hoeffding upper risks ranged from $
0.349$ to $0.420$, whereas the corresponding permutation-null means ranged approximately from $0.553$ to $0.557$. Similarly, real-label empirical-Bernstein upper risks ranged from $0.349$ to $0.424$, whereas the permutation-null means ranged from approximately $0.560$ to $0.564$. Hence the deterministic better-than-chance upper-risk pattern disappeared after label randomisation. PAC--Bayes was not included in this permutation experiment, and no PAC--Bayesian permutation-null conclusion is therefore made.

\subsection{Interpretation}

Experiment VII plays a different role from the controlled simulations.
Because the true population risk is unknown in real neuroimaging data, the
experiment cannot establish empirical coverage of the proposed bounds.
Instead, it evaluates whether bounded discriminability produces coherent
behaviour under three independent challenges: held-out subjects, decreasing
sample size, and destruction of the diagnostic association through label
permutation.

The component-wise results show that all eight coordinates of the learned
whole-brain representation contain detectable NOR--AD information at the
full development sample size. Nevertheless, the amount of evidence is not
identical across coordinates. Under PAC--Bayes, some coordinates become
non-vacuously bounded with fewer labelled subjects than others. The complete eight-dimensional representation provides a complementary observation. Although using all coordinates jointly preserves or improves empirical predictive performance relative to many individual coordinates, its larger PAC--Bayesian complexity penalty prevents the corresponding upper risk from falling below $0.5$. Thus, $\text{more features}\;\not\Rightarrow\;
\text{stronger evidence}$, and similarly, $\text{observed accuracy}
\;\not\equiv\; \text{bounded discriminability}$. The real-data analysis therefore provides a proof-of-concept application of the central principle of this work: empirical performance describes what was observed in a particular sample, whereas a generalisation bound quantifies how strongly the available evidence supports extrapolation beyond that sample.

\section{Synthetic-experiment implementation details}
\label{app:synthetic}

This appendix contains only experiment-specific details required for reproducibility. The statistical theory and the scientific interpretation of the results are given in the main text.

\subsection{Common simulation settings}

Unless otherwise stated, all synthetic experiments use balanced binary classification with $d=2$. The baseline multimodal generator draws each class from a Gaussian mixture,
\begin{equation}
p(z\mid Y=c) = \sum_{m=1}^{M_c} \pi_{cm}\,
\mathcal N(z;\mu_{cm},\Sigma_{cm}),
\end{equation}
with one dominant component per class and one or more minority components. The baseline parameters are
\begin{equation*}
\pi_{\mathrm{major}}=0.6,
\qquad
\mathrm{separation}=1.5,
\qquad
\mathrm{cluster\;SD}=1.2,
\end{equation*}
unless modified by a specific experiment. A logistic-regression classifier with $\ell_2$ regularisation is trained after feature standardisation. For resubstitution, the classifier is fitted on all $n$ observations and evaluated on the same sample. For $K$-fold validation, a new classifier is fitted independently inside every training fold. Unless otherwise stated, $K=5$ and $1-\delta=0.95$. 

For simulation only, population risk is approximated using an independent test sample of size $N_{\mathrm{test}}=10\,000$. Each experimental configuration uses 60 Monte Carlo repetitions. Whenever a bound requires calibration of latent generalisation gaps, the repetitions are split into disjoint calibration and evaluation blocks, with 50\% assigned to each whenever possible. Parameters estimated from the calibration block are frozen before evaluation. For continuous Monte Carlo summaries, 95\% intervals are computed as $\bar q\pm1.96\frac{s_q}{\sqrt{B_{\mathrm{eval}}}}$, where $B_{\mathrm{eval}}$ is the number of independent evaluation repetitions. Coverage and certification probabilities are binomial proportions and use Wilson score intervals. These Monte Carlo intervals quantify simulation uncertainty and are distinct from the generalisation confidence level $1-\delta$.

\subsection{Experiment-specific configurations}

The six experiments differ in the structural factor under study. For Experiment II, the explicit difficulty ladder is
\begin{align*}
L1:\;& \text{linear},\quad N_c=2,\\
    & \mathrm{sep}=2.4,\quad \sigma=0.70,\\[1mm]
L2:\;& \text{linear multimodal},\quad N_c=4,\\
    & \mathrm{sep}=1.8,\quad \sigma=0.90,\\[1mm]
L3:\;& \text{biased alternating ring},\quad N_c=4,\\
    & r=2.4,\quad \sigma=0.60,\quad b=0.70,\\[1mm]
L4:\;& \text{biased alternating ring},\quad N_c=6,\\
    & r=2.2,\quad \sigma=0.75,\quad b=0.35,\\[1mm]
L5:\;& \text{alternating ring},\quad N_c=10,\\
    & r=2.0,\quad \sigma=0.95,\quad b=0.
\end{align*}
This is an ordered difficulty ladder for the fixed linear learner rather than
a pure measure of multimodality.

For Experiment III, two families of balanced two-dimensional Gaussian-mixture problems were compared over $N_c\in\{2,4,6,8,10\}$. In the \emph{linear-multimodal} family, each class contains $N_c/2$ Gaussian components, with one dominant component of weight $0.6$ centred at $\pm1.5$ along the first coordinate and additional minority components placed closer to the decision boundary while preserving a global linear trend; all components use standard deviation $1.2$. In the \emph{nonlinear-alternating} family, the $N_c$ component centres are equally spaced on a circle of radius $2.2$ and assigned alternately to the two classes, again with Gaussian standard deviation $1.2$. Thus, increasing $N_c$ increases multimodality in both families, but only the alternating construction progressively destroys global linear separability. 

For Experiment IV, all four sources correspond to the same balanced binary classification task, but each source is generated from its own non-homogeneous two-component-per-class Gaussian mixture. Before rotation, the class centres form the interlaced pattern of Eq.~\eqref{eq:expIV-interlaced}. The source-specific parameter tuples
$(a,b,\sigma,\theta,p)$ are
\[
(1.50,1.30,1.05,0^\circ,0.65),\qquad
(1.45,1.35,1.08,15^\circ,0.60),
\]
\[
(1.55,1.25,1.10,-15^\circ,0.68),\qquad
(1.40,1.40,1.08,30^\circ,0.58),
\]
where $p$ denotes the probability of sampling the first within-class
mode. Subject labels are common across sources, but the latent
within-class mode assignment is sampled independently for each source.

For Experiment V, let $Z\in\mathbb{R}^d$ denote the baseline random feature vector generated by the underlying binary classification problem, shared across all sources.
Source 1 preserves this representation exactly, $X_1=Z$, whereas the remaining sources are generated as 
\begin{equation}
X_m=s_m Z+\varepsilon_m,
\end{equation}
where $s_m$ is the source-specific signal attenuation factor and $\varepsilon_m$ denotes additive noise. The signal strength decreases and the noise standard
deviation increases across a single master source ladder. The largest configuration uses strengths decreasing from $1.0$ to $0.4$ and noise standard deviations increasing
from $0$ to $1.0$. Smaller-$M$ configurations are prefixes of the same ladder, so increasing $M$ genuinely adds sources without redefining those already present.

For Experiment VI, source geometry changes rather than merely adding noise.
Source 1 is linear, Source 2 is globally linearly exploitable but multimodal,
and Sources $j\ge3$ use alternating modes on a circle with $N_c=2j$. Thus,
increasing $M$ progressively increases source heterogeneity and nonlinearity
while retaining the same subject labels across sources.

\begin{table*}[t]
\centering
\small
\begin{tabular}{p{0.055\textwidth}p{0.19\textwidth}p{0.64\textwidth}}
\toprule
Exp. & Main factor & Configuration \\
\midrule
I & Sample size &
$n\in\{50,100,200,500,1000,2000,5000\}$, $d=2$, $N_c=4$;
resubstitution and $K$-fold bounds compared. \\
II & Structural complexity &
Five-level ladder from a simple linear problem to a highly nonlinear
overlapping alternating-ring problem. \\
III & Multimodality vs.\ nonlinearity &
$N_c\in\{2,4,6,8,10\}$; globally linear multimodal geometry compared
with alternating nonlinear geometry. \\
IV & Fusion rule &
$M=4$ non-homogeneously sampled multimodal sources with two interlaced
Gaussian modes per class; Uniform, Empirical-risk, Variance-only,
Fusion-CUBV, and Safe-Fusion-CUBV compared across the Experiment-I
sample-size grid. \\
V & Number of sources &
Fusion-CUBV fixed; $M\in\{2,4,6,8,10\}$ using a nested ladder of
progressively weaker/noisier views of the same subjects. \\
VI & Structural source heterogeneity &
Fusion-CUBV fixed; $M\in\{2,4,6,8,10\}$ with additional sources having
progressively different and more nonlinear class geometries. \\
\bottomrule
\end{tabular}
\caption{Summary of the synthetic experiments.}
\label{tab:appendix-experiments}
\end{table*}

\subsection{Representative synthetic datasets}
\label{app:synthetic-examples}

To make the structural manipulations directly interpretable,
Figs.~\ref{fig:app-complexity-examples}--\ref{fig:app-heterogeneous-source-examples}
show representative realisations generated by the same routines and parameterisations used in Experiments II and IV--VI. These panels are illustrative draws rather than additional experimental results. Figure~\ref{fig:app-complexity-examples} visualises the five-level structural ladder of Experiment II. The progression changes both multimodality and compatibility with the fixed linear classifier: the first two levels preserve a globally exploitable linear direction, whereas the later levels introduce
increasingly nonlinear alternating geometry and overlap.

\begin{figure*}[t]
\centering
\includegraphics[width=0.98\textwidth]
{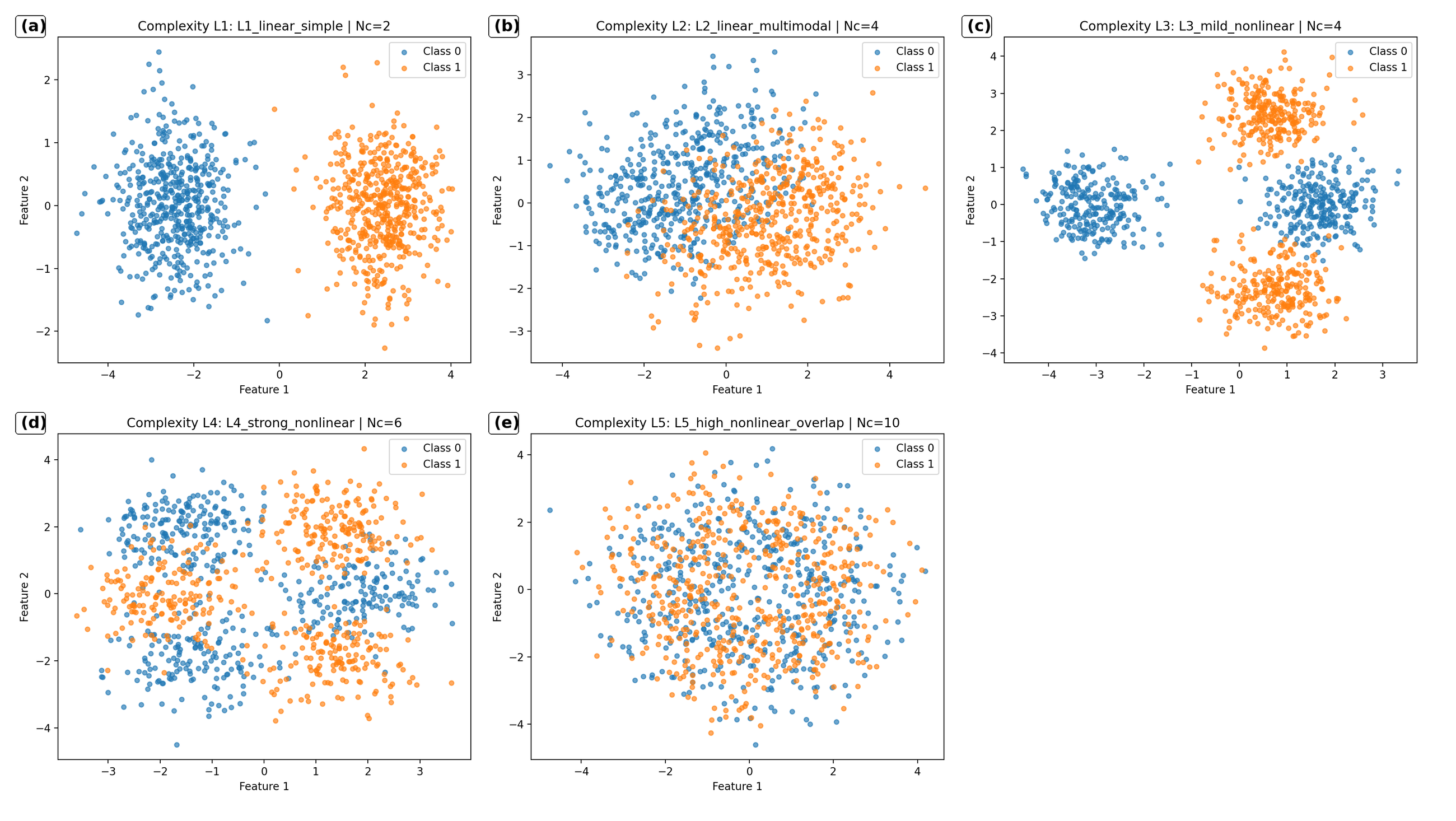}
\caption{Representative datasets from Experiment II:
(a) $L1$, simple linear;
(b) $L2$, linear multimodal;
(c) $L3$, mildly nonlinear;
(d) $L4$, strongly nonlinear; and
(e) $L5$, highly nonlinear with substantial class overlap.
The panels are generated by the same data-generating routines used in the
Monte Carlo study.}
\label{fig:app-complexity-examples}
\end{figure*}

Figure~\ref{fig:app-expIV-sources} visualises the four heterogeneous multimodal sources used in Experiment IV. Each source preserves the same binary labels, while the latent within-class component is resampled according to its source-specific non-homogeneous mixture proportion.

\begin{figure*}[t]
\centering
\includegraphics[width=0.98\textwidth]
{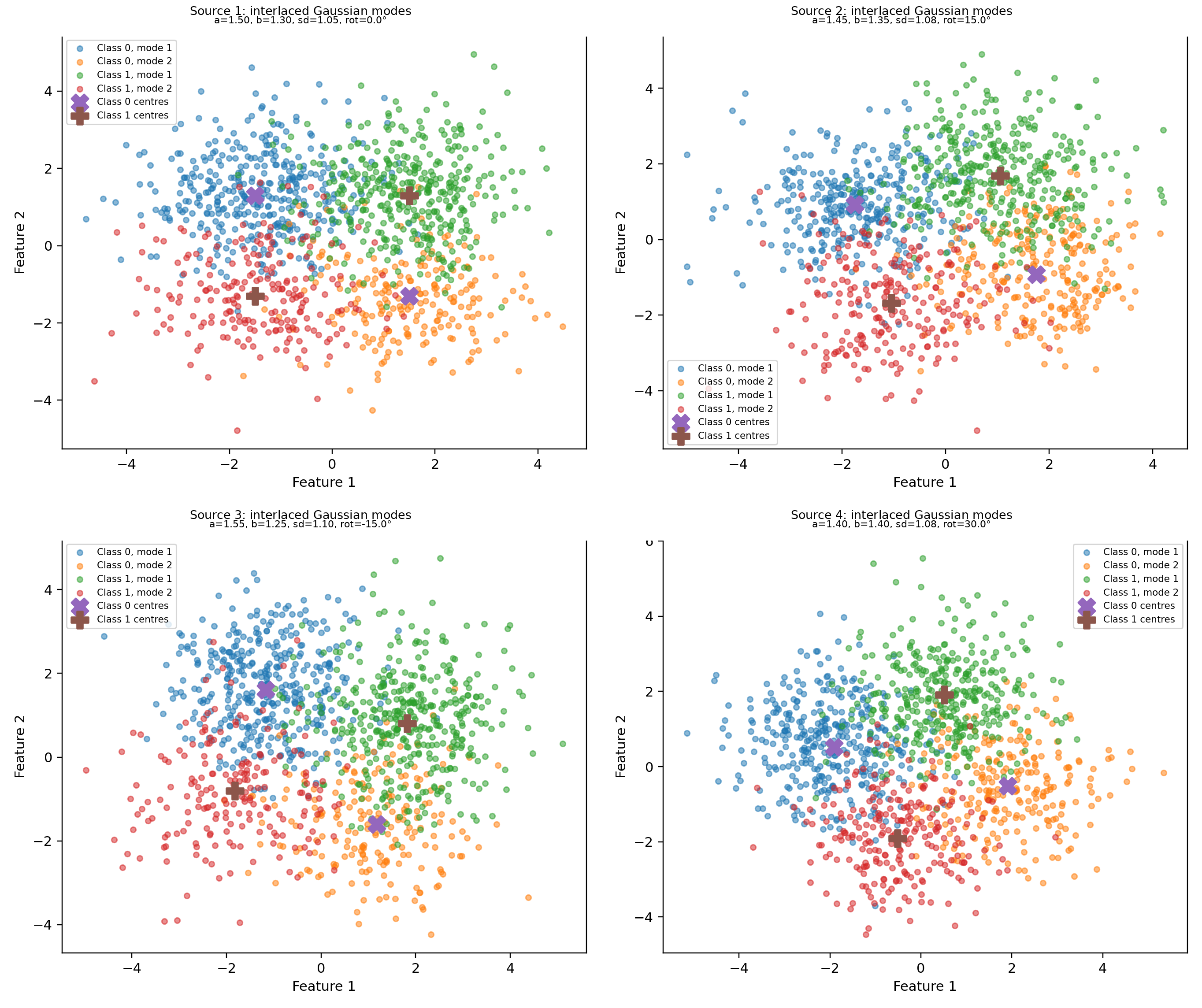}
\caption{Representative heterogeneous multimodal sources from
Experiment IV. Each source contains two interlaced Gaussian modes per
class. Subject labels are shared across sources, whereas the latent
within-class component is sampled independently using source-specific
mixture probabilities; separation, dispersion, and orientation also
vary across sources. Cross-shaped markers indicate the generating
component centres.}
\label{fig:app-expIV-sources}
\end{figure*}

Figure~\ref{fig:app-noisy-source-examples} illustrates the nested source construction of Experiment V. All panels contain the same underlying subjects and labels. Source 1 is the exact baseline representation, while later sources progressively attenuate the common signal and add noise.

\begin{figure*}[t]
\centering
\includegraphics[width=0.98\textwidth]
{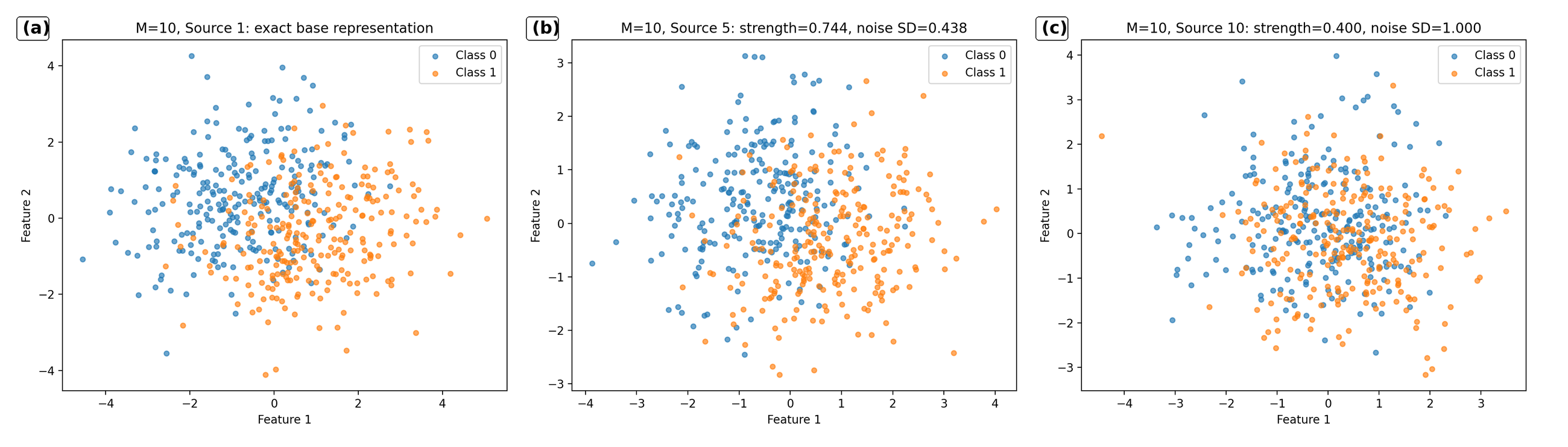}
\caption{Representative views from the $M=10$ nested source ladder of
Experiment V: (a) Source 1, the exact baseline representation;
(b) Source 5, an intermediate weakened/noisy representation; and
(c) Source 10, the final and weakest/noisiest representation.
The subject identities and class labels are shared across all three panels.}
\label{fig:app-noisy-source-examples}
\end{figure*}

Finally, Fig.~\ref{fig:app-heterogeneous-source-examples} visualises the structurally heterogeneous construction of Experiment VI. Unlike Experiment V, the source index changes the class geometry itself rather than merely adding noise to a common representation.

\begin{figure*}[t]
\centering
\includegraphics[width=0.90\textwidth]
{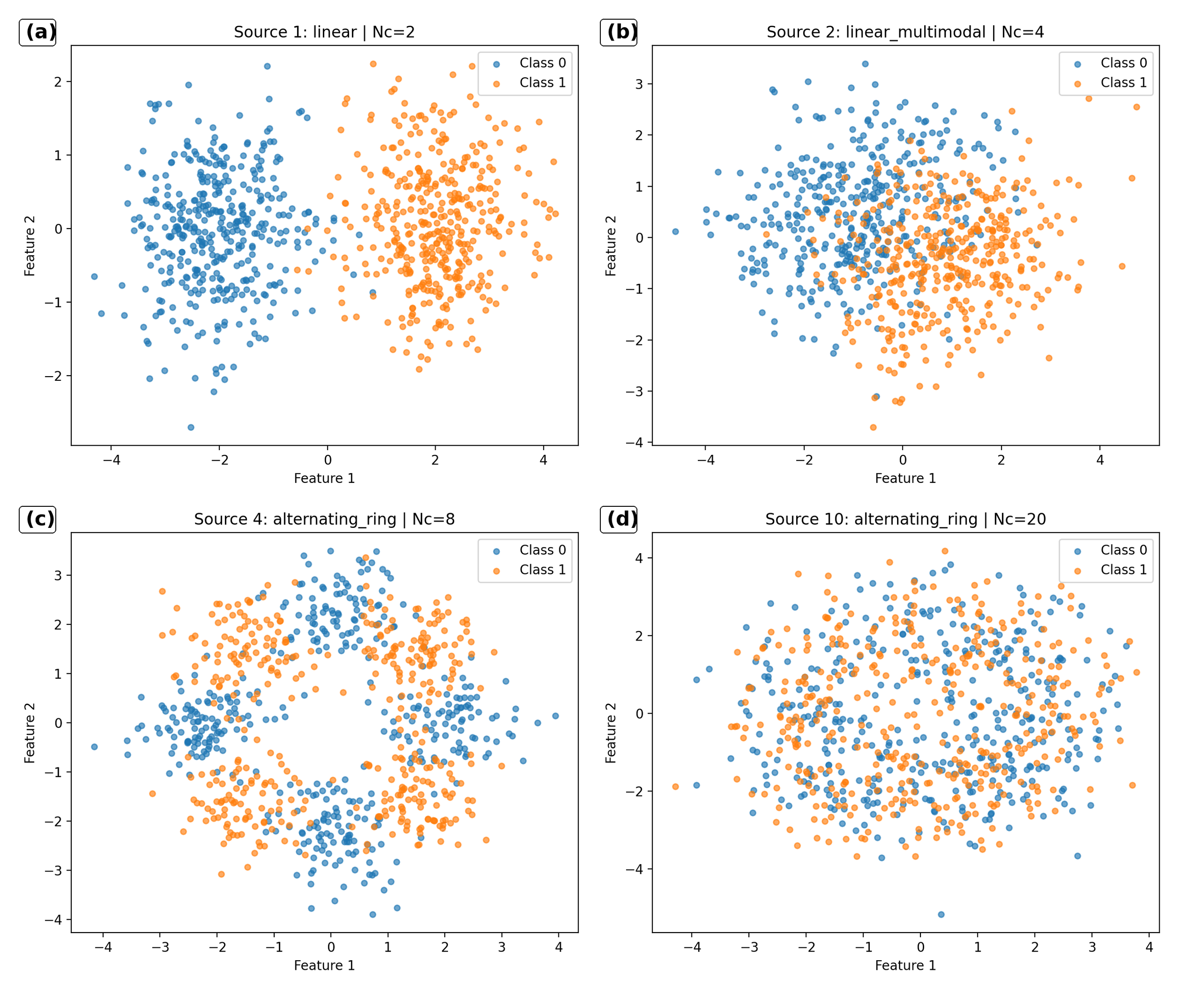}
\caption{Representative structurally heterogeneous sources from Experiment VI:
(a) Source 1, linear with $N_c=2$;
(b) Source 2, globally linearly exploitable but multimodal with $N_c=4$;
(c) Source 4, alternating nonlinear geometry with $N_c=8$; and
(d) Source 10, the most structurally complex alternating geometry with
$N_c=20$.}
\label{fig:app-heterogeneous-source-examples}
\end{figure*}

\subsection{Additional K-fold comparisons}
\label{app:kfold-comparisons}

Because the main experimental discussion focuses on the broader generalisation-bound behaviour, we additionally report here a direct comparison between resubstitution and $K$-fold validation for the three bound families that share a common deterministic-classifier target: Hoeffding, empirical Bernstein, and MGF/Dependency. Figure~\ref{fig:app-kfold-tightness} reports the mean tightness $U-\R_{\mathrm{true}}$
as a function of sample size. At small sample sizes, the $K$-fold curves are consistently wider than their resubstitution counterparts. For example, at
$n=50$, mean tightness is approximately $0.165$ versus $0.152$ for Hoeffding,
$0.175$ versus $0.160$ for empirical Bernstein, and $0.170$ versus $0.154$ for
MGF/Dependency, with the first value in each pair corresponding to $K$-fold.
The difference contracts rapidly with increasing $n$. By $n=5000$, the two
validation paradigms are nearly indistinguishable for all three constructions. This pattern is consistent with the interpretation developed in the main text: partitioning the available observations can introduce an additional finite- sample uncertainty penalty because each fitted $K$-fold model is trained on only a fraction of the data. As the sample becomes sufficiently large, this partition-induced penalty becomes negligible relative to the total generalisation uncertainty.

\begin{figure*}[t]
\centering
\includegraphics[width=0.90\textwidth]
{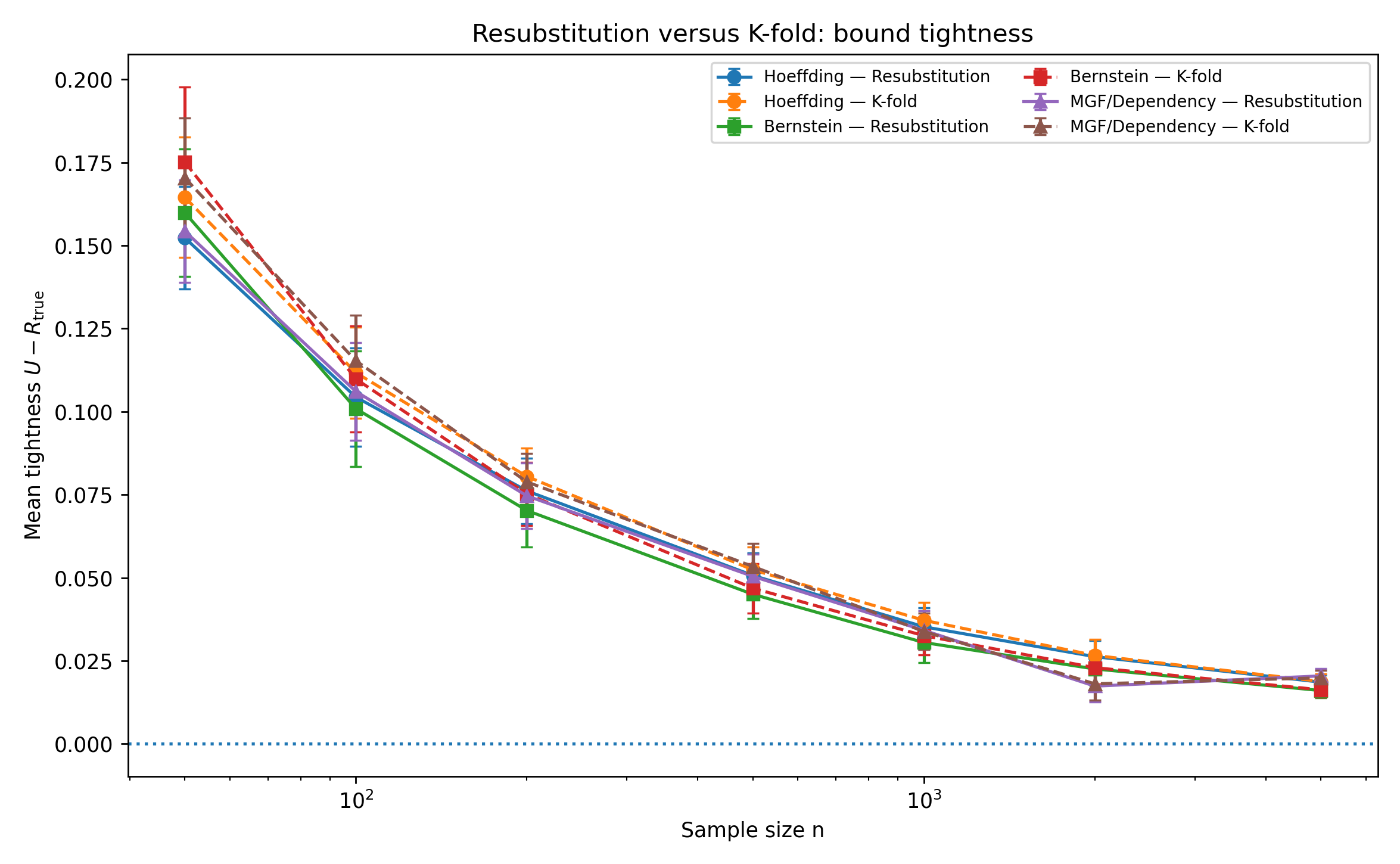}
\caption{Mean bound tightness, $U-\R_{\mathrm{true}}$, for resubstitution and
$K$-fold validation across sample size. Error bars are 95\% Monte Carlo
confidence intervals. At small $n$, the $K$-fold bounds are systematically
wider for the three displayed constructions; the difference decreases as
sample size grows.}
\label{fig:app-kfold-tightness}
\end{figure*}

Figure~\ref{fig:app-kfold-coverage} complements the tightness comparison by
showing empirical $K$-fold coverage for the three dependence-oriented
constructions. Safe-MGF/Dependency and H\"older-Hoeffding achieve coverage
equal to one at every tested sample size. The plug-in MGF/Dependency rule also
achieves coverage equal to one in all but one configuration, where it attains
$27/30=0.90$ at $n=2000$. Because each reported coverage value is estimated from only
$B_{\mathrm{eval}}=30$ independent evaluation repetitions, the corresponding
Wilson intervals are necessarily broad. In particular, the interval associated with $27/30$ coverage extends approximately from $0.744$ to $0.965$. This isolated deviation should therefore not be interpreted as precise evidence of systematic undercoverage. Rather, it illustrates the distinction between the empirically calibrated MGF construction and the more conservative Safe-MGF and H\"older alternatives.

\begin{figure*}[t]
\centering
\includegraphics[width=0.90\textwidth]
{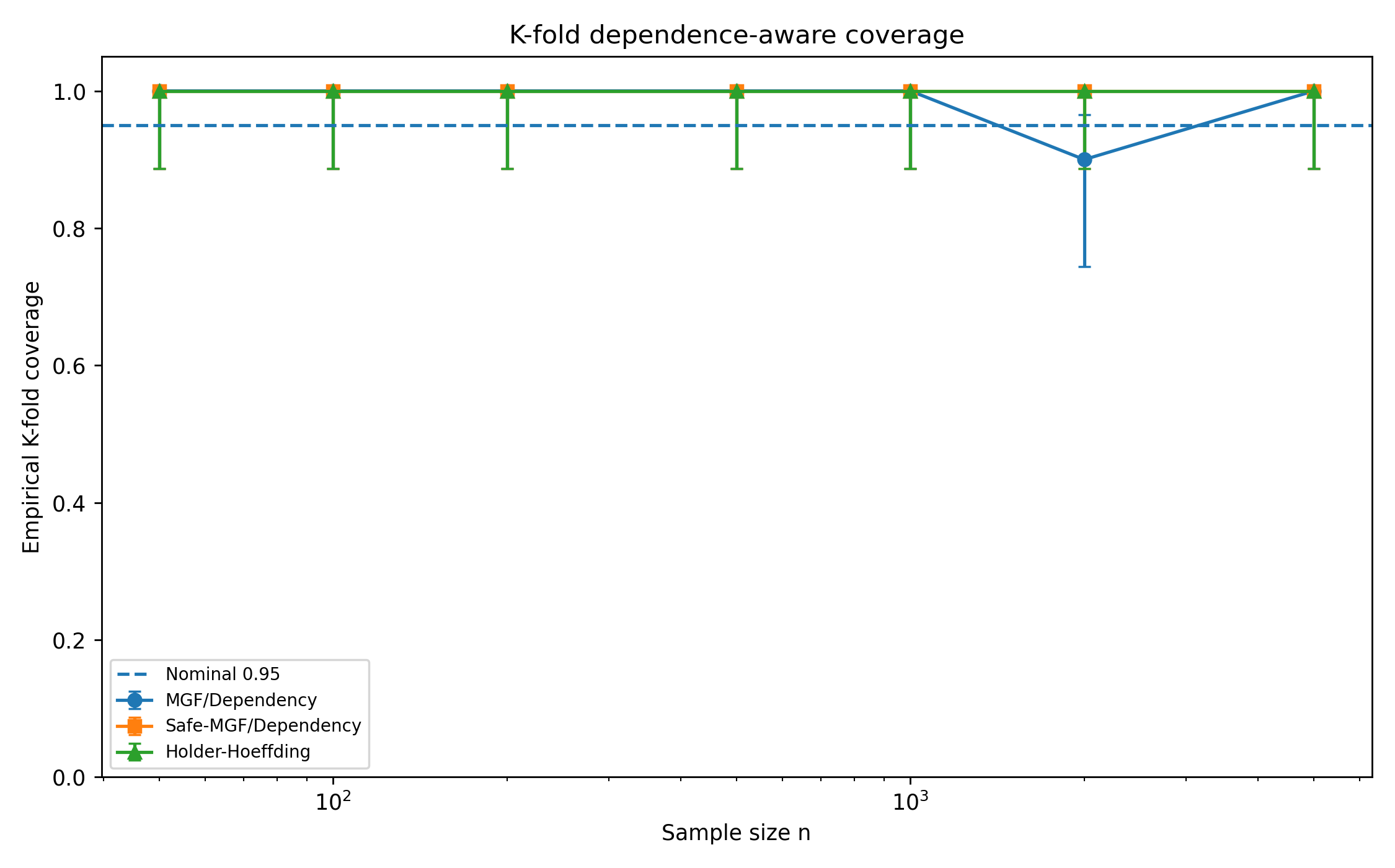}
\caption{Empirical $K$-fold coverage for the dependence-aware constructions.
Error bars are 95\% Wilson score intervals. The dashed horizontal line marks
the nominal coverage level $1-\delta=0.95$. Safe-MGF/Dependency and
H\"older-Hoeffding are conservative throughout the sweep, whereas the
plug-in MGF/Dependency rule exhibits one isolated value of $0.90$ at
$n=2000$.}
\label{fig:app-kfold-coverage}
\end{figure*}

\section{Experimental instantiation of classical bounds}
\label{app:classical-instantiation}

The derivations of the concentration bounds are given in the main text. Here we record only the concrete forms used in the simulations and the qualifications needed to interpret them.

\subsection{Fixed-hypothesis Hoeffding rule}
\label{app:hoeffding-full}

For a fixed classifier evaluated on $n$ independent observations under
$0$--$1$ loss, the implemented rule is
\begin{equation}
U_H
=
\Rh
+
\sqrt{\frac{\log(1/\delta)}{2n}}.
\end{equation}
This fixed-hypothesis result is not automatically a post-training
resubstitution guarantee when the same sample is used both to fit and evaluate
the classifier. We therefore use it as a comparative diagnostic rule and
assess its empirical coverage explicitly.

\subsection{Uniform VC-Hoeffding reference}
\label{app:vc-hoeffding}

For affine classifiers in $\mathbb R^p$, $d_{\mathrm{VC}}=p+1$. The
implemented uniform reference is
\begin{equation}
U_{\mathrm{VC-H}}
=
\Rh_S(\widehat h)
+
\sqrt{
\frac8n
\left[
d_{\mathrm{VC}}
\log\!\left(\frac{2en}{d_{\mathrm{VC}}}\right)
+
\log\frac4\delta
\right]
}.
\end{equation}
The reported risk is truncated at one. This construction explicitly pays for
uniform control of the linear hypothesis class and is therefore intentionally
conservative in small samples.

\subsection{Empirical Bernstein rule}
\label{app:bernstein-full}

The implemented empirical-Bernstein rule is
\begin{equation}
U_B
=
\Rh
+
\sqrt{\frac{2\widehat V\log(1/\delta)}{n}}
+
\frac{2\log(1/\delta)}{3n},
\end{equation}
where $\widehat V$ is the empirical variance of the observed $0$--$1$ losses.
As for fixed-hypothesis Hoeffding, direct post-training resubstitution use is
treated as a comparative experimental rule unless an appropriate uniform or
algorithm-dependent result is invoked.


\section{Simulation calibration of Dependency-CUBV}
\label{app:dependency-full}

The theoretical Dependency-CUBV construction is developed in the main text.
Here we specify only how its latent quantities are calibrated in the synthetic
experiments. For every Monte Carlo repetition and fold, the synthetic setting permits
approximation of the fold population risk on an independent test population.
We therefore form
\begin{equation}
G_k=\R_k-\Rh_k.
\end{equation}
Calibration repetitions are stacked into a matrix whose columns correspond to
folds. The implementation estimates $\widehat{\bm\mu}\qquad\text{and}\qquad
\widehat{\bm\Sigma}$ from this calibration block and freezes both quantities before evaluation on disjoint Monte Carlo repetitions. This is an oracle-calibrated simulation device rather than a turnkey single-dataset procedure. In real data, an analogue would require an untouched audit sample, a justified analytical dependence model, algorithmic-stability information, or another procedure that propagates uncertainty in the
dependence estimate.

\section{PSD-inflated Dependency-CUBV: experimental instantiation}
\label{app:safe-dependency-full}

The theoretical PSD upper-confidence construction is given in the main text.
In the simulations it is instantiated as
\begin{equation}
\bm U_\eta
=
\widehat{\bm\Sigma}
+
\varepsilon_\eta\bm I,
\end{equation}
with
\begin{equation}
\varepsilon_\eta
=
K
\sqrt{
\frac{2\log(2K^2/\eta)}{B}
},
\end{equation}
where $B$ is the number of calibration-gap vectors. The failure budget is
split as $\delta_{\mathrm{risk}}=\eta=\delta/2$. This implementation is deliberately conservative. The inflation protects all covariance directions even though the final risk bound depends only on the quadratic form $\bm w^\top\bm\Sigma\bm w$. Together with the entrywise union bound and the conversion to an operator-norm radius, this produces substantial over-protection; in the present simulations the resulting upper risk is often truncated at one. The radius above is derived for independently observed centred calibration-gap vectors, whereas the gap mean is also estimated from the calibration block in the present implementation. The resulting safe curve should therefore be
interpreted as a conservative proof-of-principle rather than as a complete finite-sample theorem for simultaneous estimation of $\bm\mu$ and $\bm\Sigma$.

\section{Experimental specialisation of H\"older-CUBV}
\label{app:holder-full}

The general H\"older-CUBV construction is derived in the main text. For the
experiments, each validation fold contains $m_k$ observations and the
fold-wise $0$--$1$ loss uses the Hoeffding envelope
\begin{equation}
\gamma_k(t)=\frac{t^2}{8m_k}.
\end{equation}
With equal fold weights $w_k=1/K$ and the symmetric choice $p_k=K$, the
implemented bound is
\begin{equation}
U_{\mathrm{Holder}}
=
\Rh_{\mathrm{CV}}
+
\sqrt{
\frac{\log(1/\delta)}{2}
\left(
\frac1K\sum_{k=1}^{K}\frac1{m_k}
\right)
}.
\end{equation}
For equal validation-fold sizes $m=n/K$,
\begin{equation}
U_{\mathrm{Holder}}
=
\Rh_{\mathrm{CV}}
+
\sqrt{\frac{\log(1/\delta)}{2m}}.
\end{equation}
Thus the implemented rule does not obtain an artificial $K^{-1/2}$ gain by
treating dependent folds as independent. Its cost is conservatism; its
advantage is that no estimate of $\bm\Sigma$ is required.

\section{Finite-sample derivation of the PAC--Bayesian square-root bound}
\label{app:pac-bayes-derivation}

This appendix derives the explicit PAC--Bayesian upper-risk expression
used in Eq.~\eqref{eq:pac-applied}. The derivation starts from the
finite-sample exponential inequality of Maurer~\cite{maurer2004pac} and
then applies the PAC--Bayesian change-of-measure argument followed by
Pinsker's inequality. Let $S=(Z_1,\ldots,Z_n)$ be an i.i.d. sample and let the loss satisfy $\ell(h,Z)\in[0,1]$. For a predictor $h$, define the population and
empirical risks as
\begin{equation}
\R(h)
=
\E_Z[\ell(h,Z)],
\qquad
\Rh_S(h)
=
\frac{1}{n}
\sum_{i=1}^{n}
\ell(h,Z_i).
\end{equation}
Let $P$ be a prior distribution over predictors independent of $S$,
and let $Q\ll P$ denote an arbitrary data-dependent posterior. The
corresponding Gibbs risks are
\begin{equation}
\R(Q)
=
\E_{h\sim Q}\R(h),
\qquad
\Rh(Q)
=
\E_{h\sim Q}\Rh_S(h).
\end{equation}

For $q,p\in[0,1]$, denote the binary relative entropy by
\begin{equation}
\kl(q,p)
=
q\log\frac{q}{p}
+
(1-q)\log\frac{1-q}{1-p}.
\end{equation}

Maurer~\cite{maurer2004pac} establishes, for bounded losses and
$n\geq 8$, the finite-sample exponential control
\begin{equation}
\E_S
\exp\left\{
n\,
\kl\!\left(
\Rh_S(h),
\R(h)
\right)
\right\}
\leq
2\sqrt{n}.
\label{eq:app-maurer-mgf}
\end{equation}
Since the prior $P$ is independent of the observed sample, averaging
\eqref{eq:app-maurer-mgf} with respect to $h\sim P$ gives
\begin{equation}
\E_S\E_{h\sim P}
\exp\left\{
n\,
\kl\!\left(
\Rh_S(h),
\R(h)
\right)
\right\}
\leq
2\sqrt{n}.
\label{eq:app-maurer-prior}
\end{equation}

Applying Markov's inequality to
\eqref{eq:app-maurer-prior}, with probability at least $1-\delta$
over the draw of $S$,
\begin{equation}
\E_{h\sim P}
\exp\left\{
n\,
\kl\!\left(
\Rh_S(h),
\R(h)
\right)
\right\}
\leq
\frac{2\sqrt n}{\delta}.
\label{eq:app-markov-pb}
\end{equation}
Equivalently,
\begin{equation}
\log
\E_{h\sim P}
\exp\left\{
n\,
\kl\!\left(
\Rh_S(h),
\R(h)
\right)
\right\}
\leq
\log\frac{2\sqrt n}{\delta}.
\label{eq:app-log-mgf-pb}
\end{equation}

For the realised sample $S$, the Donsker--Varadhan
change-of-measure inequality~\cite{donsker1975} states that, for any
measurable function $f$,
\begin{equation}
\E_{h\sim Q}f(h)
\leq
\KL(Q\|P)
+
\log\E_{h\sim P}e^{f(h)}.
\label{eq:app-dv-pb}
\end{equation}
Taking
\begin{equation}
f(h)
=
n\,
\kl\!\left(
\Rh_S(h),
\R(h)
\right)
\end{equation}
and combining \eqref{eq:app-dv-pb} with
\eqref{eq:app-log-mgf-pb} yields, simultaneously for all
$Q\ll P$,
\begin{equation}
n\,
\E_{h\sim Q}
\kl\!\left(
\Rh_S(h),
\R(h)
\right)
\leq
\KL(Q\|P)
+
\log\frac{2\sqrt n}{\delta}.
\label{eq:app-pb-before-jensen}
\end{equation}

The binary relative entropy is jointly convex. Hence, by Jensen's
inequality,
\begin{align}
\kl\!\left(
\Rh(Q),
\R(Q)
\right)
&=
\kl\!\left(
\E_{h\sim Q}\Rh_S(h),
\E_{h\sim Q}\R(h)
\right)
\nonumber\\
&\leq
\E_{h\sim Q}
\kl\!\left(
\Rh_S(h),
\R(h)
\right).
\end{align}
Therefore, from \eqref{eq:app-pb-before-jensen},
\begin{equation}
\kl\!\left(
\Rh(Q),
\R(Q)
\right)
\leq
\frac{
\KL(Q\|P)
+
\log(2\sqrt n/\delta)
}{
n
}.
\label{eq:app-pb-kl}
\end{equation}

Equation~\eqref{eq:app-pb-kl} is the PAC--Bayes--kl form underlying
the finite-sample construction used in the experiments. To obtain an
explicit upper bound on $\R(Q)$, we use Pinsker's inequality for
Bernoulli distributions,
\begin{equation}
\kl(q,p)
\geq
2(q-p)^2.
\label{eq:app-pinsker}
\end{equation}
Setting $q=\Rh(Q)$ and $p=\R(Q)$ in
\eqref{eq:app-pinsker} and combining it with
\eqref{eq:app-pb-kl} gives
\begin{equation}
2
\left[
\R(Q)-\Rh(Q)
\right]^2
\leq
\frac{
\KL(Q\|P)
+
\log(2\sqrt n/\delta)
}{
n
}.
\end{equation}
Consequently,
\begin{equation}
\R(Q)
\leq
\Rh(Q)
+
\sqrt{
\frac{
\KL(Q\|P)
+
\log(2\sqrt n/\delta)
}{
2n
}
}.
\label{eq:app-pb-square-root}
\end{equation}

Thus,
\begin{equation}
U_{PB}
=
\Rh(Q)
+
\sqrt{
\frac{
\KL(Q\|P)
+
\log(2\sqrt n/\delta)
}{
2n
}
},
\end{equation}
which is precisely Eq.~\eqref{eq:pac-applied} used in the numerical
experiments. The finite-sample factor $2\sqrt n$ therefore originates from the
exponential inequality of Maurer~\cite{maurer2004pac}, whereas the
explicit square-root form follows from the subsequent application of
the change-of-measure argument, joint convexity of the binary
relative entropy, and Pinsker's inequality.

\section{Information-fusion implementation}
\label{app:fusion-implementation}

The theoretical Fusion-CUBV criterion is derived in the main text. Here we
describe the calibration and weighting rules used in Experiments IV--VI.

Let
\begin{equation}
\widehat{\bm R}
=
(\widehat R_1,\ldots,\widehat R_M)^\top,
\qquad
\bm a\in\Delta_M,
\end{equation}
where
\begin{equation}
\Delta_M
=
\{\bm a:\;a_m\ge0,\;\sum_m a_m=1\}.
\end{equation}
Calibration repetitions provide estimates of the mean generalisation-gap
vector $\bm\mu$ and covariance proxy $\bm\Sigma$. For fixed weights, the
calibrated fused bound used in the experiments is
\begin{equation}
U_F(\bm a)
=
\bm a^\top\widehat{\bm R}
+
\bm a^\top\bm\mu
+
\sqrt{
2\log(1/\delta)\,
\bm a^\top\bm\Sigma\bm a
}.
\end{equation}
The zero-mean form in the main text is recovered when $\bm\mu=0$. Fusion-CUBV chooses its weights exclusively on the calibration block:
\begin{equation}
\bm a^\star
=
\arg\min_{\bm a\in\Delta_M}
\left[
\bm a^\top\overline{\bm R}_{\mathrm{cal}}
+
\bm a^\top\widehat{\bm\mu}
+
\sqrt{
2\log(1/\delta)\,
\bm a^\top\widehat{\bm\Sigma}\bm a
}
\right].
\end{equation}
The resulting weights are frozen before evaluation, preventing post-hoc
optimisation on the reported repetitions. The comparison rules used in Experiment IV are:
\begin{itemize}
\item \emph{Uniform}: $a_m=1/M$;
\item \emph{Empirical-risk}: all mass is assigned to the source with the
smallest mean calibration risk, with ties shared equally;
\item \emph{Variance-only}: $\bm a$ minimises
$\bm a^\top\widehat{\bm\Sigma}\bm a$ over the simplex;
\item \emph{Fusion-CUBV}: $\bm a$ minimises the calibrated fused upper-risk
criterion above;
\item \emph{Safe-Fusion-CUBV}: the same optimisation is performed after
replacing $\widehat{\bm\Sigma}$ by
$\widehat{\bm\Sigma}+\varepsilon_\eta I$ and splitting the failure budget
between risk and covariance uncertainty.
\end{itemize}

For Experiments IV--VI, all sources in a repetition share the same
subject labels, but the source-generation mechanisms differ by experiment.
Experiment IV uses independently sampled heterogeneous multimodal Gaussian
mixtures with source-specific mixture proportions, dispersion, separation,
and orientation. Experiment V uses nested noisy and attenuated views of a
common baseline representation. Experiment VI changes the class geometry
itself, progressing from linear to multimodal and increasingly nonlinear
alternating-ring sources. Fusion therefore operates on multiple
representations of the same underlying classification task rather than on
unrelated datasets. For every evaluation repetition we record empirical fused risk, true fused risk, upper risk, coverage, tightness, and effective variance $\bm a^\top\bm\Sigma\bm a$. For the balanced binary setting, bounded-discriminability certification is the empirical frequency of $U_F<0.5$.

\end{document}